\documentclass{article}
\usepackage[margin=1in]{geometry}

\usepackage[utf8]{inputenc} % allow utf-8 input
\usepackage[T1]{fontenc}    % use 8-bit T1 fonts
\usepackage{hyperref}       % hyperlinks
\usepackage{url}            % simple URL typesetting
\usepackage{booktabs}       % professional-quality tables
\usepackage{amsfonts}       % blackboard math symbols
\usepackage{nicefrac}       % compact symbols for 1/2, etc.
\usepackage{microtype}      % microtypography
\usepackage{xcolor}         % colors

\usepackage{graphicx,wrapfig,color,amsmath}
\usepackage{amsfonts,enumitem,amsthm, amssymb}
\usepackage{comment}
\usepackage{cite}
\usepackage{subcaption}
\usepackage{graphicx}

\newtheorem{theorem}{Theorem}
\newtheorem{lemma}{Lemma}

\newtheorem{proposition}{Proposition}

\theoremstyle{definition}

\newcommand{\numreviewers}{m}
\newcommand{\numpapers}{n}
\newcommand{\numrev}{\numreviewers}
\newcommand{\numpap}{\numpapers}
\newcommand{\revset}{\mathcal{R}}
\newcommand{\papset}{\mathcal{P}}
\newcommand{\stagetwofrac}{\beta}
\newcommand{\rank}{k}
\newcommand{\adrev}{r}
\newcommand{\adpap}{p}
\newcommand{\numsamples}{K}
\newcommand{\simmat}{S}
\newcommand{\objfn}{f}
\newcommand{\sobjfn}{\overline{\objfn}}

\newcommand{\loadscale}{\mu}
\newcommand{\scaledavgval}{s^{(\loadscale)}}
\newcommand{\scaledval}{v^{(\loadscale)}}
\newcommand{\optavgval}{s^{(1)}}
\newcommand{\optval}{v^{(1)}}
\newcommand{\revload}{\ell_{rev}}
\newcommand{\papload}{\ell_{pap}}
\newcommand{\valfn}{Q}

\newcommand{\adassign}{A}
\newcommand{\optterm}{oracle optimal}

\title{Randomly Splitting Reviewers is Near-Optimal\\ in Two-Stage Paper Reviewing} 
\title{Near-Optimal Reviewer Splitting in Two-Phase Paper Reviewing \\and Conference Experiment Design} 
\author{
    Steven Jecmen \\
    Carnegie Mellon University \\
    \texttt{sjecmen@cs.cmu.edu} \\
    \and
    Hanrui Zhang \\
    Duke University \\
    \texttt{hrzhang@cs.duke.edu} \\
    \and
    Ryan Liu \\
    Carnegie Mellon University \\
    \texttt{ryanliu@andrew.cmu.edu} \\
    \and   
    Fei Fang \\
    Carnegie Mellon University \\
    \texttt{feif@cs.cmu.edu} \\
    \and
    Vincent Conitzer \\
    Duke University \\
    \texttt{conitzer@cs.duke.edu} \\
    \and
    Nihar B. Shah \\
    Carnegie Mellon University \\
    \texttt{nihars@cs.cmu.edu} \\
}
\date{} 
\date{}

\begin{document}

\maketitle

\begin{abstract}
Many scientific conferences employ a two-phase paper review process, where some papers are assigned additional reviewers after the initial reviews are submitted. 
Many conferences also design and run experiments on their paper review process, where some papers are assigned reviewers who provide reviews under an experimental condition. 
In this paper, we consider the question: how should reviewers be divided between phases or conditions in order to maximize total assignment similarity? 
%In this paper, we consider the question: which reviewers should be saved for the second stage of reviews in order to maximize total assignment similarity when the set of papers requiring additional review is unknown? 
We make several contributions towards answering this question. 
%First, we identify and formulate this two-stage reviewing problem. % sjecmen: not sure if this should be here since we use it as a contribution later
First, we prove that when the set of papers requiring additional review is unknown, a simplified variant of this problem is NP-hard. 
Second, we empirically show that across several datasets pertaining to real conference data, dividing reviewers between phases/conditions uniformly at random allows an assignment that is nearly as good as the \optterm{} assignment. This uniformly random choice is practical for both the two-phase and conference experiment design settings. 
Third, we provide explanations of this phenomenon by providing theoretical bounds on the suboptimality of this random strategy under certain natural conditions. From these easily-interpretable conditions, we provide actionable insights to conference program chairs about whether a random reviewer split is suitable for their conference. 
%These results and insights also apply to the setting of conference experiment design, showing that conferences can experimentally test changes to their reviewing process while still achieving a high-quality assignment.
\end{abstract}

\section{Introduction} \label{sec:intro} 
Peer review is a widely-adopted method for evaluating scientific research. Careful assignment of reviewers to papers is critically important in order to ensure that reviewers have the requisite expertise and that the resulting reviews are of high quality. 
At large scientific conferences, the paper assignment is usually chosen by solving an optimization problem. %the assignment of reviewers to papers is usually chosen in order to maximize the total quality of the assignment, as represented by similarity scores for each reviewer-paper pair.  % introduce standard problem here?
Given a set of papers, a set of reviewers, and a similarity matrix consisting of scores representing the level of expertise each reviewer has for each paper, the standard paper assignment problem is to find an assignment of reviewers to papers that maximizes total similarity, subject to constraints on the reviewer and paper loads. This standard paper assignment problem is a simple matching problem and so can be efficiently solved (for example, through linear programming). %\vc{don't even need linear programming for basic problem, just a matching problem -- not sure if this is worth pointing out}
Our work is motivated by two scenarios that arise in the context of paper assignment in conference peer review.  

{\bf Motivation 1: Two-phase paper assignment.} Many conferences (e.g., AAAI 2021-2022, IJCAI 2022) have adopted a two-phase review process. After the initial reviews are submitted, a subset of papers proceed to a second phase of reviews with additional reviewers assigned. There are a variety of reasons that a two-phase reviewing process can be helpful. For example, the process can be used to triage papers based on reviews in the first phase. This can allow the conference to solicit additional reviews only on papers that obtained sufficiently high ratings in the first phase and have any chance of getting accepted (as done at AAAI 2021). 
The second phase can also help focus on evaluation of the papers in the ``messy middle''---the papers at the borderline between acceptance and rejection. 
This messy middle model~\cite{price2014}, which hypothesizes that the acceptance decisions for some percentage of submitted papers are effectively random, was proposed after the NeurIPS 2014 experiment~\cite{lawrence2014} in order to explain the observed inconsistency in acceptance decisions.  %\vc{one might argue, why spend reviewers on it if it's random anyway?}
Additional reviewers could improve the evaluation of these papers to more accurately discern which should be accepted. 
Later analysis of the NeurIPS 2015 and 2016 review process found that the size of the messy middle in these conferences was $45\%$ and $30\%$ of submissions respectively~\cite{shah2017design}. %these two sentences can be deleted if we need to save space
A second phase of reviews can also be used to help compensate for reviewers who were unresponsive or minimally responsive in the first phase, who can no longer review due to problems in their personal lives, who discovered conflicts they had with a paper assigned to them and recused themselves from it, etc.%\vc{extended this a bit} 

In all of these cases, the set of papers that will require additional review is unknown beforehand. 
While some venues choose to recruit new reviewers after knowing which papers proceed to phase two, the tight timeline of many conferences makes it hard to recruit new reviewers after phase one. 
In SIGMOD 2019~\cite{ailamaki2019sigmod}: ``\textit{Additional reviews were solicited manually by the chairs and this was a huge time sink, especially when some reviewers refused to take on the additional assignment. The additional review solicitation needs to be automated and reviewer expectations need to be set appropriately beforehand}.'' 
%For this reason, it is best if all the reviewers are recruited at the beginning, and a key question is then which reviewers should be saved for the second round of reviews.
For this reason, it is best if all the reviewers are recruited at the beginning, and a key question is then how to assign reviewers to papers in the first phase such that enough review capacity is saved for the second phase. 

{\bf Motivation 2: Conference experiment design.} Reviewers also need to be split into two groups when conferences run controlled experiments on the paper review process.  Conferences often run such experiments to test changes to the review process. For example, the WSDM 2017 conference conducted an experiment to test the effects of single-blind versus double-blind reviewing~\cite{Tomkins12708}. 
As another example, the aforementioned NeurIPS 2014 experiment tested the consistency of acceptance decisions by providing some papers with a second set of reviews from a separate group of reviewers. 
In these experiments, all papers receive reviews conducted in the usual manner (the control condition), but a random subset of papers are additionally assigned reviewers who provide reviews under an experimental condition. In the NeurIPS 2014 experiment, a random $10\%$ of papers were put in the experimental condition and received a second set of reviews. In the WSDM 2017 experiment, the subset of papers was the full paper set; that is, all papers were reviewed under both single-blind and double-blind conditions. 
%Since the subset of papers in the experimental condition should be chosen in a random manner to allow for statistical claims to be made about the results, we assume that it is chosen uniformly at random. 
The key question is then how to divide the reviewers between the control and experimental conditions. 
As in the NeurIPS 2014 and WSDM 2017 experiments, this is often done randomly for statistical purposes. However, conferences still want to ensure that the resulting assignment of papers to reviewers is of high similarity.

As our results apply to both the two-phase and experiment design settings, we will use the generic terminology of ``stages'' to refer to both phases and conditions simultaneously across the two settings. %\vc{not quite sure how to parse ``simultaneously across settings'' -- settings are the two motivations?} sj: added "the two" to clarify I'm referring back to the settings referenced earlier in the sentence?

%The results and insights we present in this work apply equally well to this setting. Our empirical results show that splitting reviewers uniformly at random between the two conditions gives near-optimal assignments on realistic similarities (without needing to compromise any statistical guarantees of the experiment), and our theoretical explanations help conference program chairs decide if running an experiment in this way will give a high-similarity assignment at their own conference. For ease of presentation, we use the terminology of the two-stage setting throughout the paper and elaborate on the application to the conference experiment design setting in Appendix~\ref{apdx:controlledexp}. 

{\bf Problem outline.} In this paper, we formally analyze the two-stage paper assignment problem, which encompasses both above motivations. %Given a set of papers, a set of reviewers, and a similarity matrix consisting of scores representing the level of expertise each reviewer has for each paper, the standard paper assignment problem is to find an assignment of reviewers to papers that maximizes total similarity, subject to constraints on the reviewer and paper loads. 
%In this paper, we consider the problem of choosing a paper assignment across two stages, encompassing both above motivations. 
As stated earlier, the standard paper assignment problem is to maximize the total similarity of the assignment subject to load constraints and is efficiently solvable. 
%However, in the two-stage paper assignment problem, we must additionally decide which reviewers should be assigned to papers in the first stage (or control condition) and which reviewers should be saved to review papers in the second stage (or experimental condition). 
However, in the two-stage paper assignment problem, we must additionally decide how much of each reviewer's capacity should be saved to review papers in the second stage. We assume that the \emph{fraction} of papers that will need additional reviews is known and that the set of second-stage papers is chosen uniformly at random.

Because of constraints present in each setting, the maximum-similarity paper assignment across the two stages cannot be achieved. In the two-phase setting, the set of papers that will need to be reviewed in the second stage is unobserved when the first-stage assignment is made, making the problem one of stochastic optimization. In the experiment design setting, reviewers are often randomized between stages for statistical purposes. We show that a simple strategy for choosing reviewers to save for the second stage performs near-optimally in terms of assignment similarity and can be used in either setting.

%The difficulty of finding the optimal set of reviewers to save for the second stage lies in the fact that the set of papers that will need to be reviewed in the second stage is unobserved when this decision is made. We assume that the fraction of papers that will need additional reviews is known and that the set of second-stage papers is chosen uniformly at random. We then aim to find a set of second-stage reviewers such that the expected total similarity of the paper assignments in both stages is maximized. Due to the uncertainty in the second-stage papers, the problem we are facing is not a standard matching problem but a stochastic optimization problem, as we show in Section~\ref{sec:problem}.

% new addition
%We find that a very simple ``random split'' strategy, which saves a subset of reviewers chosen uniformly at random for the second stage, gives near-optimal assignments on real conference similarity scores. In Figure~\ref{fig:perf}, we see that random reviewer splits achieve at least $90\%$ of the \optterm solution's similarity (which views the set of second-stage papers before deciding which reviewers to save for the second stage) across all trials on all datasets. These results hold across similarities constructed via a variety of methods used in practice, including text-matching, bidding, and subject areas. They also hold both when the second-stage papers are drawn uniformly at random (as in Fig.~\ref{fig:real}) and when they are selected based on the review scores of the papers (as in Fig.~\ref{fig:scores}). 

\begin{figure*}[t!]
    \centering
    \begin{subfigure}[t]{0.58\textwidth}\includegraphics[width=1\textwidth]{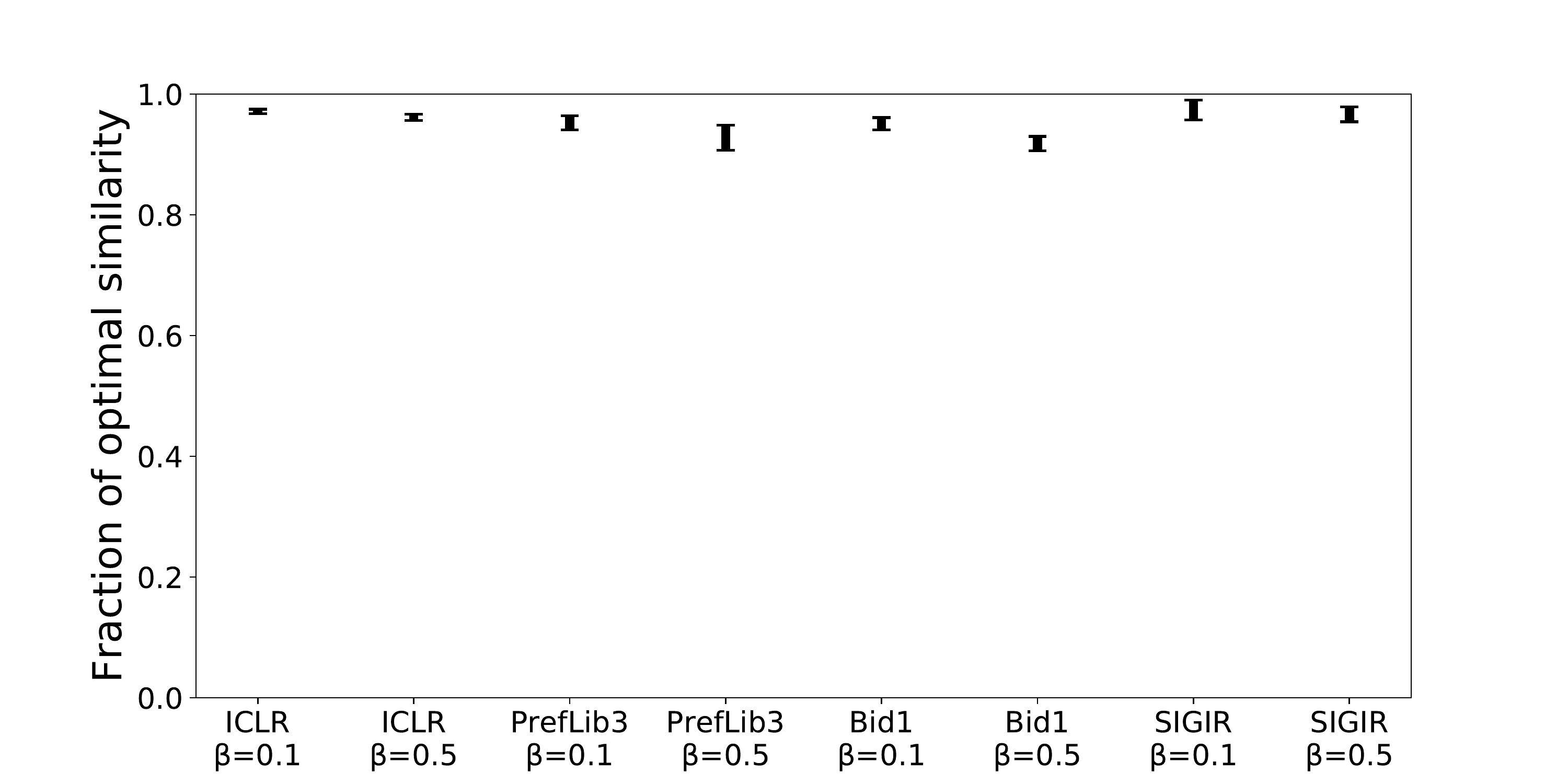}\caption{Second-stage papers drawn uniformly at random}\label{fig:real} \end{subfigure}\quad
    \begin{subfigure}[t]{0.38\textwidth}\includegraphics[width=1\textwidth]{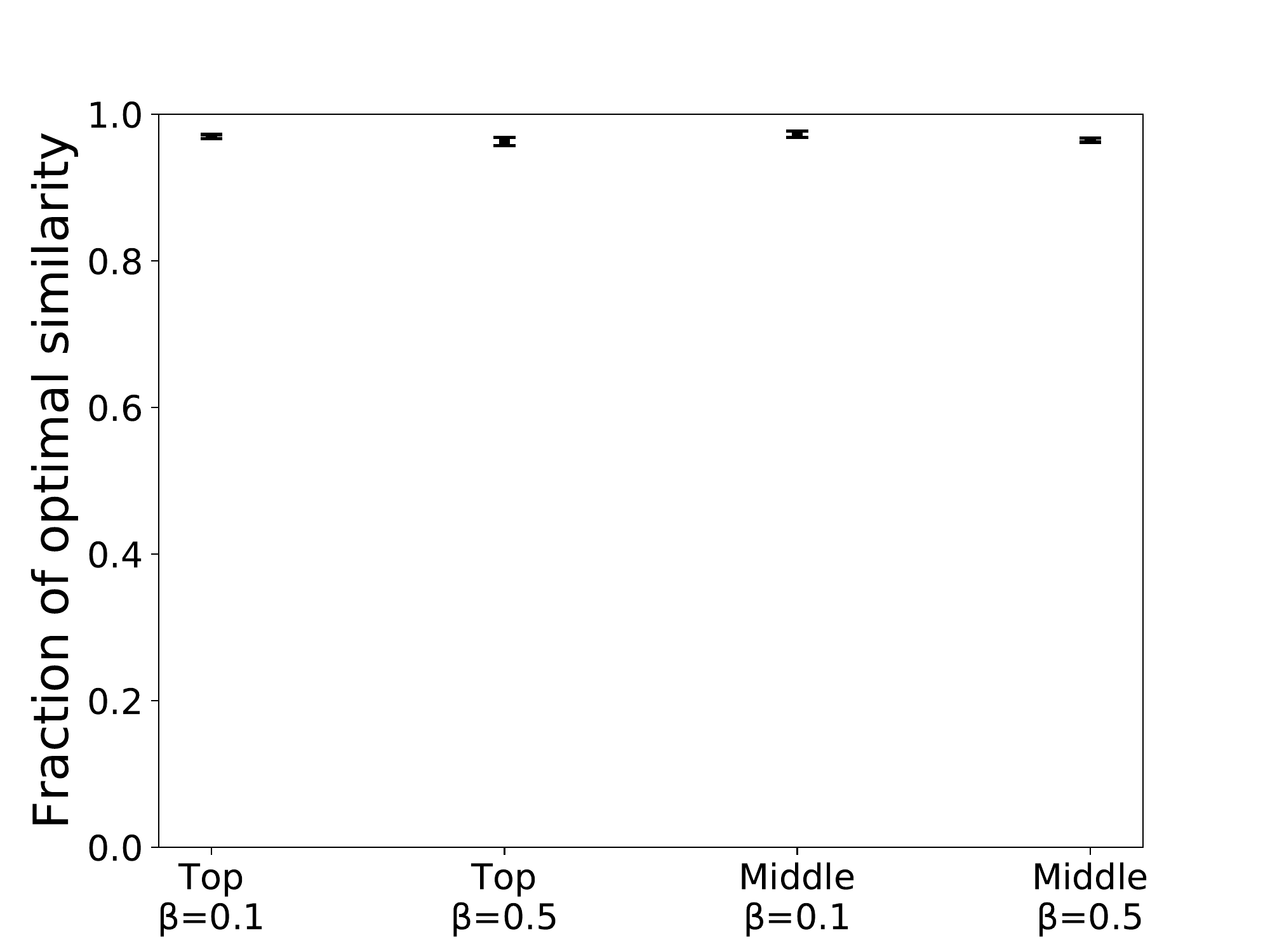} \caption{Second-stage papers chosen as the top- or middle-scoring papers from ICLR }\label{fig:scores} \end{subfigure} 
    \caption{Range of assignment similarities found over $10$ random reviewer splits on real conference data, as a fraction of the \optterm{} assignment's similarity (computed after observing the second-stage papers). $\stagetwofrac$ indicates the fraction of papers in the second stage. The
    ICLR similarities~\cite{xu2018strategyproof} (911 papers, 2435 reviewers) are constructed from text-matching between papers and reviewers' past work, PrefLib3~\cite{MaWa13a} (176 papers, 146 reviewers) and Bid1~\cite{meir2020market} (600 papers, 400 reviewers) similarities are constructed from bidding data, and SIGIR~\cite{karimzadehgan2008multi} similarities (73 papers, 189 reviewers) are constructed from reviewer and paper subject areas.
    } \label{fig:perf}
\end{figure*}

{\bf Contributions.} 
Our contributions are as follows.

%\begin{enumerate}
    First, we identify and formulate the two-stage paper assignment problem, an issue of practical importance to modern conferences, with applications to two-phase paper assignment and conference experiment design (Section~\ref{sec:problem}). 
    
    Second, we prove that a simplified version of the problem is NP-hard, suggesting that the problem may not be efficiently solvable  (Section~\ref{sec:hard}).
    
    Third, we empirically show that a very simple ``random split'' strategy, which chooses a subset of reviewers uniformly at random to save for the second stage, gives near-optimal assignments on real conference similarity scores (Section~\ref{sec:approachexp}). 
    This result is summarized in Figure~\ref{fig:perf}, which shows the assignment similarity achieved using random split as compared to the \optterm{} assignment (which views the set of second-stage papers before optimally assigning reviewers across both stages) for several datasets. We find that all random reviewer splits achieve at least $90\%$ of the \optterm{} solution's similarity on all datasets and at least $94\%$ on all but two experiments. 
    These results hold across similarities constructed via a variety of methods used in practice (including text-matching, bidding, and subject areas), indicating that random split is robust across methods of similarity construction. They also hold both when the second-stage papers are drawn uniformly at random (as in Figure~\ref{fig:real}) and when they are selected based on the review scores of the papers (as in Figure~\ref{fig:scores}). 
    In practice, this means that program chairs planning a two-phase review process or a conference experiment can simply split reviewers across the two phases/conditions at random without concerning themselves with the potential reduction in assignment quality. 
    %These results also indicate that in conference experiment design, reviewers can be randomly split between conditions without significant loss of assignment quality on these real similarity matrices. 
    %We find that saving a subset of reviewers uniformly at random for the second stage gives near-optimal assignments on real conference similarity scores (Section~\ref{sec:approachexp}). 
    
    We also show that this good performance is not achieved in general: there exist similarity matrices on which random split performs very poorly (Section~\ref{sec:approachcounterex}). 
    
    Fourth, we theoretically \textit{explain} why random split performs well on our real conference similarity matrices by deriving theoretical bounds on the suboptimality of this random strategy under certain natural conditions (Sections~\ref{sec:cond1} and \ref{sec:cond2}). 
    We consider two such sufficient conditions here, which are met by our datasets: if the reviewer-paper similarity matrix is low-rank, and if the similarity matrix allows for a high-value assignment (in terms of total similarity) with a large number of reviewers assigned to each paper. % if the similarity matrix allows for a high-total similarity solution in an auxiliary assignment problem.
    %exists. if the number of reviewers needed for each paper is high
    From these results, we give key actionable insights to conference program chairs to help them decide--well before the reviewers and/or papers are known--if random split is likely to perform well in their conference.
    % sj: haven't atm for space, but could emphasize futher: before the precise similarity matrix for the conference is known, since the similarities may not be known in full until late in the planning process. 
%\end{enumerate}

All of the code for our empirical results is freely available online.\footnote{\url{https://github.com/sjecmen/multistage_reviewing_bounds}}

\section{Related Work} \label{sec:relwork}
Our work assumes that the ``similarities'' between reviewers and authors are given. In practice, there are several ways in which these similarities are computed, and different program chairs often make different decisions on how this computation is done. The similarities are generally computed using one or more of the following three sources of data:
%(i) {\bf Text-matching of papers:} Natural language processing techniques~\cite{mimno07topicbased,liu14graphpropagation,rodriguez08coauthorsip,tran17expertsuggestion,charlin13tpms} are used to match the text of the submitted paper with the text of the reviewer's past papers; (ii) {\bf Subject areas:} The program chairs create a list of subject areas relevant to the conference. Each reviewer selects a subset of these subject areas that are representative of their expertise, and each submitted paper is accompanied by the authors selecting the relevant subject areas; (iii) {\bf Bids:} A number of conferences adopt a bidding system, where reviewers are shown a list of the submitted papers (which do not conflict with them) and asked to indicate the papers which they are willing to review~\cite{meir2020market,fiez2020super}.
\begin{itemize}
    \item {\bf Text-matching of papers:} Natural language processing techniques~\cite{mimno07topicbased,liu14graphpropagation,rodriguez08coauthorsip,tran17expertsuggestion,charlin13tpms} are used to match the text of the submitted paper with the text of the reviewers' past papers.
     \item {\bf Subject areas:} The program chairs create a list of subject areas relevant to the conference. Each reviewer selects a subset of these subject areas that are representative of their expertise, and each submitted paper is accompanied by the authors selecting the subject areas relevant to the paper. 
     \item {\bf Bids:} A number of conferences adopt a bidding system, where reviewers are shown a list of (some of) the papers that are submitted to the conference (and which do not conflict with them) and asked to indicate the papers which they are willing to review~\cite{cabanac2013capitalizing,fiez2020super,meir2020market}. 
\end{itemize}
If more than one such source of data is used by the conference, they are combined in a manner deemed suitable by the program chairs~\cite{shah2017design}. These computed similarities are then used to assign reviewers to papers. By far the most popular method of doing this assignment is to solve an optimization problem that maximizes the sum of the similarities of the assigned reviewer-paper pairs, subject to constraints on the reviewer and paper loads~\cite{charlin13tpms,Long13gooadandfair, goldsmith07aiconf, tang10constraied,flach2010kdd,taylor2008optimal}. Given its widespread popularity, we analyze this sum-similarity objective in our paper. 

That being said, there are other objectives that are also proposed for automated assignment using the similarities, such as the max-min fairness objective~\cite{Garg2010papers,stelmakh2018forall,kobren19localfairness}. A recent work~\cite{jecmen2020manipulation} proposes assignments via optimizing the sum similarity but with some randomness in order to prevent fraud in peer review. Another line of work~\cite{alon2011sum,xu2018strategyproof,aziz2019strategyproof} proposes assigning reviewers to papers in a manner that a reviewer cannot influence the outcome of their own paper by manipulating the reviews they provide. Finally, in practice, the conference organizers may also additionally apply manual tweaks to the outputs of any such automated procedure.

At a high level, the problem we study in the two-phase setting shares several common characteristics with problems in online (stochastic) matching~\cite{karp1990optimal,feldman2009online, dickerson2012dynamic, dickerson2018allocation, brubach2016new}, often considered in the context of ride-sharing, kidney exchange, or internet advertising. 
%There, the normal assumption is we have a bipartite graph, where vertices on one side (corresponding to reviewers) are known beforehand, and those on the other side (corresponding to papers) arrive one by one online, possibly following some known prior distribution.
%Upon the arrival of each online vertex, we see the weights of matching the arriving online vertex to each of the offline vertices (corresponding to similarities in our problem), and have to irrevocably match the arriving vertex to one or more of the offline vertices subject to capacity constraints.
Particularly related to our results is the line of research on two-stage stochastic matching~\cite{kong2006factor,katriel2008commitment,escoffier2010two,lee2020maximum,feng2021two}, which generally focuses on providing algorithms with tight approximation ratios that hold for any (i.e., worst-case) problem instances. 
%The main difference between our results and this line of work is that conceptually, we focus on the applied aspects of the problem, aiming to provide and justify simple and practical solutions (such as choosing reviewers uniformly at random) based on data-dependent conditions likely to hold in real-world paper assignment scenarios.
To the best of our knowledge, the specific stochastic matching problem we consider (which arises in the context of paper assignment for peer review) has not previously been studied. 
Additionally, in contrast to this line of work, we aim to provide and justify simple and practical solutions (such as choosing reviewers uniformly at random) based on data-dependent conditions likely to hold in real-world paper assignment instances. 
%Additionally, we care mostly about the additive suboptimality with respect to the optimal matching since it has a direct interpretation as the loss in assignment quality, whereas most results on two-stage stochastic matching consider the approximation ratio (which is not directly comparable). 
%To the best of our knowledge, two-stage stochastic matching has not previously been studied specifically in the context of paper assignment for peer review. 

The simplified version of our problem considered in Section~\ref{sec:hard} can be seen as an instance of maximizing a submodular function subject to a cardinality constraint (see Appendix~\ref{apdx:submod}). The paper~\cite{buchbinder2014submodular} gives an approximation algorithm achieving an approximation ratio of no greater than $0.5$. However, this guarantee is very weak in the paper assignment setting since it can be trivially achieved by maximizing similarity in the first stage alone. % maybe remove this and replace with more experimentally-relevant stuff?
%Furthermore, this approximation is a constant and does not improve as the problem size grows (unlike our results). % could remove now

One motivation for our work is that of running controlled experiments in peer review. Controlled experiments pertaining to peer review are conducted in many different scientific fields~\cite{armstrong1980unintelligible,pier2017your,teplitskiyasocial,ceci1982peer,patat2019distributed}, including several controlled experiments recently conducted in computer science~\cite{lawrence2014,Tomkins12708,stelmakh2020resubmissions,stelmakh2020herding}. These experiments have also led to a relatively nascent line of work on careful design of experimental methods for peer review~\cite{stelmakh2019testing,stelmakh2020catch}, and our work sheds some light in this direction in terms of trading off assignment quality with randomization in the assignment. Some other experiments in conferences~\cite{madden2006impact,tung2006impact,manzoor2020uncovering} do not operate under controlled settings, but exploit certain changes in the conference policy such as a switch from single blind to double blind reviewing (i.e., natural experiments). %\vc{i.e., a natural experiment?}
Overall, experiments offer important insights into the peer review process; see~\cite{shah2021survey} for more discussion on challenges in peer review and some solutions. 

% could add something on conference experiment design, or tradeoffs between randomness and other objectives?
%\vc{should we include any references from the NeurIPS reviews here?  Also, it may be more natural to reorder the last few paragraphs above to be consistent with the order of motivation 1 and motivation 2}

\section{Problem Formulation} \label{sec:problem} 
In this section, we formally define the two-stage paper assignment problem. 
Given a set of $\numpap$ papers $\papset = [\numpap]$ and a set of $\numrev$ reviewers $\revset = [\numrev]$, define $\simmat \in [0, 1]^{\numrev \times \numpap}$ as the similarity scores between each reviewer and paper. An assignment of papers to reviewers is represented as a matrix $\adassign \in \{0, 1\}^{\numrev \times \numpap}$, where $\adassign_{\adrev \adpap} = 1$ if reviewer $\adrev$ is assigned to paper $\adpap$ and $\adassign_{\adrev \adpap} = 0$ otherwise. 
In the standard paper assignment problem, the objective is to find an assignment $\adassign$ of reviewers to papers such that the total similarity $\sum_{\adrev \in \revset, \adpap \in \papset} \adassign_{\adrev\adpap} \simmat_{\adrev\adpap}$ is maximized, subject to constraints that each paper is assigned exactly a certain load of reviewers, each reviewer is assigned to at most a certain load of papers, and any reviewer-paper pairs with a conflict of interest are not assigned~\cite{charlin13tpms, charlin2012framework, goldsmith07aiconf, flach2010kdd, kobren19localfairness}. 
In this work, we accommodate conflicts of interest by assuming the corresponding similarities are set to $0$. 
This problem can be formulated as a min-cost flow problem or as a linear program, and can be efficiently solved. 
%\ns{this part states using the sum similarity. Later the mean is used. Let's try to be consistent, or otherwise, write a clarifying statement later saying that wlog you'll consider the mean since the later results will be easier to interpret} \sj{Added a statement before oracle optimal definition.} 

In a two-stage assignment, all papers $\papset$ require a certain number of reviewers in the first stage and a subset of papers $\papset_2 \subseteq \papset$ require additional review in the second stage. We assume that $\papset_2$ consists of a fixed fraction $\stagetwofrac$ of papers and is drawn uniformly at random from $\papset$. Specifically, for some $\stagetwofrac \in \{\frac{1}{\numpap}, \dots, \frac{\numpap}{\numpap}\}$, we assume that $\papset_2 \sim \mathcal{U}_{\stagetwofrac\numpap}(\papset)$, the uniform distribution over all subsets of size $\stagetwofrac \numpap$ of $\papset$. 
In the two-phase setting, the fraction $\stagetwofrac$ itself can be viewed as a parameter that the program chairs set based on available reviewer resources, or it can be estimated from past editions of the conference. 
Our empirical results detailed in Section~\ref{sec:approachexp} also cover the case where papers are chosen for the second phase based on their first-phase review scores.  
In the conference experiment design setting, the value of $\stagetwofrac$ and the uniform distribution of $\papset_2$ are both experiment design choices. The choice of a uniform distribution for $\papset_2$ is common, as in the NeurIPS 2014 and WSDM 2017 experiments. 
The question we analyze is: how should reviewers be assigned to papers across the two stages? 
%\ns{recall in text or footnote that in the papers that actually would be chosen to the second phase, we see random split doing well (fig 1b)} \sj{Added the following instead? I think it's strange to mention random split directly since it hasn't been formally introduced or to mention Fig 1b since it's formally discussed in detail later (Sec 5.1).}

Before continuing further, we introduce some notation. For subsets $\revset' \subseteq \revset$ and $\papset' \subseteq \papset$, desired paper load $\papload \in \mathbb{Z}_+$, and maximum reviewer load $\revload \in \mathbb{Z}_+$, define $\mathcal{M}(\revset', \papset'; \revload, \papload) \subseteq \{0, 1\}^{\numrev \times \numpap}$ as the set of valid assignment matrices on $\revset'$ and $\papset'$. Formally, $\adassign \in \mathcal{M}(\revset', \papset'; \revload, \papload)$ if and only if $\sum_{\adrev \in \revset'} \adassign_{\adrev \adpap} = \papload$ for all $\adpap \in \papset'$, $\sum_{\adpap \in \papset'} \adassign_{\adrev \adpap} \leq \revload$ for all $\adrev \in \revset'$, and $\adassign_{\adrev \adpap} = 0$ for all $(\adrev, \adpap) \not\in \revset' \times \papset'$.

The two-stage paper assignment problem is to maximize the total similarity of the paper assignment across both stages.
Without loss of generality, we instead consider the mean similarity so that later results will be easier to interpret. 
Fix a stage one paper load $\papload^{(1)}$, a stage two paper load $\papload^{(2)}$, and an overall reviewer load $\revload$ such that $\papload^{(1)} \numpap + \papload^{(2)} \stagetwofrac\numpap \leq \revload \numrev$ (i.e., the number of reviews required by papers is no greater than the number of reviews that can be supplied by reviewers). 
Given $\papset_2$, the \optterm{} assignment has mean similarity
\begin{align*}
    \valfn^*(\papset_2) = \max_{\substack{A \in \mathcal{M}(\revset, \papset; \revload, \papload^{(1)}), \\ B \in \mathcal{M}(\revset, \papset_2; \revload, \papload^{(2)})}}& \frac{1}{\papload^{(1)} \numpap + \papload^{(2)} \stagetwofrac\numpap} \left[  \sum_{\adrev \in \revset, \adpap \in \papset} A_{\adrev \adpap} \simmat_{\adrev \adpap} +  \sum_{\adrev \in \revset, \adpap \in \papset_2} B_{\adrev \adpap} \simmat_{\adrev \adpap} \right] \\
    \text{subject to } &\sum_{\adpap \in \papset} A_{\adrev \adpap} + B_{\adrev \adpap} \leq \revload \qquad \forall \adrev \in \revset.
\end{align*}
The last constraint ensures that each reviewer's assignment across both stages does not exceed the maximum reviewer load. Just like the standard paper assignment problem, the \optterm{} assignment for a given $\papset_2$ can be found efficiently. 
However, in both the two-phase and experiment design settings, this \optterm{} assignment is either unachievable or undesirable. In the two-phase setting, the set of papers $\papset_2$ requiring additional review is unknown until after the stage one assignment is chosen. Thus, the \optterm{} assignment cannot be computed beforehand. %Thus, the stage one assignment must be chosen to maximize the expectation (over $\papset_2 \sim \mathcal{U}_{\stagetwofrac\numpap}(\papset)$) of the total assignment similarity in both stages. 
In the experiment design setting, the assignment of reviewers to conditions is commonly randomized in order to gain statistical power, as was done in the WSDM 2017 and NeurIPS 2014 experiments. Thus, a deterministic choice of assignment may not be desirable. Additionally, depending on the experiment setup, it may not be possible for a reviewer to review papers in both conditions. In what follows, we use this \optterm{} assignment value as an unachievable baseline for comparison. %We compare to this \optterm assignment as an unachievable baseline.

We instead consider simple strategies for the two-stage assignment problem that choose a subset $\revset_2 \subseteq \revset$ of reviewers to save for the second stage without observing $\papset_2$, leaving reviewers $\revset_1 = \revset \setminus \revset_2$ to be assigned to papers in the first stage. Unlike the \optterm{} assignment, such strategies are feasible to implement in both settings since they do not require knowledge of $\papset_2$, do not split reviewer loads across conditions, and allow for a random choice of $\revset_2$. The mean similarity of the optimal assignment when reviewers $\revset_2$ and papers $\papset_2$ are in the second stage is 
\begin{align*}
    \valfn(\revset_2, \papset_2) &= \frac{1}{\papload^{(1)} \numpap + \papload^{(2)}\stagetwofrac \numpap} \left[ \max_{\substack{A \in \mathcal{M}(\revset \setminus \revset_2, \papset;\\ \revload, \papload^{(1)})}} \sum_{\adrev \in \revset \setminus \revset_2, \adpap \in \papset} A_{\adrev \adpap} \simmat_{\adrev \adpap} +  \max_{\substack{B \in \mathcal{M}(\revset_2, \papset_2;\\ \revload, \papload^{(2)})}} \sum_{\adrev \in \revset_2, \adpap \in \papset_2} B_{\adrev \adpap} \simmat_{\adrev \adpap} \right].
\end{align*}
We require that $\revload |\revset_2| \geq \papload^{(2)} \stagetwofrac\numpap$ and $\revload (\numrev - |\revset_2|) \geq \papload^{(1)} \numpap$ for feasibility in both stages. 
Given $\revset_1$, $\revset_2$, and $\papset_2$, the optimal paper assignment in each stage can be efficiently computed using standard methods. Thus, the the difficulty of the problem lies entirely in choosing $\revset_2$. 

The expected mean similarity of the optimal assignment when saving reviewers $\revset_2$ for the second stage is 
\begin{align*}
    \objfn(\revset_2) = \mathbb{E}_{\papset_2 \sim \mathcal{U}_{\stagetwofrac\numpap}(\papset)}\left[\valfn(\revset_2, \papset_2)\right]. 
\end{align*}
We can also evaluate the {\it suboptimality} of $\revset_2$ as compared to the \optterm{} assignment as 
\begin{align*}
    \valfn^*(\papset_2) - \valfn(\revset_2, \papset_2), \qquad \qquad \text{where~~} \papset_2 \sim \mathcal{U}_{\stagetwofrac\numpap}(\papset).
\end{align*}
Note that $\valfn^*$ and $\valfn$ are bounded in $[0, 1]$, so that both $\objfn$ and the suboptimality are also bounded in $[0, 1]$. 
%\ns{is this notation $\suboptfn$ really used? I didn't see it in the thm statements.} \sj{I removed the notation since it's not used. In the results, we refer to the ``suboptimality of random split'' which is never formally defined, so I added a sentence to Sec 5 to define it. We only occasionally refer to suboptimality not wrt random split (occasionally "suboptimality of an assignment"), so I think it would be better to remove the definition from here and define it instead in Sec 5 after introducing random split?}
%\ns{clearly specify somewhere that $\valfn^*$, $\valfn$ and $\suboptfn$ all lie in $[0,1]$. This will give better context for your results later where you point to a $n^{-1/2}$ etc. suboptimality} \sj{Added (slightly awkward to refer to only suboptimality in words).}

In our theoretical analysis, for simplicity, we assume that $\revload = \papload^{(1)} = \papload^{(2)} = 1$, leaving this implicit in $\objfn$, $\valfn$, and $\valfn^*$ throughout the paper. We also assume $\numrev = (1 + \stagetwofrac) \numpap$ in our analysis unless specified otherwise, so that $|\revset_2| = \stagetwofrac\numpap$. 
The intuition behind our results carries over to the cases of general loads and excess reviewers, which are covered by our empirical results in Section~\ref{sec:approachexp}. 
All asymptotic bounds are given as $\numpap$ grows.
%\ns{If any of this general cases are covered in fig 1, you may consider recalling that here.}\sj{Added, although I again like referring forward to the detailed explanation in Sec 5 rather than recalling the preview in Fig 1.} 

\section{Hardness} \label{sec:hard} 
In the two-phase setting, the \optterm{} assignment is unachievable because $\revset_2$ must be chosen before observing $\papset_2$. Therefore, conferences must choose $\revset_2$ to maximize $\objfn$, the expected mean similarity of the assignment across both stages. 
In this section, we demonstrate that maximizing a variant of $\objfn$ is NP-hard, indicating that it is unlikely that $\objfn$ can be optimized efficiently. 

First, note that evaluating $\objfn(\revset_2)$ requires computing an expectation over the draw of $\papset_2$, which naively requires evaluating a sum over the optimal assignment value for ${\numpap \choose \stagetwofrac \numpap} $ possible choices of $\papset_2$. This number is exponential in the input size, so an efficient algorithm for this problem would have to either optimize $\objfn$ without evaluating it or compute this expectation without computing the optimal assignment for each possible $\papset_2$. 

% NP-hard proof (THM)
Instead of attempting to optimize $\objfn$ exactly, a standard approach from two-stage stochastic optimization is to simplify the problem by sampling as follows~\cite{dai2000convergence, king1993asymptotic}. First, take some fixed number of samples $\papset_2^{(1)}, \dots, \papset_2^{(\numsamples)}$ from $\mathcal{U}_{\stagetwofrac\numpap}(\papset)$. Then, rather than optimizing an average over all $\papset_2$ in the support of $\mathcal{U}_{\stagetwofrac \numpap}(\papset)$, choose $\revset_2$ to optimize an average over only all sampled sets:
\begin{align*}
    %\sobjfn(\revset_2) &= \frac{1}{\papload^{(1)} \numpap + \papload^{(2)}\stagetwofrac \numpap} \left[ \max_{\substack{A \in \mathcal{M}(\revset \setminus \revset_2, \papset; \\ \revload, \papload^{(1)})}}  \sum_{\adrev \in \revset \setminus \revset_2, \adpap \in \papset} A_{\adrev \adpap} \simmat_{\adrev \adpap} + \frac{1}{\numsamples} \sum_{i=1}^\numsamples \left( \max_{\substack{B \in \mathcal{M}(\revset_2, \papset_2^{(i)}; \\ \revload, \papload^{(2)})}} \sum_{\adrev \in \revset_2, \adpap \in \papset_2^{(i)}} B_{\adrev \adpap} \simmat_{\adrev \adpap} \right) \right].\\
    \sobjfn(\revset_2) &= \frac{1}{\numsamples} \sum_{i=1}^\numsamples \valfn(\revset_2, \papset_2^{(i)}).
\end{align*} 
This is a natural simplification of the two-stage paper assignment problem, because the sum in the objective is now taken over only a constant $\numsamples$ subsets rather than an exponential number. 
However, this problem is still not efficiently solvable, as the following theorem shows.
\begin{theorem} \label{thm:samplednph}
It is NP-hard to find $\revset_2 \subseteq \revset$ such that $\sobjfn(\revset_2)$ is maximized, even when $\numsamples = 3$. 
%\ns{over what? I think you should expand on this statement just a little bit.} \sj{Reworded, and added a reminder about the size of $\revset_2$ above. Also added a few sentences in the previous section about feasibility for the loads and $\revset_2$.}
\end{theorem}
%The proof is by reduction from 3-Dimensional Matching, and is presented in Appendix~\ref{apdx:samplednph}.
\begin{proof}[Proof sketch]
We reduce from 3-Dimensional Matching~ \cite{karp1972reducibility}, which asks if there exists a way to select $k$ tuples from a set $T \subseteq X \times Y \times Z$ where $|X| = |Y| = |Z| = k$ such that all elements of $X$, $Y$, and $Z$ are selected exactly once. We construct $3$ samples of second-stage papers corresponding to $X$, $Y$, and $Z$ respectively, and construct reviewers corresponding to elements of $T$. These reviewers have $1$ similarity with the papers in their tuple, and $0$ similarity with all other papers. Thus, checking if there exists a choice of $\revset_2$ which gives full expected similarity in the second stage would require solving 3-Dimensional Matching. We add additional reviewers and papers to ensure that this choice of $\revset_2$ is optimal over both stages. 
\end{proof}
The full proof is presented in Appendix~\ref{apdx:samplednph}.

Since it is NP-hard to find the optimal $\revset_2$ even when estimating the objective by sampling three random choices of $\papset_2$, this suggests that the original objective $\objfn$ may be hard to optimize efficiently. Therefore, in the two-phase assignment setting, we instead look for efficient approximation algorithms.

\section{Our Approach: Random Split} \label{sec:approach}
Our proposed approach for finding a two-stage assignment is extremely simple: choose $\revset_2$ uniformly at random  (i.e., $\revset_2 \sim \mathcal{U}_{\stagetwofrac \numpap}(\revset)$). We refer to this as a ``random split'' of reviewers into the two review stages. %Intuitively, this follows from the idea that there are many substitutes for the ideal reviewer-paper pairs so that even when a paper's best reviewer is randomly assigned to the other stage, the loss in total similarity is not too large. 

In the two-phase setting, random split is an efficient approximation algorithm for the problem of optimizing $\objfn$, which is likely difficult (as shown in Section~\ref{sec:hard}). Because random split does not execute $\objfn$, it produces a two-stage paper assignment without needing to estimate $\objfn$ by sampling. 

In the conference experiment design setting, our proposed random-split strategy corresponds to a uniform random choice of reviewers for the experimental condition. Recall that in this setting, assigning reviewers to conditions uniformly at random is already a common experimental setup. The performance of random split on $\objfn$ therefore indicates how well this common setup performs in terms of the expected assignment similarity. 
%As defined in Section~\ref{sec:problem}, the two-stage paper assignment problem only evaluates the reviewer split in terms of the assignment similarity. However, since random split already maximizes the statistical effectiveness, the performance of random split on the two-stage paper assignment problem will indicate how well both ends of this tradeoff can be optimized simultaneously. 
%The framework of a two-stage review process with a random reviewer split also encompasses controlled experiments on the paper review process. For example, at WSDM 2017, all papers were reviewed under both single-blind and double-blind conditions with reviewers randomly split between conditions~\cite{Tomkins12708}. In Appendix~\ref{appdx:controlledexp}, we elaborate on this connection and provide additional experiments in a related setting.

In our theoretical results, we often refer to the {\it suboptimality of random split}, defined as the suboptimality of $\revset_2$ chosen via random split when $\papset_2$ is chosen uniformly at random:
\begin{align}
    \valfn^*(\papset_2) - \valfn(\revset_2, \papset_2), \qquad \qquad \text{where~~} \papset_2 \sim \mathcal{U}_{\stagetwofrac\numpap}(\papset), \revset_2 \sim \mathcal{U}_{\stagetwofrac\numpap}(\revset) \label{eq:suboptrs}.
\end{align}
Recall from Section~\ref{sec:problem} that $\valfn^*(\papset_2)$ is the mean similarity of the \optterm{} assignment given second-stage papers $\papset_2$ and that $\valfn(\revset_2, \papset_2)$ is the mean similarity of the optimal assignment given second-stage reviewers and papers $\revset_2, \papset_2$. Additionally, many of our results evaluate the expected mean similarity under random split:
\begin{align*}
    \mathbb{E}_{\revset_2 \sim \mathcal{U}_{\stagetwofrac\numpap}(\revset)} \left[\objfn(\revset_2)\right] = \mathbb{E}_{\revset_2 \sim \mathcal{U}_{\stagetwofrac\numpap}(\revset), \papset_2 \sim \mathcal{U}_{\stagetwofrac\numpap}(\papset)}\left[\valfn(\revset_2, \papset_2)\right]. 
\end{align*}

In the following subsections, we first elaborate on the good performance random split displays empirically before showing that there exist cases where random split performs very poorly. 

\subsection{Empirical Performance} \label{sec:approachexp}
As introduced earlier in Figure~\ref{fig:perf}, random split performs very well in practice on four real conference similarity matrices. The first is a similarity matrix recreated using text-matching on ICLR 2018 data~\cite{xu2018strategyproof}. 
The second is constructed using reviewer bid data for an AI conference (conference 3) from PrefLib dataset MD-00002~\cite{MaWa13a}. The third (denoted Bid1) is a sample of the bidding data from a major computer science conference~\cite{meir2020market}. In both of these bidding datasets, we transformed ``yes,'' ``maybe,'' and ``no response'' bids into similarities of $1$, $0.5$, and $0.25$ respectively, as is often done in practice~\cite{shah2017design}. 
The fourth similarity matrix is constructed from the subject areas of ACM SIGIR 2007 papers and the subject areas of the past work of their authors (assumed to be the reviewers)~\cite{karimzadehgan2008multi}; we set the similarity between each reviewer and paper to be equal to the number of matching subject areas out of the $25$ total, normalized so that each entry of the matrix is in $[0, 1]$. In Appendix~\ref{apdx:experiments}, we present further empirical results including additional datasets. 
In Appendix~\ref{apdx:controlledexp}, we present additional empirical results particularly relevant to the conference experiment design setting.

We run several experiments, each corresponding to a choice of dataset and $\stagetwofrac$. Each experiment consists of $10$ trials, where in each trial we sample a random reviewer split and a set of second-stage papers. We then present the range of assignment values achieved across the trials as percentages of the \optterm{} assignments for each trial. The \optterm{} assignment for a trial is found by choosing the optimal assignment of reviewers across both stages after observing $\papset_2$. 
We set paper loads of $2$ in each stage (as done in AAAI 2021), and limit reviewer loads to be at most $6$ (a realistic reviewer load~\cite{shah2017design}). Since these datasets have excess reviewers, we choose $\revset_2$ to have size $\frac{\stagetwofrac}{1 + \stagetwofrac} \numrev$ so that the proportions of reviewers and papers in the second stage are equal.

%\ns{Regarding both fig 1a and 1b, give more details about the result. Something like --  For instance say that across all trials and all datasets, the random similarity achieved at least blah percent of the oracle optimal. The mean of the trials achieved at least blah percent. This is the key result/figure, so we should delve a bit more into it.}  \sj{Added some things, and could add the new 96\% comment from the intro (although it doesn't seem to add much value next to the 94\% comment here). I don't think the means add much either since it's almost the same as the mins.}

In Figure~\ref{fig:real}, $\papset_2$ is drawn uniformly at random in each trial (as in the problem formulation). We see that all trials of random split achieve at least $90\%$ of the \optterm{} solution's similarity on all datasets, with all trials on all but two experiments achieving at least $94\%$. %Similarly, the mean similarity is at least $91\%$ of the \optterm{} similarity on all instances, and at least $95\%$ on all but two instances. 
We see additionally that the randomness of the reviewer choice does not cause much variance in the value of the assignment, as there is at most a $5\%$ difference between the minimum and maximum similarity (as a percentage of \optterm{}) for each experiment. Note that this is true despite the fact that the similarity matrices of the different datasets are constructed in several different ways, indicating that random split is robust across methods of similarity construction. 

In Figure~\ref{fig:scores}, $\papset_2$ is chosen as a fixed set for all trials based on the actual review scores received by the papers at ICLR 2018~\cite{he2020openreview} (as review scores were not available for other datasets). We run trials where either the top-scoring papers or the messy-middle papers are given additional reviews. Since about $37\%$ of papers were accepted, we define the messy middle as the range of $\frac{\stagetwofrac}{1+\stagetwofrac}\numrev$ papers centered on the $63$rd-percentile paper when ordered by score. These are sets of papers that a conference may potentially want to assign additional reviewers to. In all cases, random split shows consistently good performance, similar to when $\papset_2$ was drawn uniformly at random. All trials achieve at least $95\%$ of the \optterm{} similarity, with at most a $2\%$ difference between the minimum and maximum for each experiment. This suggests that the good performance of random reviewer split naturally holds in these practical cases.

\subsection{A Counterexample} \label{sec:approachcounterex}
The good results random split shows in practice may be somewhat surprising because random split does not perform well in all settings. The following theorem shows that %in some cases, it performs very poorly:
for any $\stagetwofrac$, there exist instances of the two-stage paper assignment problem where the suboptimality of random split~\eqref{eq:suboptrs} is $\Omega(1)$ in expectation.
\begin{theorem} \label{thm:generallb}
For any constant $\stagetwofrac \in [0, 1]$, there exists $\numpap_0$ such that for all $\numpap \geq \numpap_0$ where $\stagetwofrac \numpap \in \mathbb{Z}_+$, there exist instances of the two-stage paper assignment problem where the suboptimality of random split is at least $\frac{\stagetwofrac^4}{(1+\stagetwofrac)^3}$ in expectation.
%For any $\stagetwofrac$, there exist instances of the two-stage paper assignment problem with arbitrarily large $\numpap$ where the suboptimality of random split is at least $\Omega(1)$ in expectation.
%\ns{What does arbitrarily large $n$ mean when you are anyways considering $\Omega$ with respect to $n$? Are you saying that the quantity which asymptotes in the $\Omega$ notation is $n$?} \sj{In Thm 4 in the next section, the purpose of that specific phrase is to handle the fact that we do not actually have an instance for every value of n, only for certain values of n (that can be arbitrarily large). Here we have instances for any n and bound the suboptimality for all sufficiently large n, which is encapsulated in the $\Omega$ (so this clause should just be removed from the thm?). For Thm 4, my understanding (which may be wrong) is that: we provide a method of constructing problem instances for various values of n, write the suboptimality of these as a function of $n$, and then check which $\Omega$ the function is; so the ``arbitrarily large n'' indicates which n we can construct instances for, whereas the ``quantity that asymptotes'' just indicates how to write it in $\Omega$ notation. }
\end{theorem}
\begin{proof}[Proof sketch]
Consider $\stagetwofrac=1$. We construct a similarity matrix where every reviewer has similarity $1$ with only $1$ paper, and all papers have similarity $1$ with only $2$ reviewers. The optimal reviewer split puts the two good reviewers for each paper in separate stages and always achieves a mean similarity of $1$. Random split puts both good reviewers in the same stage with at least constant probability for each paper, giving a constant mean similarity $< 1$.
\end{proof}
The full proof is presented in Appendix~\ref{apdx:generallb}.

Note that the above lower bound on the objective value of random split holds even in the easy case of $\stagetwofrac = 1$, where the problem could be solved simply through standard paper assignment methods. This case is particularly relevant in the conference experiment setting, where all papers are commonly reviewed under both conditions (as in the WSDM 2017 experiment). %Therefore, the explanation for the good performance of random split is not obvious even when $\stagetwofrac = 1$.

Although the above lower bound demonstrates that random split cannot hope to do well in general, the constructed example is unrealistic for real conferences. %In practice, real conference similarity matrices are more robust to perturbations in the assignment, as reviewer expertise is more spread out over the set of papers. 
However, program chairs may understandably want some guarantee that a random reviewer split will work well for their conference before deciding to use it. Ideally, this guarantee should be given before the precise similarity matrix for the conference is known, since the similarities may not be known in full until late in the planning process. 

In the following sections, we provide such guarantees, thereby showing that the good performance of random split is not just an artifact of our specific datasets. We focus our attention on two sufficient conditions on the similarity matrix under which we show random split performs well. These conditions are natural for real similarity matrices, implying that random split will perform well for many real conferences, whether in the context of a two-phase review process or a conference experiment. Using these conditions, we provide actionable insights to program chairs based on simple properties of their conference's similarities that they may have intuition about. These insights are designed to be useful well before the full paper and reviewer sets are known.

\section{Condition 1: Low-Rank Similarity Matrix} \label{sec:cond1}
The first condition we consider is that the similarity matrix $\simmat$ has low rank $\rank$. This condition naturally arises in practice when reviewer-paper similarities are calculated from the number of subject area agreements between reviewers and papers; in such cases, the rank is no greater than the number of subject areas. For example, the SIGIR similarity matrix used in Figure~\ref{fig:perf} is constructed in this way and thus has rank no greater than 25 (the number of subject areas). In this section, we provide asymptotic upper and lower bounds on the suboptimality of random split for constant-rank similarity matrices.

\subsection{Theoretical Bounds}
We first provide an upper bound on the suboptimality of random split~\eqref{eq:suboptrs}. This shows that random reviewer splits perform well on constant-rank similarity matrices, including the SIGIR similarity matrix examined earlier. More precisely, the following theorem shows that if the similarity matrix $\simmat$ has constant rank $\rank$, the suboptimality of random split is at most $\widetilde{O}(\numpap^{-\frac{1}{2}})$ when $\rank = 1$, $\widetilde{O}(\numpap^{- \frac{1}{2} + o(1)})$ when $\rank = 2$, and $\widetilde{O}(\numpap^{- \frac{1}{\rank} + o(1)})$ when $\rank \geq 3$ with high probability.
\begin{theorem} \label{thm:lowrankub}
Consider any constants $\stagetwofrac \in [0, 1]$ and $\rank \in \mathbb{Z}_+$. There exists $\numpap_0$ and constants $C, \eta$ such that, for any $\numpap \geq \numpap_0$ where $\stagetwofrac \numpap \in \mathbb{Z}_+$ and for any similarity matrix $\simmat \in [0, 1]^{(1+\stagetwofrac)\numpap \times \numpap}$ of rank $\rank$, the suboptimality of random split is at most:
\begin{itemize}
    \item $C (\log\numpap)^{\eta} \numpap^{-\frac{1}{2}}$ if $\rank = 1$
    \item $C (\log\numpap)^{\eta} \numpap^{-\frac{1}{2} + \frac{1}{\log\log\numpap}}$ if $\rank = 2$
    \item $C (\log\numpap)^{\eta} \numpap^{-\frac{1}{\rank} + \frac{1}{\log\log\numpap}}$ if $\rank \geq 3$
\end{itemize}
with probability at least $1 - \frac{1}{\numpap}$ 
(where $\log$ indicates the base-$2$ logarithm). 
% over the random choices of both $\revset_2$ and $\papset_2$.
%For any $\stagetwofrac \in [0, 1]$, if the similarity matrix $\simmat$ has constant rank $\rank$, the suboptimality of random split is at most $\widetilde{O}(\numpap^{-\frac{1}{2}})$ when $\rank = 1$, $\widetilde{O}(\numpap^{- \frac{1}{2} + o(1)})$ when $\rank = 2$, and $\widetilde{O}(\numpap^{- \frac{1}{\rank} + o(1)})$ when $\rank \geq 3$ with probability at least $1 - \frac{1}{\numpap}$ over the random choices of both $\revset_2$ and $\papset_2$. 
\end{theorem}
\begin{proof}[Proof sketch]
By Lemma 4 of~\cite{rothvoss2014direct}, a rank $\rank$ similarity matrix $\simmat \in [0, 1]^{\numrev \times \numpap}$ can be factored into vectors $u_{\adrev} \in \mathbb{R}^\rank$ for each $\adrev \in \revset$ and $v_{\adpap} \in \mathbb{R}^\rank$ for each $\adpap \in \papset$ such that $\simmat_{\adrev \adpap} = \langle u_{\adrev}, v_{\adpap} \rangle$, $||u_{\adrev}||_2 \leq \rank^{1/4}$, and $||v_{\adpap}||_2 \leq \rank^{1/4}$. We cover the $\rank$-dimensional ball containing all paper vectors with smaller cells, and consider a reviewer to be in one of these cells if the \optterm{} assignment (given $\papset_2$) assigns it to a paper in that cell. Using a concentration inequality on the number of reviewers and papers in each cell in each stage, we can upper bound the number of reviewers that we cannot match to papers within the same cell. We then increase the size of the cells and attempt to match the remaining reviewers in this way, continuing until all reviewers are matched. We upper bound the suboptimality of the resulting assignment by the L2 distance between a reviewer's assigned paper and the paper they are assigned by the \optterm{} assignment. 
\end{proof}
The constants $C$ and $\eta$ may depend on $\rank$, which is itself assumed to be constant. The full proof is presented in Appendix~\ref{apdx:lowrankub}.

For constant-rank similarity matrices, the suboptimality diminishes as $\numpap$ grows, unlike when the rank of the similarity matrix is unrestricted. 
Conceptually, our proof technique of finding a minimum-distance matching between two samples of points resembles the optimal transport problem solved when finding the Wasserstein distance between a probability distribution and its empirical measure. Thus, our upper bounds nearly match those found in the literature on the expected empirical $1$-Wasserstein distance for continuous measures (see~\cite{panaretos2019statistical} and references therein). 
%\ns{I have no idea what is happening here.} \sj{I added extra details to clarify. The purpose was to add an interesting connection to existing work, but it could be taken out.}
%Our upper bounds nearly match the bounds on the expected empirical $1$-Wasserstein distance for continuous measures (see~\cite{panaretos2019statistical} and references therein); the conceptual connection to the empirical Wasserstein distance can be seen from the proof. %We conjecture that with a more careful analysis, the upper bounds can be tightened to $O(\numpap^{-2/\rank})$, matching the lower bounds given earlier that mirror the squared empirical $1$-Wasserstein distance.

We now complement the above upper bound with lower bounds on the suboptimality of random split~\eqref{eq:suboptrs} for constant rank similarity matrices. The following theorem shows that, for similarity matrices of constant rank $\rank$, the suboptimality of random split is $\Omega(\numpap^{-1/2})$ in expectation and $\Omega(\numpap^{-2/\rank})$ with high probability. 
\begin{theorem} \label{thm:lowranklb}
Suppose $\stagetwofrac = 1$. For any constant $\rank \in \mathbb{Z}_+$, there exists $\numpap_0$ and constants $C, \zeta$ such that for all $\numpap \geq \numpap_0$:
\begin{enumerate}[label=(\alph*)]
\item There exist instances of the two-stage paper assignment problem with similarity matrices $\simmat \in [0, 1]^{2\numpap \times \numpap}$ of rank $\rank$ such that the suboptimality of random split is at least $C \numpap^{-1/2}$ in expectation.
\item There exist instances of the two-stage paper assignment problem with similarity matrices $\simmat \in [0, 1]^{2\numpap \times \numpap}$ of rank $\rank$ such that the suboptimality of random split is at least $C \numpap^{-2/\rank}$ with probability $1 - \zeta e^{-\numpap/10}$. 
%\item There exist instances of the two-stage paper assignment problem with similarity matrices $\simmat$ of rank $\rank$ such that the suboptimality of random split is at least  $\Omega(\numpap^{-1/2})$ in expectation.
%\item There exist instances of the two-stage paper assignment problem with a similarity matrices $\simmat$ of rank $\rank$ such that the suboptimality of random split is at least $\Omega(\numpap^{-2/\rank})$ with probability $1 - O(e^{- 0.1 \numpap})$. 
\end{enumerate}
\end{theorem}
\begin{proof}[Proof sketch]
(a) We construct $\rank$ groups of reviewers and papers, where all reviewers and papers in the same group have similarity $1$ with each other and similarity $0$ with all other reviewers/papers. The first group contains $\frac{\numpap}{2}$ papers and $\numpap$ reviewers. The optimal reviewer split puts half of each group's reviewers in each stage and assigns all reviewers to papers with similarity $1$. By an anti-concentration inequality, with constant probability, at least $\Omega(\sqrt{\numpap})$ reviewers in the first group cannot be assigned to a paper in their group under random split. 

(b) We construct a vector in $\mathbb{R}^{\rank}$ for each reviewer and each paper and set the similarity between that reviewer and that paper to be the inner product of their corresponding vectors. We place one paper vector and two reviewer vectors at each point in an evenly-spaced grid throughout the cube $[0, 1/\sqrt{\rank}]^{\rank}$. The resulting similarity matrix has rank $\rank$.  The optimal assignment assigns the two reviewers at each point to the paper at that point. With high probability, random split places $\Omega(\numpap)$ pairs of reviewer vectors into the same stage. One of each of these reviewer pairs must be assigned to a paper at a different point, which is at least $\Omega(\numpap^{-1/\rank})$ away in L2 distance. The suboptimality of the resulting assignment can be written in terms of the total squared L2 distance between each reviewer and their assigned paper, giving the stated bound.
\end{proof}
The constants $C$ and $\zeta$ may depend on $\rank$, which is itself assumed to be constant. The full proof is presented in Appendix~\ref{apdx:lowranklb}.

\subsection{Interpretation of Results}
As discussed earlier in this section, certain methods of constructing similarities (such as counting subject area agreements) may inherently lead to low-rank similarity matrices. If a conference is using such a method, the results in this section provide guarantees to the program chairs that random split will perform well, particularly if the rank of the matrix is low compared to the number of papers and reviewers. Alternatively, program chairs may be able to estimate that their reviewers and papers can be grouped into a small number of communities with little variation within them, in which case the similarity matrix may also be low rank.

\section{Condition 2: High-Value, Large-Load Assignment} \label{sec:cond2}
A natural condition on the similarity matrix to consider is that each paper has a large number $\loadscale$ of reviewers with high similarity for that paper. It turns out that this condition is insufficient for guaranteeing good performance of random split, since the same group of $\loadscale$ reviewers could have high similarity with all papers, thus satisfying this condition without changing the assignment value by much (since we can only assign these reviewers to a few papers). In this section, we consider a condition on the similarity matrix that is similar in spirit: the existence of a high-value assignment (in terms of total similarity) on the full reviewer and paper sets where each paper is assigned a large number $(1 + \stagetwofrac)\loadscale$ of reviewers. Our proposed condition handles the issue with the naive ``large number of reviewers'' condition by requiring that the high-value reviewers for each paper can all be simultaneously assigned. 

In the following subsections, we first provide theoretical guarantees about the performance of random split under this condition. We then demonstrate that this condition helps to explain the good performance of random split on the real similarity matrices presented earlier.

\subsection{Theoretical Bounds} \label{sec:cond2theory}
The first result of this section gives a lower bound on the expected value of random split in terms of the value of a single, large-load  assignment. All results in this section still hold if there are excess reviewers (i.e., if $\numrev \geq (1 + \stagetwofrac) \numpap$ and $\revset_2 \sim \mathcal{U}_{\frac{\stagetwofrac}{1+\stagetwofrac}\numrev}(\revset)$).
\begin{theorem} \label{thm:oneroundub}
Consider any $\loadscale \in [10,000]$ and $\stagetwofrac \in \left\{ \frac{1}{100}, \dots, \frac{100}{100} \right\}$ such that $\stagetwofrac \loadscale \in \mathbb{Z}_+$. 
If there exists an assignment $\adassign^{(\loadscale)} \in \mathcal{M}(\revset, \papset; \loadscale, (1+\stagetwofrac )\loadscale)$ with mean similarity $\scaledavgval$, 
choosing $\revset_2$ via random split gives that
\begin{align*} 
\mathbb{E}_{\revset_2}\left[\objfn(\revset_2)\right] \geq \scaledavgval \Bigg[  1
- \sqrt{\frac{\stagetwofrac}{2 \pi (1+\stagetwofrac )^2 \loadscale }}\left( 2 \sqrt{\frac{1}{1+\stagetwofrac}} 
+ \sqrt{1 -\stagetwofrac}  \right)
\Bigg]. 
\end{align*}
A similar bound holds when $\stagetwofrac \loadscale$ is not integral, with some additional small terms due to rounding.
\end{theorem}
\begin{proof}[Proof sketch]
We construct assignments with paper and reviewer loads of at most $\loadscale$ in stage one and at most $\stagetwofrac \loadscale$ in stage two using the reviewer-paper pairs assigned by $\adassign^{(\loadscale)}$. We drop any extra assignments at random so that no reviewers and papers are overloaded, and assume any pairs that must be assigned from outside of $\adassign^{(\loadscale)}$ have similarity $0$. From within each of these larger assignments, we can find an assignment with paper and reviewer loads of $1$ with at least the same mean similarity. The expected mean similarity of these assignments can be written as the expectation of a function of binomial random variables. Approximating these by normal random variables and checking via simulation that this is in fact a lower bound for the stated values of $\stagetwofrac$ and $\loadscale$, we get the stated bound.
\end{proof}
The more general version of the bound and the full proof are stated in Appendix~\ref{apdx:oneroundub}.

The above bound works well when the reviewer-paper pairs in the large-load assignment are all nearly equally valuable. However, it cannot take advantage of the fact that certain reviewers may be extremely valuable for a certain paper and can be prioritized for assignment to that paper when possible. The next result uses additional information about the value of an assignment with smaller loads, along with a large-load assignment disjoint from the small assignment, to make use of these highly valuable reviewer-paper pairs in the case where $\stagetwofrac = 1$. Recall from Section~\ref{sec:approachcounterex} that $\stagetwofrac=1$ is still not an easy case for random split in general and is particularly relevant for the conference experiment setting.

\begin{theorem} \label{thm:tworoundub}
Suppose $\stagetwofrac=1$, and consider any $\loadscale \in [10,000]$ such that $\frac{\loadscale}{4} \in \mathbb{Z}_+$.  Suppose there exists an assignment $\adassign^{(1)} \in \mathcal{M}(\revset, \papset; 1, 2)$ with mean similarity $\optavgval$. 
Suppose there also exists an assignment $\adassign^{(\loadscale)} \in \mathcal{M}(\revset, \papset; \loadscale, 2\loadscale)$ with mean similarity $\scaledavgval$ that does not contain any of the pairs assigned in $\adassign^{(1)}$. 
Then, choosing $\revset_2$ via random split gives that
\begin{align*}
\mathbb{E}_{\revset_2}\left[\objfn(\revset_2)\right] \geq \frac{3}{4} \optavgval + \left(1 - \frac{1.44}{\sqrt{\loadscale}} 
\right) \frac{1}{4} \scaledavgval.
\end{align*}
A similar bound holds when $\frac{\loadscale}{4}$ is not integral, with some additional small terms due to rounding. 
\end{theorem}
\begin{proof}[Proof sketch]
We first attempt to assign as many pairs as possible from within $\adassign^{(1)}$; in expectation we can assign $\frac{3}{4}$ of them. Among the remaining reviewers and papers, we attempt to construct assignments with paper and reviewer loads of $\frac{\loadscale}{4}$ in both stages from within the reviewer-paper pairs assigned by $\adassign^{(\loadscale)}$. This is done in a similar way as in Theorem~\ref{thm:oneroundub}. 
\end{proof} 
The more general version of the bound and the full proof are stated in Appendix~\ref{apdx:tworoundub}.

If we consider $\adassign^{(1)}$ as the optimal assignment and assume that $\loadscale$ is divisible by $4$, we get an approximation ratio (between the random split assignment and \optterm{} assignment's similarities) of $\frac{3}{4} + \frac{\gamma_{\loadscale}}{4} \left(1 - \frac{1.44}{\sqrt{\loadscale}} \right)$ where $\gamma_{\loadscale} = \frac{\scaledavgval}{\optavgval}$. With $\loadscale = 8$, we achieve an approximation ratio of at least $\frac{3}{4} + \frac{\gamma_8}{8}$. Additionally, if $\gamma_{\loadscale} \to 1$ as $\numpap$ grows for any $\loadscale = \omega(1)$, the suboptimality of random split~\eqref{eq:suboptrs} approaches $0$. For example, this means that the suboptimality of random split approaches $0$ as $\numpap$ grows if the mean similarity of an assignment with paper loads of $\log \numpap$ improves faster than the mean similarity of the optimal assignment. 
%\ns{ratio between what and what? Also, previously we had said we want to bound the difference but now we are looking at the ratio?} \sj{This result is easier to interpret non-asymptotically as a ratio. We define the suboptimality as the difference but I don't think this implies that we don't care about the ratio. }

\subsection{Empirical Evaluation} \label{sec:cond2exp}
We now show the performance of these bounds on our real conference datasets in order to evaluate the extent to which they explain the good performance of random split. We use three of the conference datasets introduced earlier with $\stagetwofrac=1$. In Appendix~\ref{apdx:experiments}, we evaluate the bounds on additional datasets (including the SIGIR dataset). 
On PrefLib3 and Bid1, the problem is infeasible with paper and reviewer load constraints of $1$ since $\numrev < 2 \numpap$, so we modify the datasets by splitting each reviewer into $3$ copies as follows. For each paper, we arbitrarily give one of the copies the same similarity as the original reviewer and give the other copies similarity $0$. In this way, the similarity of the optimal assignment on this modified dataset is no greater than the similarity of the optimal assignment on the original dataset.

In Figure~\ref{fig:bounds}, we vary the value of the parameter $\loadscale$ (indicating the loads of the  assignment $\adassign^{(\loadscale)}$) and show the bounds of Theorem~\ref{thm:oneroundub} and Theorem~\ref{thm:tworoundub} as compared to the estimated expected value of random split. The estimated expected value is averaged over $10$ trials with the standard error of the mean shaded, although it is sometimes not visible because it is small. We see that on these datasets, the bound of Theorem~\ref{thm:oneroundub} performs best for low values of $\loadscale$ and not very well for higher values, likely due to the presence of a few ``star'' reviewers for each paper which hold a lot of the value. By making use of extra information about the values of these reviewers, the bound of Theorem~\ref{thm:tworoundub} achieves a high fraction of the actual random split value. Although this bound is maximized at large values of $\loadscale$ on these datasets, it is close to its maximum even with reasonably low values of $\loadscale$. For example, on ICLR, the lower bound achieves $86\%$ of the estimated expected value of random split with $\loadscale=8$.  This indicates the good performance of random split is explained well by the presence of just a few good reviewers per paper that can be simultaneously assigned. 

\begin{figure*}[t!] 
    \centering
    \begin{subfigure}{0.65\textwidth}\includegraphics[width=1\textwidth]{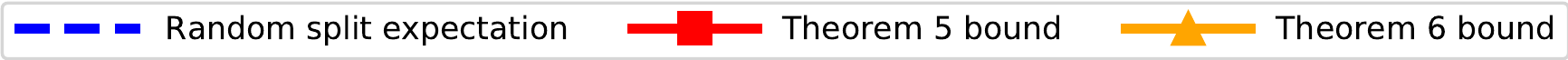} \end{subfigure} \\
    \begin{subfigure}{0.32\textwidth}\includegraphics[width=1\textwidth]{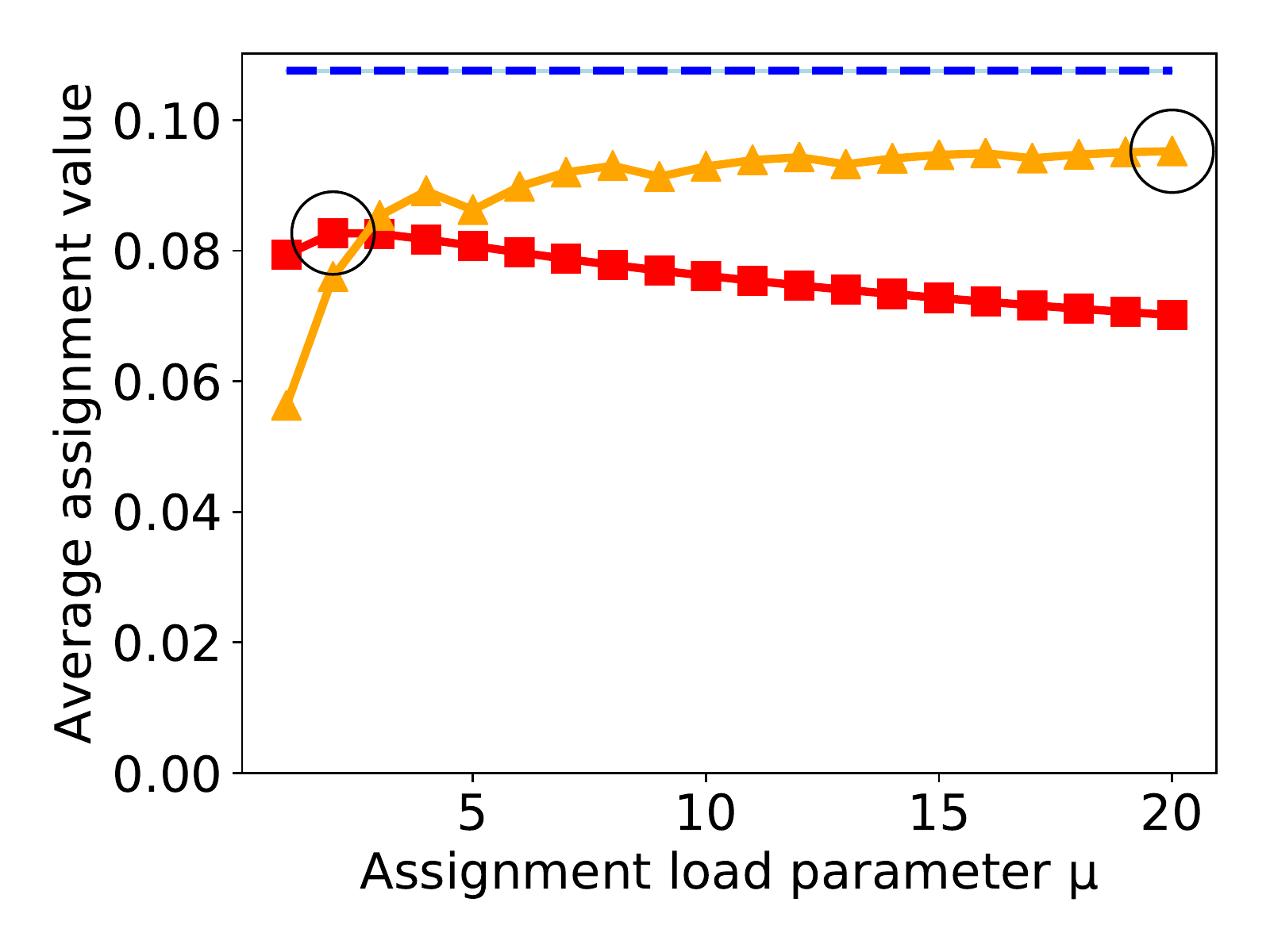}\caption{ICLR}\label{fig:bound1} \end{subfigure}
    \begin{subfigure}{0.32\textwidth}\includegraphics[width=1\textwidth]{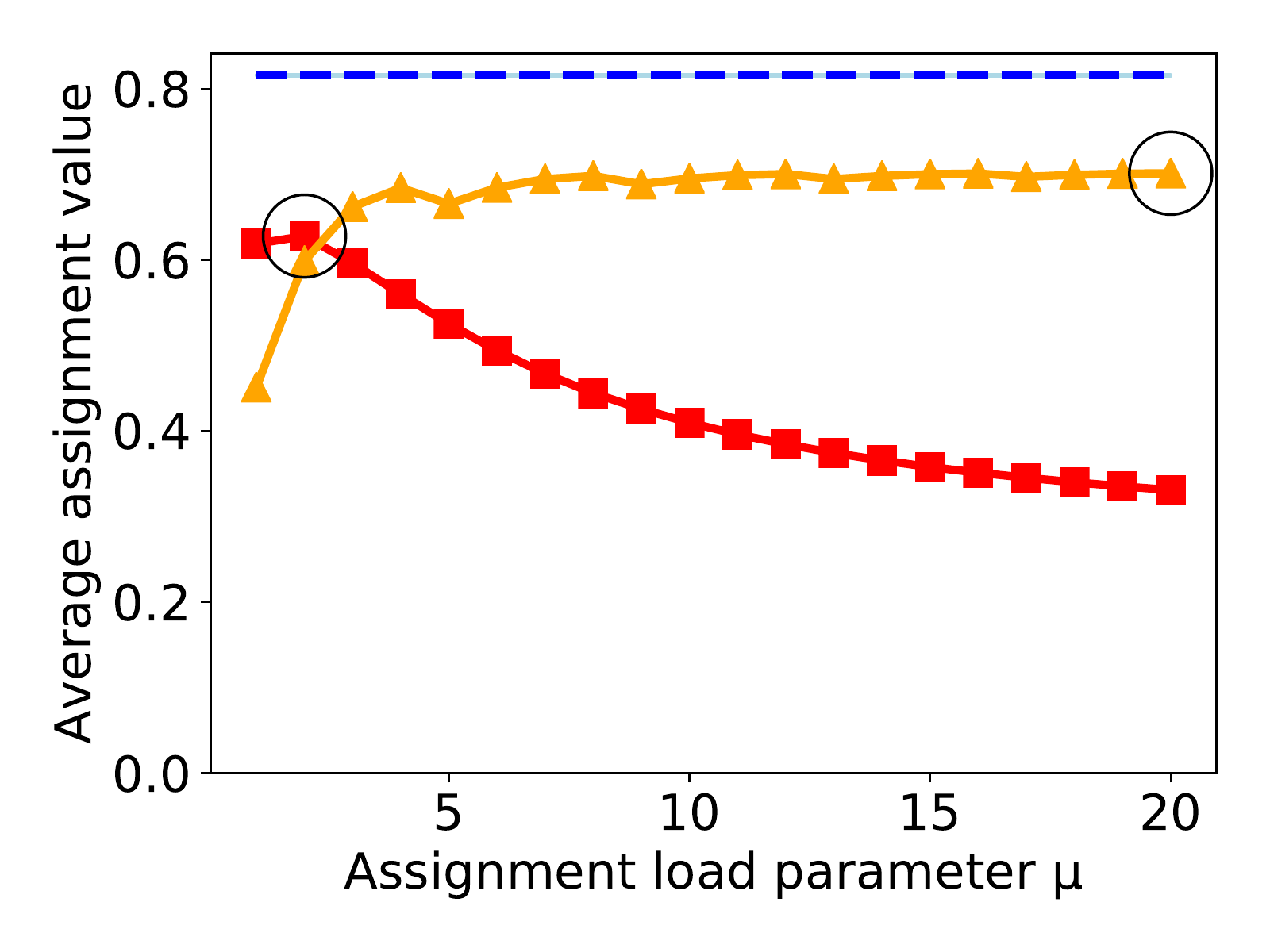}\caption{PrefLib3}\label{fig:bound2} \end{subfigure} 
    \begin{subfigure}{0.32\textwidth}\includegraphics[width=1\textwidth]{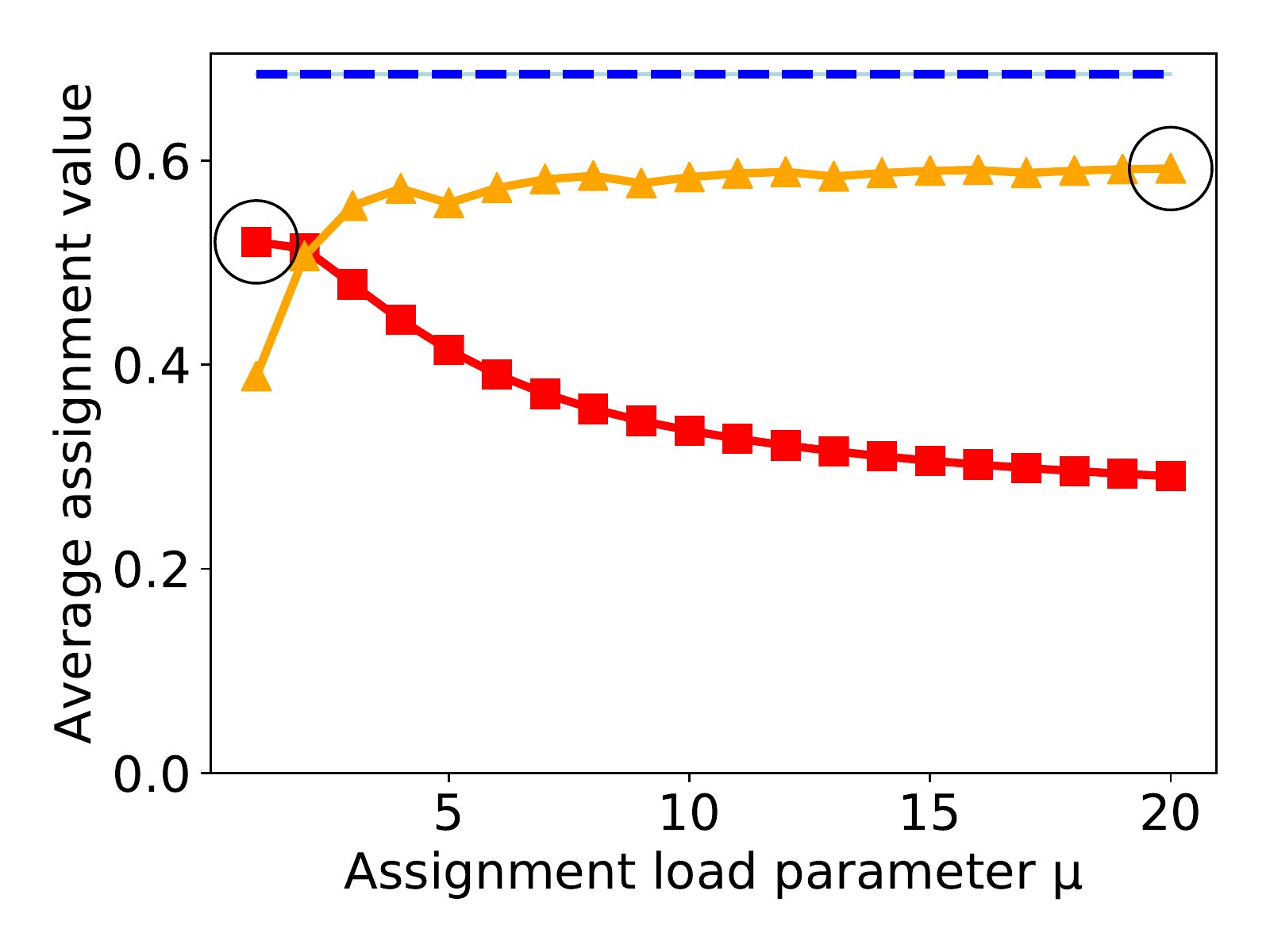}\caption{Bid1}\label{fig:bound3} \end{subfigure}
    \caption{Performance of the ``high-value large-load'' bounds on real conference datasets, $\stagetwofrac=1$. On the x-axis we vary the parameter $\loadscale$, which determines the loads of the  assignment $\adassign^{(\loadscale)}$ used in the bound. The best setting of $\loadscale$ for each bound is circled.} \label{fig:bounds}
\end{figure*}

\subsection{Interpretation of Results}
Although our results in this section are stated in terms of the precise values of high-load assignments, they can be interpreted by program chairs in a simple and practical way. Roughly, our results indicate that if several good reviewers can be \emph{simultaneously} assigned to each paper (as was the case for the three conference similarity matrices in Figure~\ref{fig:bounds}), random split will perform well. When considering the potential performance of randomly splitting reviewers, program chairs should consider the reviewer and paper pools they expect to have at their conference and make a judgement about how many good-quality reviewers they think could be assigned to each paper (if the reviewer loads are scaled up proportionately). For example, the program chairs of a large AI conference might be confident that the top several reviewers for each paper are about equally valuable (due to the depth of the reviewer pool) and could be assigned to each paper with only a modest loss in average review quality; this would imply that random split would perform very well for this conference.

\section{Conclusion} \label{sec:conc}
We showed that randomly splitting reviewers between two reviewing phases or two reviewing conditions produces near-optimal assignments on realistic conference similarity matrices. Our analysis of this phenomenon can help future program chairs make decisions about whether random split will work well for their conference's two-phase review process, based on their assessment of whether a few simple conditions are applicable to their case. 
In the setting of conference experiment design, our analysis allows program chairs to understand if running an experiment on their review process will significantly impact their assignment quality.

In addition, our results can potentially be further generalized to related reviewing models such as those of academic journals (which accept submissions on a rolling basis), or to other multi-stage resource allocation problems that involve matching resources based on similarities. 
For example, datacenters receiving a large batch of jobs may have to select some to run on various servers immediately and some to run later when additional servers have been freed, or hospitals may want to assign nurses to shifts based on expertise but without knowledge of which expertise will be most applicable in later shifts.

One limitation of our work is that while our empirical results demonstrate the effectiveness of the random-split strategy with real conference data, our theoretical results make the simplifying assumption that paper and reviewer loads are $1$, which is unrealistic for real conferences. However, we believe that incorporating this detail would not change our explanations for the good performance of random split. Another limitation is that we assume the set of papers requiring reviews in the second stage is drawn uniformly at random. Although this is a reasonable belief without further information in the two-phase setting, one direction for future work is to consider non-uniform distributions of second-stage papers and analyze if a form of random split still performs well there. 

Our work could potentially produce negative outcomes in the form of worse paper assignments if program chairs decide to use random split on an incorrect belief that their conference will fit our conditions. However, program chairs are required to make such decisions about how to perform the paper assignment anyway, so this is not a significant increase in risk. The use of random reviewer splits, as opposed to some alternate strategy where reviewers can self-select their stage, could also negatively impact reviewers with strong preferences over which stage they review in (e.g., due to schedule constraints). These preferences should ideally be taken into account along with the similarity of the resulting assignment when choosing the reviewer split; we leave this as an interesting direction for future work.
%We also only evaluate our assignments on the common sum-similarity objective, rather than any alternative objectives (such as fairness). It is possible that the use of random split has a negative impact on the fairness of the assignment.

%Although we focus on conference paper assignment in this paper, our results can likely be extended to similar reviewing models. For example, academic journals that solicit paper submissions throughout the year may want to know which of their reviewers they should save for submissions later in the year and which can be assigned immediately to incoming papers. Our results can also potentially be applied to other multi-stage resource allocation problems that involve matching resources based on similarities between them. Datacenters receiving a large batch of jobs may have to select some to run on various servers immediately and some to run later when additional servers have been freed, or hospitals may want to assign nurses to shifts based on expertise but without knowledge of which expertise will be most applicable in later shifts. 

\section*{Acknowledgments}
This work was supported by NSF CAREER awards 1942124 and 2046640, NSF grant CIF-1763734, NSF grants IIS-1850477 and IIS-1814056, and a Google Research Scholar Award.

\bibliographystyle{unsrt}
\bibliography{bibtex}

%%%%%%%%%%%%%%%%%%%%%%%%%%%%%%%%%%%%%%%%%%%%%%%%%%%%%%%%%%%%

~\\~\\

\appendix
\noindent{\LARGE \bf Appendices}

\begin{figure*}[h]
    \centering
    \begin{subfigure}{0.7\textwidth}\includegraphics[width=1\textwidth]{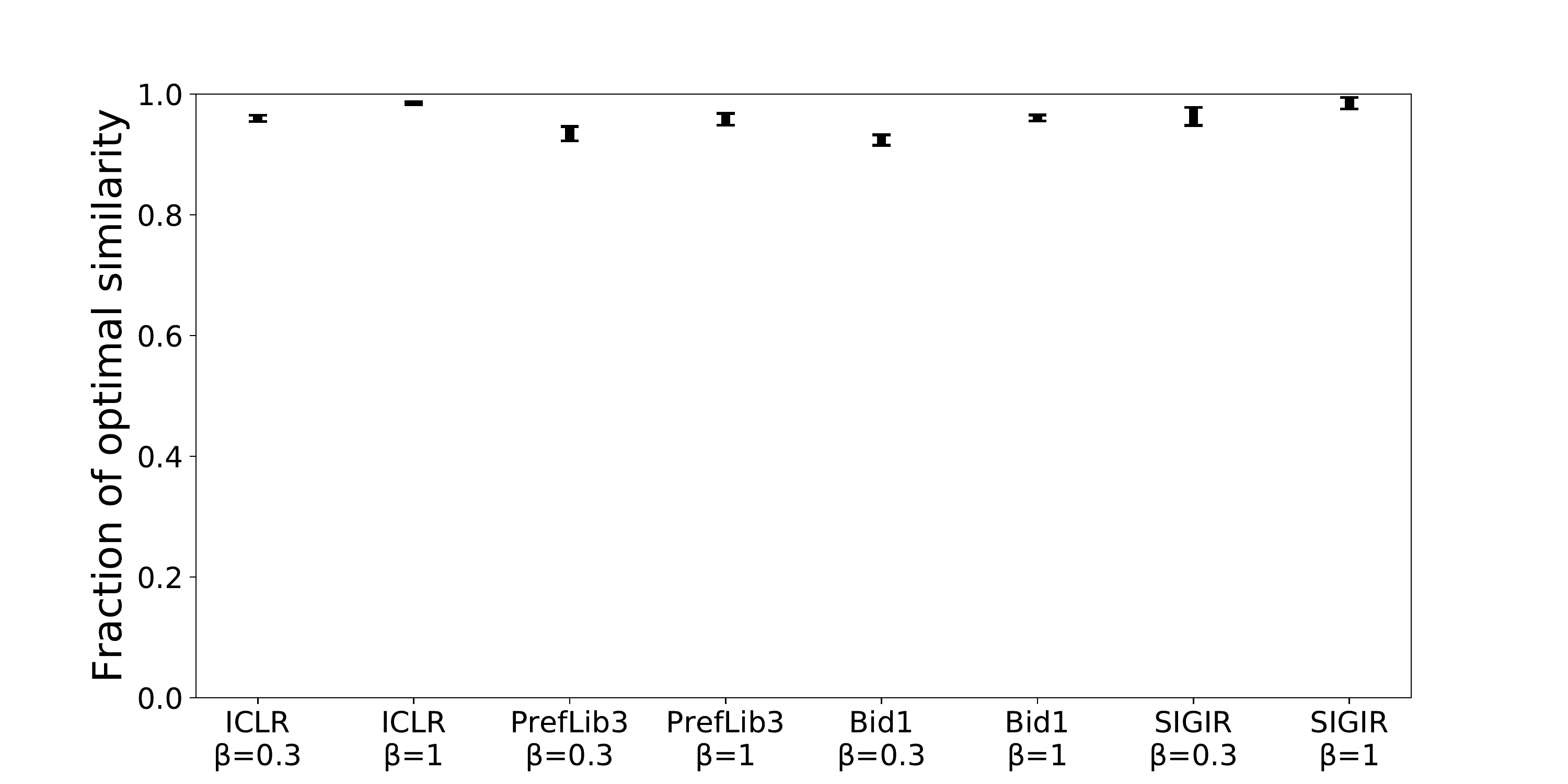}\caption{Additional values of $\stagetwofrac$}\label{fig:suppbeta} \end{subfigure} \\
    \begin{subfigure}{1\textwidth}
    \includegraphics[width=1\textwidth]{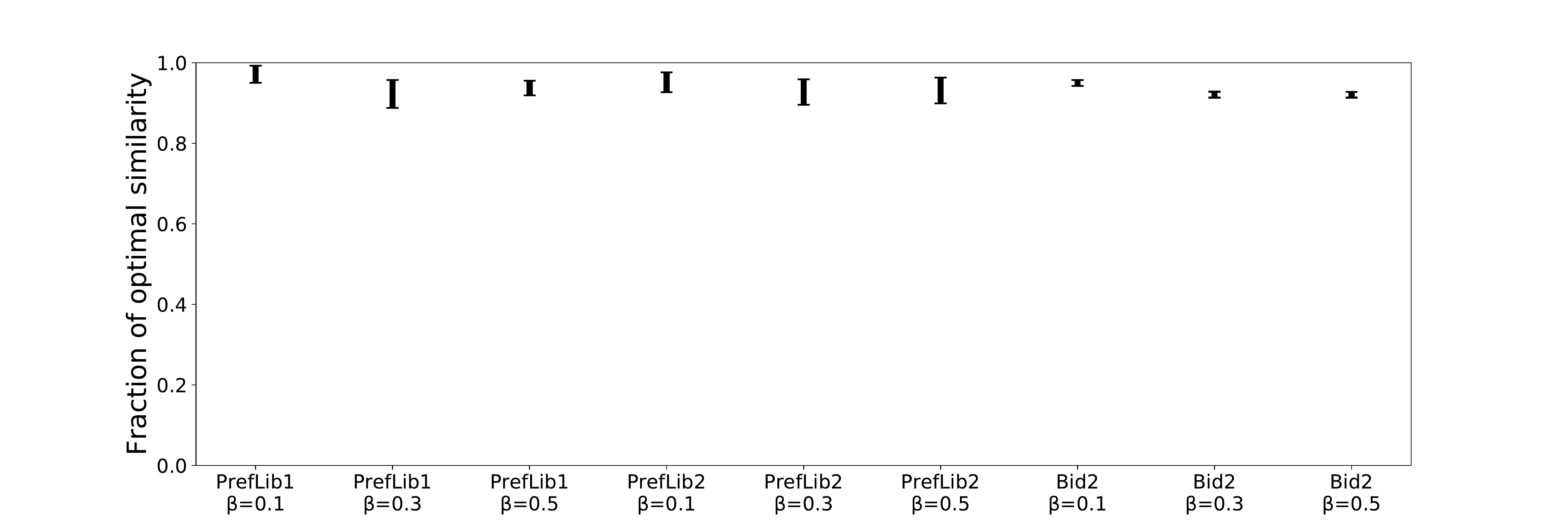}\caption{Additional similarity matrices}\label{fig:suppdata} \end{subfigure}
    \caption{Additional results showing ranges of values found over $10$ random reviewer splits.} \label{fig:supp}
\end{figure*}

\section{Additional Empirical Results} \label{apdx:experiments}
%\begin{figure*}
%    \centering
%    \includegraphics[width=1\textwidth]{images/supp_exp.pdf}
%    \caption{Range of values found over $10$ random reviewer splits on additional real similarity matrices.} \label{fig:supp}
%\end{figure*}

We first present empirical evaluations showing the performance of random split on additional values of $\stagetwofrac$ for the similarity matrices used in Section~$\ref{sec:approachexp}$ (Figure~\ref{fig:suppbeta}), as well as on additional similarity matrices constructed from bidding data (Figure~\ref{fig:suppdata}). Two of these additional similarity matrices were constructed using bid data for the two other AI conferences (conferences 1 and 2) from PrefLib dataset MD-00002~\cite{MaWa13a} with sizes $(\numpap=54, \numrev=31)$ and $(\numpap=52, \numrev=24)$ respectively. Another additional similarity matrix  (marked Bid2) was constructed from another sample of the bidding data from a major computer science conference~\cite{meir2020market} with size $(\numpap=1200, \numrev=300)$. As in the bidding datasets shown earlier, we transformed ``yes,'' ``maybe,'' and ``no response'' bids into similarities of $1$, $0.5$, and $0.25$ respectively.

We run several experiments, each corresponding to a choice of dataset and $\stagetwofrac$. Each experiment consists of $10$ trials, where in each trial we sample a random reviewer split and a set of second-stage papers, and report the range of assignment values found as percentages of the \optterm{} assignments for each trial. We set paper loads of $2$ in each stage, and limit reviewer loads to be at most $6$ for all datasets except PrefLib2 and Bid2, which limit reviewer loads to be at most $12$ (for feasibility). As in Section~\ref{sec:approachexp}, we draw $\revset_2$ uniformly at random with size $\frac{\stagetwofrac}{1+\stagetwofrac}\numrev$ and draw $\papset_2$ uniformly at random with size $\stagetwofrac\numpap$. In general, we see that random split performs very well on these datasets as well. We see that all trials of random split achieve at least $88\%$ of the \optterm{} solution's similarity on all datasets, with all trials on all but three experiments achieving at least $94\%$. The range of values on each experiment is generally small (at most $7\%$), with the largest ranges occuring on the small PrefLib datasets.

%Conference experiments are often run with $\stagetwofrac =1$, meaning that all papers are reviewed under both conditions. For example, this was the case in the WSDM 2017 experiment, as all papers were reviewed under both single-blind and double-blind conditions~\cite{Tomkins12708}. In Figure~\ref{fig:betaone}, we provide empirical results on the datasets introduced in Section~\ref{sec:approachexp} for the case where $\stagetwofrac=1$. As in Section~\ref{sec:approachexp}, we set paper loads of $2$ in each stage and limit reviewer loads to be at most $6$; we draw $\revset_2$ uniformly at random with size $\frac{\stagetwofrac}{1+\stagetwofrac}\numrev$. For each dataset, we then take $10$ samples of random reviewer split and show the range of assignment values achieved as a fraction of the \optterm{} assignment's value. On all datasets, all trials of random reviewer split achieve over $94\%$ of the \optterm{} assignment's total similarity with low variation. 

\begin{figure*}[t!] 
    \centering
    \begin{subfigure}{0.65\textwidth}\includegraphics[width=1\textwidth]{images/legend.pdf} \end{subfigure} \\
    \begin{subfigure}{0.45\textwidth}\includegraphics[width=1\textwidth]{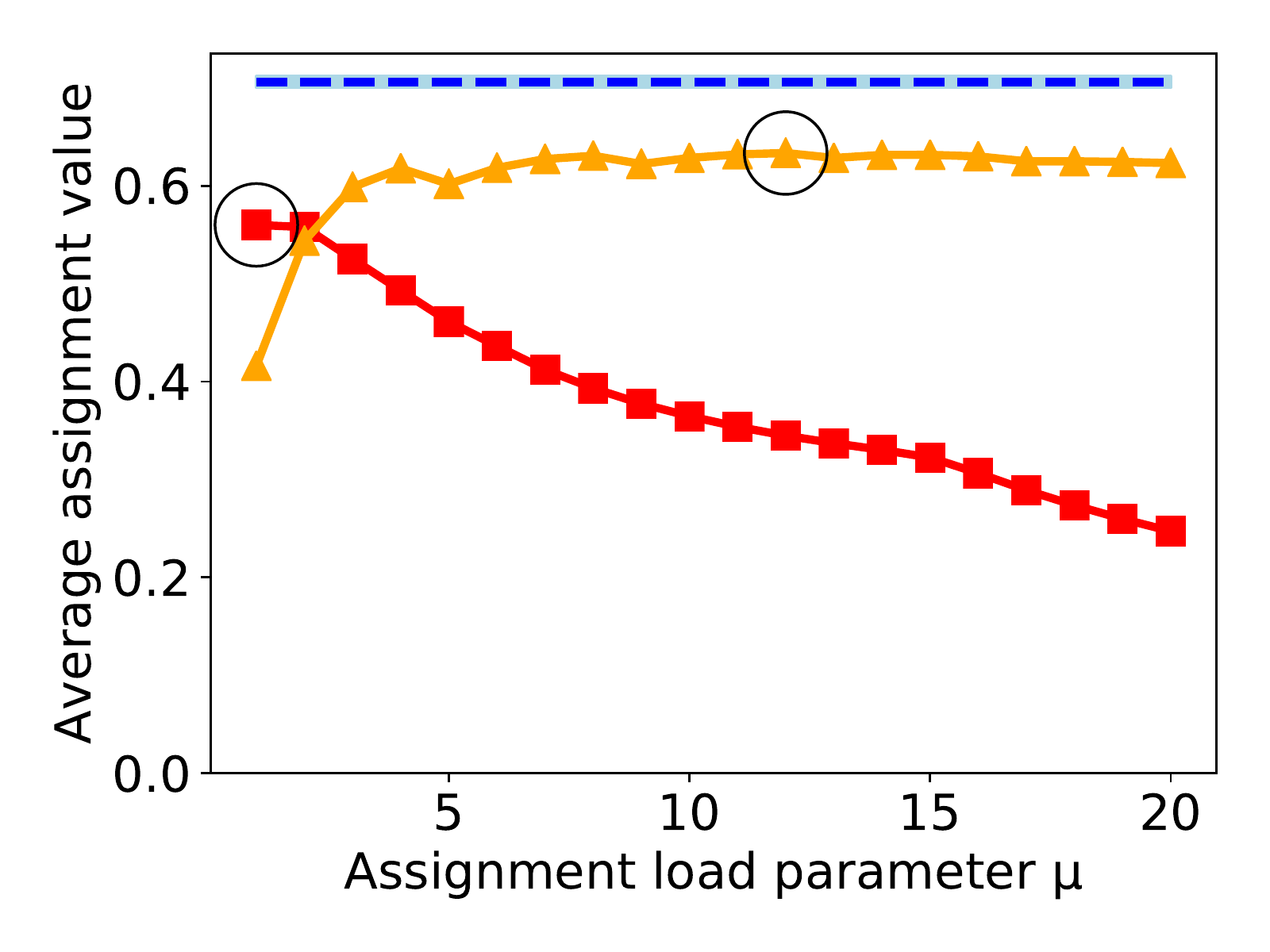}\caption{PrefLib1}\label{fig:suppbounds1} \end{subfigure} \quad
    \begin{subfigure}{0.45\textwidth}\includegraphics[width=1\textwidth]{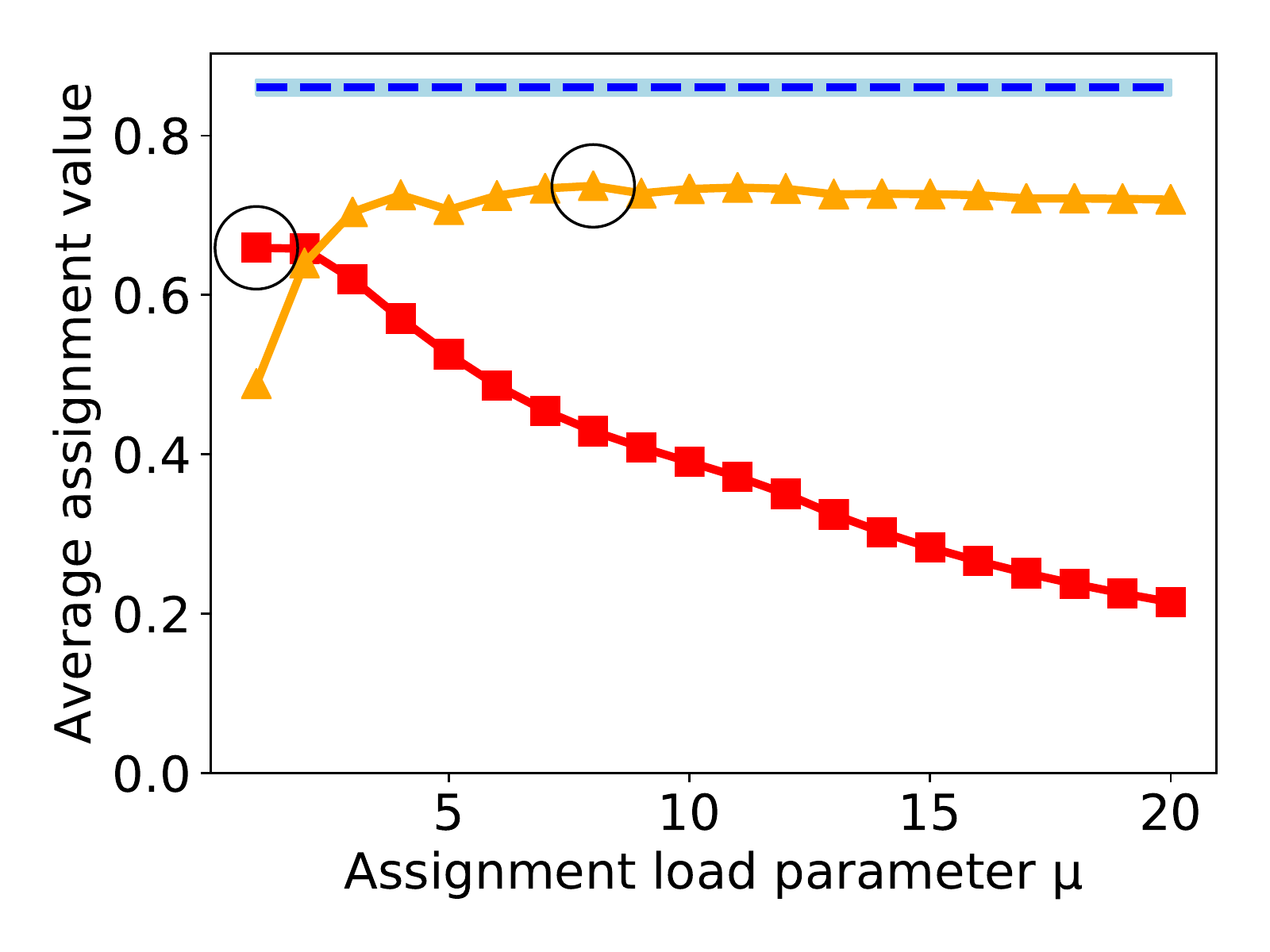}\caption{PrefLib2}\label{fig:suppbounds2} \end{subfigure} \\
    \begin{subfigure}{0.45\textwidth}\includegraphics[width=1\textwidth]{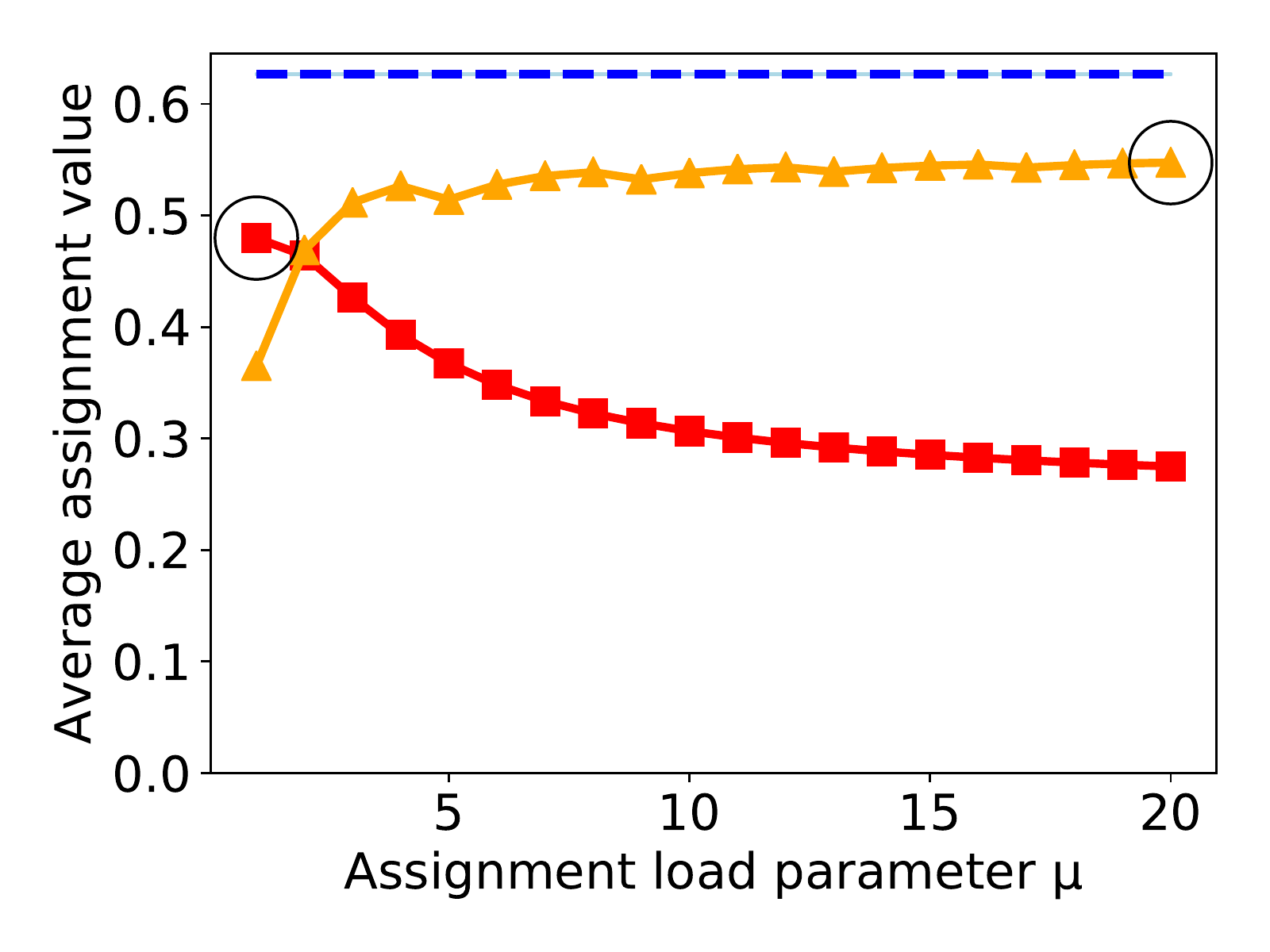}\caption{Bid2}\label{fig:suppbounds3} \end{subfigure} \quad   \begin{subfigure}{0.45\textwidth}\includegraphics[width=1\textwidth]{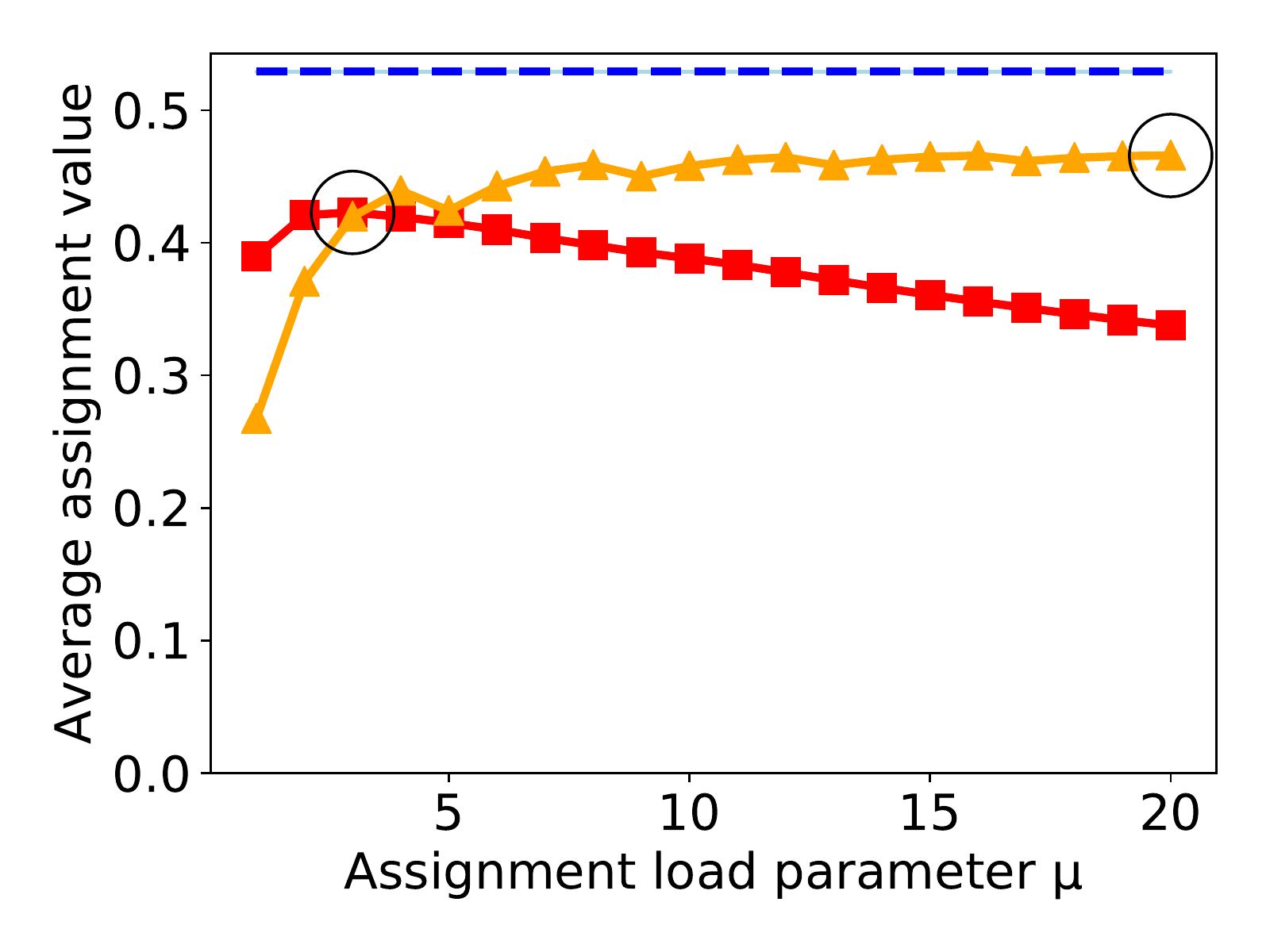}\caption{SIGIR}\label{fig:suppbounds4} \end{subfigure}
    \caption{Performance of Theorem~\ref{thm:oneroundub} and \ref{thm:tworoundub} bounds on additional real conference datasets, $\stagetwofrac=1$. The best setting of $\loadscale$ shown for each bound is circled.} \label{fig:suppbounds}
\end{figure*}

We additionally test the bounds of Section~\ref{sec:cond2} on these datasets as well as the SIGIR dataset to evaluate how well they explain the performance of random split. On PrefLib1, Preflib2, and Bid2, we scale up the number of reviewers by $4$, $5$, and $8$ respectively for feasibility, as described in Section~\ref{sec:cond2exp}. On the x-axis we vary the parameter $\loadscale$, which determines the loads of the  assignment $\adassign^{(\loadscale)}$ used in the bound. 

In Figure~\ref{fig:suppbounds}, we see similar results to those shown earlier. The Theorem~\ref{thm:oneroundub} bound performs best at low values of $\loadscale$. The Theorem~\ref{thm:tworoundub} bound performs better at higher values of $\loadscale$, although for some datasets a more moderate value of $\loadscale$ is better since the assignment value $\scaledavgval$ drops too quickly at higher $\loadscale$. From the Theorem~\ref{thm:tworoundub} bound, we see that the good performance of random split on these datasets is generally explained fairly well by the large-load, high-similarity assignment.

All empirical evaluations in this paper were run on a computer with $8$ cores and $16$ GB of RAM, running Ubuntu 18.04 and solving the LPs with Gurobi 9.0.2~\cite{gurobi}.

\section{Empirical Results for Paper-Split Variant} \label{apdx:controlledexp} 
In this section, we provide some additional empirical results that are particularly relevant to the conference experiment design setting.
Sometimes, conferences may not have the reviewing resources to provide a significant number of papers with two sets of reviews as part of an experiment. Instead, they may want to provide reviews to each paper under only one of the conditions. 
If the papers and reviewers are both split between conditions uniformly at random, this can be seen as a variant of our standard two-stage paper assignment problem where only papers in $\papset_1 = \papset \setminus \papset_2$ are assigned reviewers in stage one.

To test whether such experiments will still give high-similarity assignments in practice, we conduct additional empirical evaluations. The results of these experiments are shown in Figure~\ref{fig:papersplit}. For each dataset of those introduced in Section~\ref{sec:approachexp}, we take $10$ samples of random reviewer and paper splits where $|\revset_2| = \numrev / 2$ and $|\papset_2| = \numpap / 2$ so that half of the reviewers and papers are in each stage (i.e., each condition). We then find assignments in each stage with paper loads of $3$ and reviewer loads of at most $6$ (standard conference loads), and display the range of assignment values found as a fraction of the \optterm{} assignment's value. On all datasets, all trials of random reviewer split achieve over $75\%$ of the \optterm{} assignment's total similarity with low variation (at most $4\%$). On ICLR and SIGIR, all trials achieve over $90\%$ of the \optterm{} similarity. 
%These results show that controlled experiments where each paper is assigned to only one condition can be consistently performed without a large loss of assignment quality on real conference similarity matrices. 
Overall, the average assignment quality is slightly worse than in the standard model (where all papers are in stage one). This is likely because it is more difficult for reviewers to be assigned to their optimal papers when each paper is in only one of the two stages. 
%\ns{Let us be careful about drawing such conclusions. We don't have any justification to say that 75\% is not a large loss (I've heard from program chairs that they are usually ok with 90\%). Let's just give the objective numbers here and leave the interpretation of 75\% to the reader. Also, point out that it is roughly 90\% for ICLR and SIGIR etc. Essentially write a few more lines interpreting the plot 3b  (I think I had made a similar comment about he empirical resutls section in the main text, and if not, then please consider this comment for that section as well.)}\sj{Changed.}

\begin{figure*}
    \centering
    \includegraphics[width=0.45\textwidth]{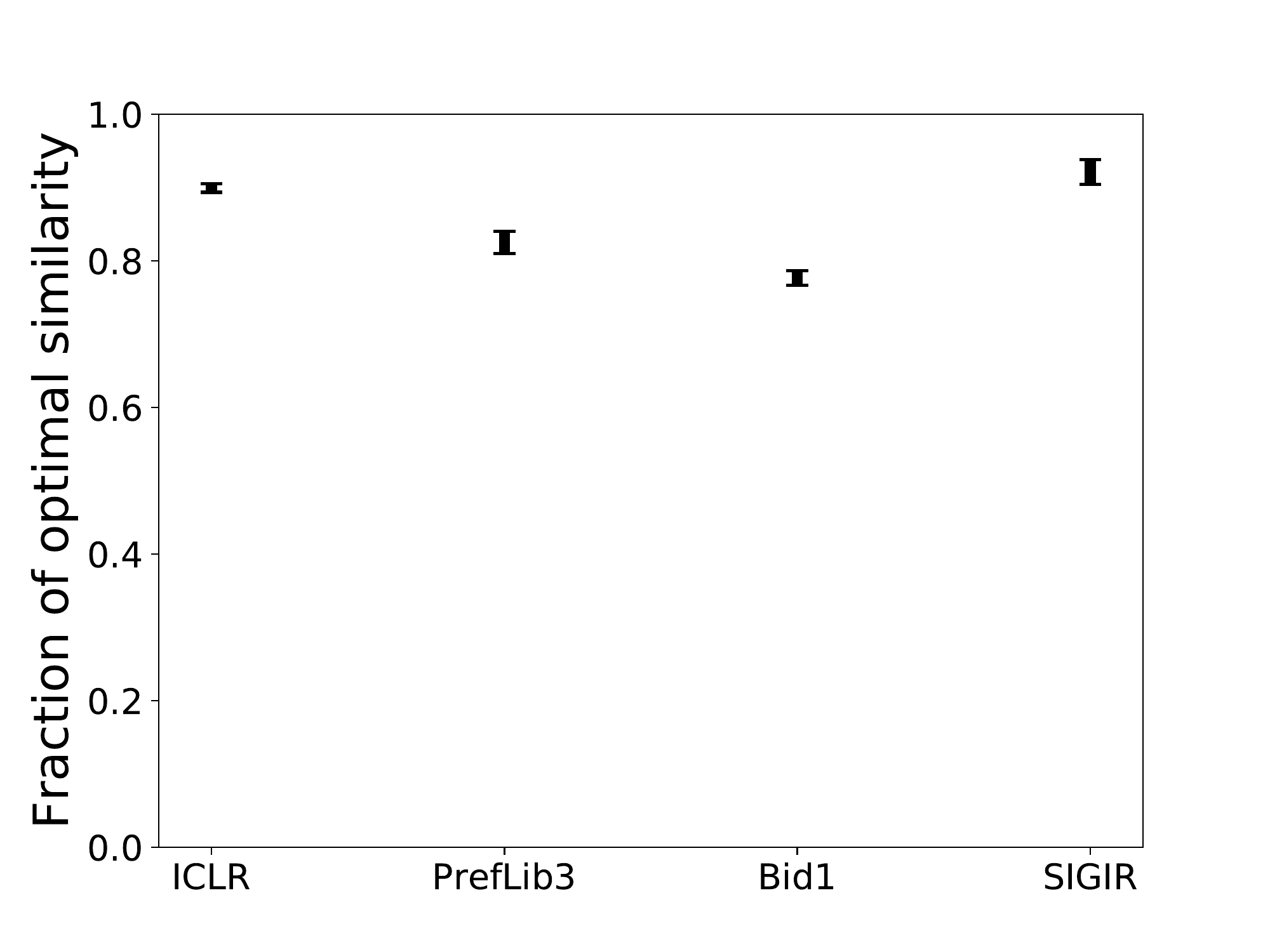}
    \caption{Range of values found over $10$ random reviewer splits when papers split between stages.} \label{fig:papersplit} 
\end{figure*}

\section{Submodularity of $\objfn$} \label{apdx:submod}
In this section, we show that the problem of optimizing $\objfn$ (or $\sobjfn$) is actually an instance of submodular optimization. For simplicity, we consider the case where $\numrev = (1 + \stagetwofrac)\numpap$ and $\revload = \papload^{(1)} = \papload^{(2)} = 1$.

For some set $\mathcal{N}$, a function $g: 2^\mathcal{N} \to \mathbb{R}$ is submodular if $g(A \cup \{u\}) - g(A) \geq g(B \cup \{u\}) - g(B)$ for all $A \subseteq B \subseteq \mathcal{N}$ and all $u \in \mathcal{N} \setminus B$. Since $\objfn$ and $\sobjfn$ are defined only for $\revset_2$ where $|\revset_2| = \stagetwofrac\numpap$, we modify them to be defined over $2^{\revset}$.

Recall that for subsets $\revset' \subseteq \revset$ and $\papset' \subseteq \papset$, desired paper load $\papload$, and maximum reviewer load $\revload$, $\mathcal{M}(\revset', \papset'; \revload, \papload)$ is the set of assignment matrices that assign a load of exactly $\papload$ to all papers in $\papset'$ (and a load of at most $\revload$ to all reviewers in $\revset'$). Define $\mathcal{M}'(\revset', \papset'; \revload, \papload)$ as the set of assignment matrices that assign a load of at most $\papload$ to all papers in $\papset'$ (and a load of at most $\revload$ to all reviewers in $\revset'$). Formally, $\adassign \in \mathcal{M}'(\revset', \papset'; \revload, \papload)$ if and only if $\sum_{\adrev \in \revset'} \adassign_{\adrev \adpap} \leq \papload$ for all $\adpap \in \papset'$, $\sum_{\adpap \in \papset'} \adassign_{\adrev \adpap} \leq \revload$ for all $\adrev \in \revset'$, and $\adassign_{\adrev \adpap} = 0$ for all $(\adrev, \adpap) \not\in \revset' \times \papset'$. 

Consider the modified version of $\valfn$
\begin{align*}
    \valfn'(\revset_2, \papset_2) &= \frac{1}{(1+\stagetwofrac) \numpap} \left[  \max_{A \in \mathcal{M}'(\revset \setminus \revset_2, \papset; 1, 1)} \sum_{\adrev \in \revset \setminus \revset_2, \adpap \in \papset} A_{\adrev \adpap} \simmat_{\adrev \adpap} +  \max_{B \in \mathcal{M}'(\revset_2, \papset_2; 1, 1)} \sum_{\adrev \in \revset_2, \adpap \in \papset_2} B_{\adrev \adpap} \simmat_{\adrev \adpap} \right]
\end{align*} 
which allows papers to be underloaded and so is defined for all $\revset_2$ and $\papset_2$. 
Define $\objfn_{sub}(\revset_2) = \mathbb{E}_{\papset_2}\left[\valfn'(\revset_2, \papset_2)\right]$ and $\sobjfn_{sub}(\revset_2) = \frac{1}{\numsamples} \sum_{i=1}^\numsamples \valfn'(\revset_2, \papset_2^{(i)})$ as modifications of $\objfn$ and $\sobjfn$.
Since $\simmat \geq 0$, there exists a maximum-similarity assignment from within $\mathcal{M}'(\revset', \papset'; 1, 1)$ that meets all paper load constraints with equality when $|\revset'| \geq |\papset'|$ and thus is contained in $\mathcal{M}(\revset', \papset'; 1, 1)$.  Also, $\mathcal{M}(\revset', \papset'; 1, 1) \subseteq \mathcal{M}'(\revset', \papset'; 1, 1)$. 
Thus, when $|\revset_2| = \stagetwofrac\numpap$, $\valfn(\revset_2, \papset_2) = \valfn'(\revset_2, \papset_2)$. Therefore, subject to the constraint $|\revset_2| = \stagetwofrac\numpap$, maximizing $\objfn_{sub}$ (or $\sobjfn_{sub}$) is equivalent to maximizing $\objfn$ (or $\sobjfn$).

\begin{proposition} \label{prop:submod} 
%$\valfn(\revset_2, \papset_2)$ is submodular in $\revset_2$ and $\papset_2$. Further, 
$\objfn_{sub}$ and $\sobjfn_{sub}$ are submodular in $\revset_2$. 
\end{proposition}
\begin{proof}
Note that $\max_{\adassign \in \mathcal{M}'(\revset', \papset'; 1, 1)} \sum_{\adrev \in \revset', \adpap \in \papset'} \adassign_{\adrev \adpap} \simmat_{\adrev \adpap}$ is a submodular function of the reviewer set $\revset'$ when the paper set $\papset'$ is held fixed~\cite{kulik2019generalized}. 
Submodularity in $\revset_2$ is equivalent to submodularity in $\revset_1 = \revset \setminus \revset_2$, so $\valfn'(\revset_2, \papset_2)$ is submodular in $\revset_2$. As sums over terms submodular in $\revset_2$, $\objfn_{sub}$ and $\sobjfn_{sub}$ are submodular in $\revset_2$.
\end{proof} 
Therefore, the two-stage paper assignment problem is an instance of maximizing a non-monotone submodular function subject to a cardinality constraint $|\revset_2| = \stagetwofrac\numpap$. 
However, no value oracle for the function $\objfn$ is available due to the expectation over $\papset_2$. Since a polynomial-time value oracle is available for $\sobjfn$, the paper~\cite{buchbinder2014submodular} gives an approximation algorithm achieving an approximation ratio of no greater than $0.5$ (depending on $\stagetwofrac$). This guarantee does not imply much about the assignment quality, since it can be trivially achieved by maximizing assignment similarity in the first stage only. Furthermore, it is known that achieving an approximation ratio of greater than $0.5$ requires an exponential number of queries to the value oracle; this holds true even without the cardinality constraint~\cite{gharan2011submodular, feige2011maximizing}. Thus, generic algorithms for submodular maximization are not helpful for our problem.

\section{Proofs}

\subsection{Proof of Theorem~\ref{thm:samplednph}} \label{apdx:samplednph}
We will show that it is NP-hard to determine if there exists a choice of $\revset_2$ with value $\sobjfn(\revset_2) = 1$ when $\numsamples=3$, for some instance of $\papset_2^{(1)}, \papset_2^{(2)}, \papset_2^{(3)}$. If such a choice exists, it would be the optimal solution. Therefore, any algorithm to optimize $\sobjfn$ would be able to determine if there exists a solution with value $1$, solving an NP-hard problem. This implies the NP-hardness of optimizing $\sobjfn$.

We reduce from 3-Dimensional Matching, an NP-complete problem \cite{karp1972reducibility}. An instance of 3-Dimensional Matching consists of three sets $X, Y, Z$ of size $s$, and a collection of tuples $T \subseteq X \times Y \times Z$. It asks whether there exists a selection of $s$ tuples from $T$ that includes each element of $X, Y$, and $Z$ exactly once. 

Given such an instance of 3-Dimensional Matching, we construct an instance of the two-stage paper assignment problem with $\numpap = |T| + 2s$, $\numrev = |T| + 3s$, and $\stagetwofrac = \frac{\numrev}{\numpap} - 1$ ($\revload = \papload^{(1)} = \papload^{(2)} = 1$). $\stagetwofrac\numpap=s$ papers and reviewers will be in stage two. The first $s$ papers correspond to elements of $X$, the next $s$ to elements of $Y$, and the next $s$ to elements of $Z$; the remaining $|T|-s$ papers are ``dummy papers'' that all reviewers can review. The first $3s$ reviewers are ``specialty reviewers'' corresponding to each of the first $3s$ papers, and the remaining $|T|$ reviewers correspond to each of the elements of $T$. We construct the $\numsamples=3$ sampled subsets $\papset_2^{(1)} = \{1, \dots, s\}, \papset_2^{(2)} = \{s+1, \dots, 2s\}, \papset_2^{(3)} = \{2s+1, \dots, 3s\}$, where the elements of these sets correspond to the elements of $X$, $Y$, and $Z$ respectively. We then construct $\simmat$ as follows. For $i \in [3s]$, set $\simmat_{i i} = 1$ and $\simmat_{ij} = 0$ for all $j \in [3s], j \neq i$. For the remaining reviewers $i \in \{3s+1, \dots, 3s+|T|\}$ and for papers $j \in [3s]$, set $\simmat_{ij} = 1$ if the element corresponding to $j$ in $X \cup Y \cup Z$ is included in the tuple corresponding to $i$ in $T$. Finally, for the remaining papers $j \in \{3s+1, \dots, |T|+2s\}$, set $\simmat_{ij} = 1$ for all reviewers $i$.

Suppose we have a ``yes'' instance of 3-Dimensional Matching, so there exists a choice of $s$ tuples from $T$ that cover each element of $X$, $Y$, and $Z$. Choose the corresponding $s$ reviewers as $\revset_2$ and the remaining reviewers as $\revset_1$. In stage one, we can assign each specialty reviewer to each of their corresponding papers and each of the remaining $|T|-s$ reviewers in $\revset_1$ to dummy papers. In stage two, for each of the three possible samples, there exists one reviewer that has similarity $1$ with each paper since the corresponding choice of tuples from $T$ cover $X$, $Y$, and $Z$. Therefore, this partition achieves $\sobjfn(\revset_2) = 1$.

Suppose we have a ``no'' instance of 3-Dimensional Matching, so no choice of $s$ tuples from $T$ covers each element of $X$, $Y$, and $Z$. We claim that no choice of $\revset_2$ will achieve $s$ total similarity in the second stage. First, suppose we include a speciality reviewer in $\revset_2$. This reviewer has similarity $1$ with only one paper, so there exists a sample of stage two papers $\papset_2^{(i)}$ such that this reviewer must be assigned to a paper with which it has similarity $0$. Therefore, $\sobjfn(\revset_2)$ cannot be $1$ when a specialty reviewer is in $\revset_2$ and so $\revset_2$ must be chosen from the reviewers corresponding to elements of $T$. However, no choice of $s$ tuples covers each element of $X$, $Y$, and $Z$. Therefore, for every choice of $\revset_2$, some reviewer must be assigned to a paper with which they have similarity $0$ for at least one of the sampled sets of stage two papers. This means that $\sobjfn(\revset_2) = 1$ is unachievable.

\subsection{Proof of Theorem~\ref{thm:generallb}} \label{apdx:generallb} 
For any $\stagetwofrac \in [0, 1]$, choose any $\numpap$ such that $\stagetwofrac \numpap \in \mathbb{Z}_+$. We construct the following similarity matrix.  Paper $i$ has similarity $1$ with reviewer $i$, and also with reviewer $\numpap + i$ if $i \leq \stagetwofrac \numpap$. All other similarities are $0$. 

On this example, the \optterm{} assignment for any $\papset_2$ is to assign reviewers $\{1, \dots, \numpap\}$ to papers in the first stage, since this maximizes the similarity across both stages. This choice gives a total similarity of $\numpap$ in stage one and an expected similarity of $\stagetwofrac^2 \numpap$ in stage two (since each reviewer's matching paper is in stage two with probability $\stagetwofrac$), for a total similarity of $\numpap(1 +  \stagetwofrac^2)$. Since there are $(1 + \stagetwofrac) \numpap$ total assignments, the expected mean similarity is $\frac{1 +  \stagetwofrac^2}{1 + \stagetwofrac}$.

Now consider the assignment after randomly splitting reviewers. Any paper $\adpap \leq \stagetwofrac\numpap$ has two reviewers $a, b$ with similarity $1$. For sufficiently large $\numpap \geq \frac{1 + 4 \stagetwofrac}{1 + \stagetwofrac}$, the expected value of this paper's assignment is 
\begin{align*}
    &(P[a \in \revset_1 \land b \in \revset_2] + P[b \in \revset_1 \land a \in \revset_2]) (1 + P[\adpap \in \papset_2]) \\
    &\qquad + P[a, b \in \revset_1] + P[a, b \in \revset_2] P[\adpap\in\papset_2] \\
    &=\left(2 \frac{\numpap}{(1 + \stagetwofrac)\numpap} \frac{\stagetwofrac \numpap}{(1 + \stagetwofrac) \numpap - 1}\right) (1 + \stagetwofrac) + \frac{\numpap}{(1 + \stagetwofrac)\numpap} \frac{\numpap -1}{(1 + \stagetwofrac)\numpap-1} + \frac{\stagetwofrac \numpap}{(1 + \stagetwofrac)\numpap} \frac{\stagetwofrac \numpap -1}{(1 + \stagetwofrac)\numpap-1} \stagetwofrac \\
    &\leq 2\frac{\stagetwofrac}{(1 + \stagetwofrac) - \frac{1}{n}}  + \frac{1}{(1 + \stagetwofrac)^2} + \frac{\stagetwofrac^3}{(1 + \stagetwofrac)^2} \\
    &\leq \frac{1 + 4 \stagetwofrac}{2(1 + \stagetwofrac)}  + \frac{1}{(1 + \stagetwofrac)^2} + \frac{\stagetwofrac^3}{(1 + \stagetwofrac)^2}.
\end{align*} 
There are $\stagetwofrac \numpap$ of these papers.

Any of the remaining papers $\adpap > \stagetwofrac\numpap$ has only one reviewer $a$ with similarity $1$. The expected value of this paper's assignment is 
\begin{align*}
    &P[a \in \revset_1] + P[a \in \revset_2] P[\adpap \in \papset_2] \\
    &= \frac{1 + \stagetwofrac^2}{1 + \stagetwofrac}.
\end{align*}
There are $(1 - \stagetwofrac)\numpap$ of these papers.

Totalling over all papers and dividing by the total number of assignments, the mean expected similarity of random split is at most
\begin{align*}
    &\left(\frac{1 + 4 \stagetwofrac}{2(1 + \stagetwofrac)}  + \frac{1}{(1 + \stagetwofrac)^2} + \frac{\stagetwofrac^3}{(1 + \stagetwofrac)^2}\right) \frac{\stagetwofrac}{1+\stagetwofrac}  + \frac{(1 + \stagetwofrac^2)(1 - \stagetwofrac)}{(1 + \stagetwofrac)^2}.
\end{align*}

The suboptimality is therefore at least
\begin{align*}
    &\frac{1 +  \stagetwofrac^2}{1 + \stagetwofrac} - \left(\frac{1 + 4 \stagetwofrac}{2(1 + \stagetwofrac)} + \frac{1}{(1 + \stagetwofrac)^2} + \frac{\stagetwofrac^3}{(1 + \stagetwofrac)^2}\right) \frac{\stagetwofrac}{1+\stagetwofrac} - \frac{(1 + \stagetwofrac^2)(1 - \stagetwofrac)}{(1 + \stagetwofrac)^2} \\
    &= \frac{(1 +  \stagetwofrac^2)2 \stagetwofrac}{(1 + \stagetwofrac)^2} - \left(\frac{1 + 4 \stagetwofrac}{2(1 + \stagetwofrac)} + \frac{1}{(1 + \stagetwofrac)^2} + \frac{\stagetwofrac^3}{(1 + \stagetwofrac)^2}\right) \frac{\stagetwofrac}{1+\stagetwofrac} \\
    &= \left(2 (1 +  \stagetwofrac^2)(1 + \stagetwofrac) - \frac{1}{2}(1 + 4 \stagetwofrac)(1 + \stagetwofrac) - 1 - \stagetwofrac^3\right) \frac{\stagetwofrac}{(1+\stagetwofrac)^3} \\
    &= \left(\frac{1}{2} - \frac{1}{2}\stagetwofrac +  \stagetwofrac^3 \right) \frac{\stagetwofrac}{(1+\stagetwofrac)^3} \\
    &\geq \frac{\stagetwofrac^4}{(1+\stagetwofrac)^3} .
\end{align*}

\subsection{Proof of Theorem~\ref{thm:lowrankub}} \label{apdx:lowrankub}
%In this section, we denote by $\ln$ the logarithm with base $e$ and by $\log$ the logarithm with base $2$. 
By Lemma~4 of \cite{rothvoss2014direct}, a rank $\rank$ similarity matrix $\simmat \in [0, 1]^{(1+\stagetwofrac)\numpap \times \numpap}$ can be factored into vectors $u_{\adrev} \in \mathbb{R}^\rank$ for each reviewer $\adrev$ and $v_{\adpap} \in \mathbb{R}^\rank$ for each paper $\adpap$ such that $\simmat_{\adrev \adpap} = \langle u_{\adrev}, v_{\adpap} \rangle$, $||u_{\adrev}||_2 \leq \rank^{1/4}$, and $||v_{\adpap}||_2 \leq \rank^{1/4}$. 

Consider the ball of radius $\rank^{1/4}$ in $\mathbb{R}^\rank$ in which the paper vectors $v_{\adpap}$ lie. We cover this ball with smaller ``cells'' by dividing the containing $\rank$-dimensional hypercube with side length $2 \rank^{1/4}$ along each dimension to create some number of smaller hypercubes. If we divide the containing hypercube into $t$ equal-sized segments along each dimension, there are $t^\rank$ cells in total and the maximum L2 distance between two points in a cell is $\frac{2 \rank^{3/4}}{t}$. 

We construct $L$ layers of cells in this way, where the cells increase in size between layers. Denote by $t_i$ the number of divisions along each dimension at layer $i$. We choose $t_i = 2^{Z_i}$ for some integer $Z_i$ for all layers $i$ so that each cell at layer $i$ is fully contained within a single cell at each higher layer. Denote by $s_i$ the desired maximum within-cell distance at layer $i$. This distance is achieved if $t_i$ is at least $\frac{2 \rank^{3/4}}{s_i}$, so the minimum such $t_i$ that is also a power of two is at most $\frac{4 \rank^{3/4}}{s_i}$. This gives that there are at most $z_i = \left(\frac{4 \rank^{3/4}}{s_i}\right)^{\rank}$ cells in layer $i$. 
%By Lemma~23.11 of \cite{shalev2014understanding}, this ball can be covered with $z$ balls of radius $s$ if $z = \left( \frac{3 \rank^{1/4}}{s} \right)^\rank$. In what follows, we say that a paper $\adpap$ is in some ball if its vector $v_{\adpap}$ is in the ball. We say that a reviewer is in a ball if it is assigned to a paper in that ball by the \optterm{} paper assignment.

In what follows, we say that a paper $\adpap$ is in some cell if its vector $v_{\adpap}$ is in the cell. (Papers on the border of multiple cells at layer $1$ are considered to be in an arbitrary one of the bordering cells so that each paper is in exactly one cell. At higher layers, such papers are considered to be in the cell containing their layer $1$ cell.) We say that a reviewer is in a cell if it is assigned to a paper in that cell by the \optterm{} paper assignment (given $\papset_2$).

% Issue with smaller balls being split in the next layer: suppose a ball at layer 2 contains only papers from K balls at layer 1; each of these 2 balls at layer 1 has x papers and (x/K) of the papers in each layer 1 ball are in the layer 2 ball so that the layer 2 ball also has x papers. suppose that in each layer 1 ball, x reviewers outside of the layer 2 ball and sqrt(x) reviewers inside the layer 2 ball are in stage 2 (this is allowable in the high probability event). since we match arbitrarily within the smaller balls first, suppose that we match all x of the stage 2 reviewers outside the layer 2 ball and do not match the reviewers inside the layer 2 ball. the layer 2 ball then contains Ksqrt(x) reviewers and 0 papers in stage 2. this contradicts the argument in the proof that the number of unmatched reviewers in this ball is at most sqrt(x) (the actual written argument gives 3sqrt(x) to account for the case where beta < 1). 

Given $\papset_2$ and $\revset_2$ produced by random split, we proceed through layers from $1$ to $L$ in order to match reviewers to papers in the same stage. We match as many reviewers as possible to papers that are within the same cell at each layer $i$, and then continue to layer $i+1$. Define $\numpap_i$ as an upper bound on the number of reviewers unmatched before matching within layer $i$; $\numpap_1 = (1 + \stagetwofrac)\numpap$.  The difference in value between the assignment $\adassign$ produced in this way and the \optterm{} assignment $\adassign^*$ (which we call the ``value gap'') is 
\begin{align*}
    \sum_{\adrev \in \revset, \adpap \in \papset} (\adassign^*_{\adrev \adpap} - \adassign_{\adrev \adpap}) \langle u_{\adrev}, v_{\adpap} \rangle &= \sum_{\adrev \in \revset, \adpap \in \papset, \adpap^* \in \papset} \adassign_{\adrev \adpap} \adassign^*_{\adrev \adpap^*} \langle u_{\adrev}, v_{\adpap^*} - v_{\adpap} \rangle \\
    &\leq  \sum_{\adrev \in \revset, \adpap \in \papset, \adpap^* \in \papset} \adassign_{\adrev \adpap} \adassign^*_{\adrev \adpap^*}   ||u_{\adrev}||_2 ||v_{\adpap^*} - v_{\adpap} ||_2 \\
    &\leq \rank^{1/4} \sum_{\adrev \in \revset, \adpap \in \papset, \adpap^* \in \papset} \adassign_{\adrev \adpap} \adassign^*_{\adrev \adpap^*}  ||v_{\adpap^*} - v_{\adpap} ||_2.
\end{align*}

Consider some cell containing $x$ papers. All $x$ of these papers are in stage one. Define $Hyp(N, K, M)$ as the hypergeometric distribution where $N$ is the population size, $M$ is the number of draws, and $K$ is the number of successes in the population; by symmetry $Hyp(N, K, M)$ is equivalent to $Hyp(N, M, K)$. The number of stage two papers has distribution $Hyp(\numpap, x, \stagetwofrac\numpap)$. With probability $1 - 2\delta$, by Hoeffding's inequality~\cite{hoeffding1994probability} and using the symmetry property, within $\stagetwofrac x \pm \sqrt{\frac{x}{2} \ln(1/ \delta)}$ of the papers in this cell are also in stage two. (In this section, $\ln$ indicates the logarithm with base $e$ and $\log$ indicates the logarithm with base $2$.)
Call $y$ the total number of reviewers in the cell. There are exactly the same number of reviewers as total stage one and two papers in this cell, so $y$ is within $(1 + \stagetwofrac)x \pm \sqrt{\frac{x}{2} \ln(1/ \delta)}$ and is at most $2x$.
Since $\revset_2$ is produced by random split, the number of reviewers in this cell in stage one has distribution $Hyp((1+\stagetwofrac)\numpap, y, \numpap)$ and the number of reviewers in this cell in stage two has distribution $Hyp((1+\stagetwofrac)\numpap, y, \stagetwofrac\numpap)$. By Hoeffding's inequality and again using symmetry, the number of reviewers in the cell in stage one is at most $\frac{y}{1+\stagetwofrac} + \sqrt{\frac{y}{2} \ln(1/\delta)} \leq x + \sqrt{\frac{x}{2} \ln(1/ \delta)} + \sqrt{x \ln(1/\delta)} \leq x + \sqrt{3x \ln(1/\delta)}$ with probability $1 - 2 \delta$ (conditioned on the earlier event concerning the number of stage-two papers). By this argument, with probability $1 - 4 \delta$ (again conditioned on the earlier event), there are within $x \pm \sqrt{3x \ln(1 / \delta)}$ reviewers in stage one in the cell and within $\stagetwofrac x \pm \sqrt{3x \ln(1/ \delta)}$ reviewers in stage two in the cell. In total, there are at most $\numpap L$ cells with a non-zero number of papers across all layers and so the total probability of error in any of the bounds is at most $6 \delta L \numpap$. 

Assume that this high probability event occurs. In layer $i$, in any cell $j$ with $x_j$ papers (all of which are in stage one), the number of stage one reviewers is within $x_j \pm \sqrt{3x_j \ln(1/ \delta)}$. Any reviewers in this cell matched at earlier layers must have been matched to papers also in this cell. Therefore, the number of unmatched stage one reviewers after matching within this cell is at most $\sqrt{3x_j \ln(1/ \delta)}$. The number of stage two reviewers is within $\stagetwofrac x_j \pm \sqrt{3x_j \ln(1/ \delta)}$ and the number of stage two papers is within $\stagetwofrac x_j \pm \sqrt{\frac{x_j}{2} \ln(1/ \delta)}$. Therefore, the number of unmatched stage two reviewers after matching within this cell is at most $\sqrt{6x_j \ln(1/ \delta)}$. In total over both stages, the total number of unmatched reviewers after matching in layer $i$ is at most $\numpap_{i+1} = \sum_{j=1}^{z_i} \sqrt{18 x_j \ln(1/ \delta)} \leq \sqrt{18 z_i \numpap \ln(1/ \delta)}$. All of the reviewers matched at layer $i$ are matched to papers at most $s_i$ away from their optimal paper assignment. Across all layers, the value gap is therefore bounded by $\rank^{1/4} \left(\sum_{i=1}^{L-1} \numpap_i s_i + 2 \numpap_L \rank^{1/4}\right)$, since everything at layer $L$ is matched to whatever remains regardless of $s_L$.

We now determine how to set $s_i$ for all layers $i$. We choose $s_1 = s$ and set other $s_i$ such that $\numpap_i s_i = (1+\stagetwofrac )\numpap s$ for all $i$.  
This leads to the recursively-defined values of $s_i = \frac{(1+\stagetwofrac )\numpap s}{\numpap_i}$, $z_i = (4\rank^{3/4})^{\rank} s_i^{-\rank}$, $\numpap_{i+1} = \sqrt{z_i \numpap 18 \ln(1/ \delta)}$ with initial values $\numpap_1 = (1+\stagetwofrac )\numpap$ and $s_1 = s$. Unrolling the iteration, we see that 
\begin{align*} 
    s_i &= s_{i-1}^{\frac{\rank}{2}} \numpap^{\frac{1}{2}} s (4\rank^{3/4})^{-\frac{\rank}{2}} (18 \ln(1/\delta))^{-\frac{1}{2}} (1 + \stagetwofrac) \\
    &=  \numpap^{ \frac{1}{2} \sum_{j=0}^{i-2} \left(\frac{\rank}{2}\right)^{j} } 
    s^{ \sum_{j=0}^{i-1} \left( \frac{\rank}{2} \right)^{j} }
    (4\rank^{3/4})^{ -\frac{\rank}{2} \sum_{j=0}^{i-2} \left(\frac{\rank}{2}\right)^{j} }
    (18 \ln(1/ \delta))^{ -\frac{1}{2} \sum_{j=0}^{i-2} \left(\frac{\rank}{2}\right)^{j} } (1+\stagetwofrac )^{ \sum_{j=0}^{i-2} \left(\frac{\rank}{2}\right)^{j} } 
\end{align*}
for $i \geq 2$. 
Defining $\epsilon$ such that $s = \left(\frac{(1+\stagetwofrac )^2 \numpap}{18 \ln(1/ \delta)} \right)^\epsilon$,  
\begin{align*}
    s_i &= \left(\frac{(1+\stagetwofrac )^2 \numpap}{18 \ln(1/ \delta)}\right)^{ \frac{1}{2} \sum_{j=0}^{i-2} \left(\frac{\rank}{2}\right)^{j} + \epsilon \sum_{j=0}^{i-1} \left( \frac{\rank}{2} \right)^{j} }
    (4\rank^{3/4})^{ -\frac{\rank}{2} \sum_{j=0}^{i-2} \left(\frac{\rank}{2}\right)^{j} }
\end{align*} 
for $i \geq 2$.
This gives a value gap of at most $\rank^{1/4} (1+\stagetwofrac )^{1 + 2\epsilon} \numpap^{1 + \epsilon} (18 \ln(1/ \delta))^{-\epsilon} \left((L-1) + 2 \rank^{1/4} s_L^{-1} \right)$. We now continue in cases on the value of $\rank$. 

\textbf{Case $\rank = 1$.} Note that $\sum_{j=0}^{i-1} \left(\frac{\rank}{2}\right)^{j} = 2\left(1 - \frac{1}{2^i}\right)$, so  
\begin{align*}
    s_i &= \left(\frac{(1+\stagetwofrac )^2 \numpap}{18 \ln(1/ \delta)}\right)^{ 1 - \frac{1}{2^{i-1}} + \epsilon 2\left(1 - \frac{1}{2^i}\right) }
    (4\rank^{3/4})^{ -1 + \frac{1}{2^{i-1}} }.
\end{align*} 
Choose $\epsilon = - \frac{1}{2} + \frac{1}{2(2^L - 1)} $ so that $s_L = (4\rank^{3/4})^{ -1 + \frac{1}{2^{L-1}} }$. Setting $\delta = (2\numpap)^{-3}$ and $L = \log \log \numpap$, for sufficiently large $\numpap$, the value gap is bounded by 
\begin{align*}
    &\rank^{1/4} (1+\stagetwofrac )^{\frac{1}{2^L - 1}} \numpap^{\frac{1}{2} + \frac{1}{2(2^L - 1)}} (18 \ln(1/ \delta))^{\frac{1}{2} -\frac{1}{2(2^L - 1)} } \left((L-1) + 2 \rank^{1/4} (4\rank^{3/4})^{ 1 - \frac{1}{2^{L-1}} } \right) \\
    &\leq (1+\stagetwofrac ) \numpap^{\frac{1}{2}} 2^{\log(\numpap) \frac{1}{2(\log \numpap - 1)}} (54 \ln(2\numpap))^{\frac{1}{2}} \left((\log\log \numpap-1) + 8\right) \\
    &\leq 2 (1+\stagetwofrac ) \numpap^{\frac{1}{2}} (54 \ln(2\numpap))^{\frac{1}{2}} \left((\log\log \numpap -1) + 8\right) \\
    &\leq C (\log \numpap)^{\eta} \numpap^{\frac{1}{2}}
\end{align*}
for some constants $C, \eta$ with probability $1 - \frac{6 \log \log \numpap}{8 \numpap^2} \geq 1 - \frac{1}{\numpap}$.

\textbf{Case $\rank = 2$.} Note that $\sum_{j=0}^{i-1} \left(\frac{\rank}{2}\right)^{j} = i$, so 
\begin{align*}
    s_i &= \left(\frac{(1+\stagetwofrac )^2 \numpap}{18 \ln(1/ \delta)}\right)^{ \frac{1}{2} (i-1) + \epsilon i}
    (4\rank^{3/4})^{ -i+1 }.
\end{align*} 
Choose $\epsilon = - \frac{1}{2} + \frac{1}{2L} $ so that $s_L = (4\rank^{3/4})^{ -L+1 }$. Setting $\delta = (2\numpap)^{-3}$ and $L = \log \log \numpap$, for sufficiently large $\numpap$, the value gap is bounded by
\begin{align*}
    &\rank^{1/4} (1+\stagetwofrac )^{\frac{1}{L}} \numpap^{\frac{1}{2} + \frac{1}{2L}} (18 \ln(1/ \delta))^{\frac{1}{2} - \frac{1}{2L}} \left((L-1) + 2 \rank^{1/4} (4\rank^{3/4})^{ L-1 } \right) \\
    &\leq \rank^{1/4} (1+\stagetwofrac ) \numpap^{\frac{1}{2} + \frac{1}{2 \log \log \numpap}} (54 \ln(2\numpap))^{\frac{1}{2}} \left((\log\log \numpap - 1) + 2 \rank^{1/4} (\log \numpap)^{ \log (4\rank^{3/4})  } \right) \\
    &\leq C (\log \numpap)^{\eta} \numpap^{\frac{1}{2} + \frac{1}{\log \log \numpap}} 
\end{align*} 
for some constants $C, \eta$ with probability $1 - \frac{6 \log \log \numpap}{8\numpap^2} \geq 1 - \frac{1}{\numpap}$. 

\textbf{Case $\rank \geq 3$.} Note that $\sum_{j=0}^{i-1} \left(\frac{\rank}{2}\right)^{j} = \frac{\left(\frac{\rank}{2}\right)^i - 1}{\frac{\rank}{2} - 1}$, so 
\begin{align*}
    s_i &= \left(\frac{(1+\stagetwofrac )^2 \numpap}{18 \ln(1/ \delta)}\right)^{ \left(\frac{1}{2} + \epsilon \right) \left(\frac{(\rank/2)^i - 1}{(\rank/2) - 1}\right) - \frac{1}{2} \left(\frac{\rank}{2}\right)^{i-1} }
    (4\rank^{3/4})^{ -\frac{\rank}{2} \left( \frac{(\rank/2)^i - 1}{(\rank/2) - 1} \right) + \left( \frac{\rank}{2} \right)^i }.
\end{align*} 
Choose $\epsilon = - \frac{1}{\rank} +  \frac{\left( \frac{1}{2} - \frac{1}{\rank} \right)}{(\rank/2)^L - 1}$ so that $s_L = (4\rank^{3/4})^{- \frac{\rank}{2} \left( \frac{(\rank/2)^L - 1}{(\rank/2) - 1} \right) + \left( \frac{\rank}{2} \right)^L }$. Setting $\delta = (2\numpap)^{-3}$ and $L = \frac{\log\log \log \numpap}{\log (\rank/2)}$, for sufficiently large $\numpap$,  the value gap is bounded by
\begin{align*}
    &\rank^{1/4} (1+\stagetwofrac )^{1 - \frac{2}{\rank} +  \frac{2\left( \frac{1}{2} - \frac{1}{\rank} \right)}{(\rank/2)^L - 1}} \numpap^{1 - \frac{1}{\rank} +  \frac{\left( \frac{1}{2} - \frac{1}{\rank} \right)}{(\rank/2)^L - 1}} (18 \ln(1/ \delta))^{\frac{1}{\rank} -  \frac{\left( \frac{1}{2} - \frac{1}{\rank} \right)}{(\rank/2)^L - 1}} \left((L-1) + 2 \rank^{1/4} (4\rank^{3/4})^{ \frac{\rank}{2} \left( \frac{(\rank/2)^L - 1}{(\rank/2) - 1} \right) - \left( \frac{\rank}{2} \right)^L } \right) \\
    &\leq \rank^{1/4} (1+\stagetwofrac ) \numpap^{1 - \frac{1}{\rank} +  \frac{\left( \frac{1}{2} - \frac{1}{\rank} \right)}{\log\log \numpap - 1}} (54 \ln(2\numpap))^{\frac{1}{\rank}} \left(\frac{\log\log \log \numpap}{\log (\rank/2)} -1 + 2 \rank^{1/4} (4\rank^{3/4})^{2 \log\log \numpap} \right) \\
    &\leq \rank^{1/4} (1+\stagetwofrac ) \numpap^{1 - \frac{1}{\rank} +  \frac{\left( \frac{1}{2} - \frac{1}{\rank} \right)}{\log\log \numpap - 1}} (54 \ln(2\numpap))^{\frac{1}{\rank}} \left(\frac{\log\log \log \numpap}{\log (\rank/2)} -1 + 2 \rank^{1/4} (\log \numpap)^{ 2\log(4\rank^{3/4})} \right) \\
    &\leq C (\log \numpap)^{\eta} \numpap^{1 - \frac{1}{\rank} +  \frac{1}{\log\log \numpap }}
\end{align*}
for some constants $C, \eta$ with probability $1 - \frac{6 \log\log\log \numpap}{8 \log (\rank/2) \numpap^2} \geq 1 - \frac{1}{\numpap}$. 

To get the suboptimality, divide the value gap by $(1+\stagetwofrac )\numpap \leq 2 \numpap$.

\subsection{Proof of Theorem~\ref{thm:lowranklb}} \label{apdx:lowranklb}
\paragraph{(a)}
Choose $\numpap$ large enough such that $\rank \leq \frac{\numpap}{2}$. We define $\rank$ groups of reviewers and papers such that all papers have similarity $1$ with all reviewers within the same group and similarity $0$ with all other reviewers. Group $1$ contains  all papers $\adpap \leq \lceil \frac{\numpap}{2} \rceil$ and all reviewers $\adrev \leq 2 \lceil \frac{\numpap}{2} \rceil$. Each other group $2, \dots, \rank$ contains $2$ reviewers and $1$ paper. All papers and reviewers not in any group have all similarities $0$. This similarity matrix has rank $\rank$. 

The \optterm{} assignment for any $\papset_2$ will split the reviewers in each group evenly between stages, so all papers in any group can be assigned a similarity-$1$ reviewer in both stages. This gives a total similarity of at least $\numpap + 2(\rank - 1)$.

Define $X$ as the random variable representing the number of reviewers from group 1 selected to be in $\revset_2$. $X \sim Hyp(2\numpap, 2 \lceil \numpap / 2 \rceil, \numpap)$, the hypergeometric distribution corresponding to the number of successes when $\numpap$ items are sampled without replacement from a population of $2\numpap$ items where $2 \lceil \numpap / 2 \rceil$ of them are successes. By Lemma~2.1 of~\cite{blais2010lower}, $P[X = t] \leq \frac{C}{\sigma}$ for any $t \geq 0$ where $\sigma^2 = \frac{\lceil \numpap / 2 \rceil}{2} \left(1 - \frac{\lceil \numpap / 2 \rceil}{\numpap} \right)$ and $C$ is an absolute constant. Since $\sigma^2 \geq \frac{\numpap}{4} \left(1 - \frac{1}{2} - \frac{1}{\numpap} \right) \geq \frac{\numpap}{16}$ for $\numpap \geq 4$, $P[X = t] \leq \frac{4C}{\sqrt{\numpap}}$ for sufficiently large $\numpap$. Therefore, $P\left[ \frac{\numpap}{2} - \frac{\sqrt{\numpap}}{16C} + 1 \leq X \leq \frac{\numpap}{2} + \frac{\sqrt{\numpap}}{16C} \right] \leq \frac{1}{2}$. With probability at least $\frac{1}{2}$, at least $\frac{\sqrt{\numpap}}{16C}$ of the reviewers in group $1$ cannot be matched to an optimal paper in their stage. Therefore, the total expected similarity is no greater than $\numpap - \frac{\sqrt{\numpap}}{32C} + 2(\rank - 1)$ and the expected difference in value from the \optterm{} assignment is at least $\frac{\sqrt{\numpap}}{32C}$. Since there are $2\numpap$ assignments, the suboptimality is at least $\frac{1}{64C\sqrt{n}}$.

\paragraph{(b)}
We construct a similarity matrix by creating a vector in $\mathbb{R}^{\rank}$ for each reviewer and each paper, and setting the similarity between that reviewer and that paper to be the inner product of their corresponding vectors. 
Consider the cube in $\mathbb{R}^\rank$ contained in $[0, 1 /\sqrt{\rank}]^\rank$. 
We construct a grid of points within this cube by evenly spacing $z = \lceil \numpap^{1 / \rank} \rceil$ along each axis and filling in the remaining points so that there are $z^k \geq \numpap$ grid points in total.
Place the $\numpap$ paper vectors at arbitrary (unique) points on this grid, so that each vector is at least $\frac{1}{\sqrt{\rank} \lceil \numpap^{1/\rank} \rceil} \geq \frac{1}{2 \sqrt{\rank} \numpap^{1/\rank}}$ away from any other paper vector. Place the $2\numpap$ reviewer vectors such that $2$ are at each grid point with a paper vector. The inner product of any two vectors is in $[0, 1]$, so this is a valid similarity matrix. The $2\numpap \times \rank$ and $\numpap \times \rank$ matrices where the rows are the reviewer and paper vectors respectively have linearly independent columns and so have rank $\rank$; thus, the similarity matrix has rank $\rank$. 

We claim that the \optterm{} matching across both stages chooses one reviewer from each grid point and matches it to the paper at the same point. Suppose we have a matching where this is not the case. There must exist a cycle of matched reviewer and paper pairs where the corresponding vectors are not paired with themselves and are instead paired $(x_1, x_2), (x_2, x_3), \dots, (x_K, x_1)$. This cycle has a total similarity of (using $x_{K+1}$ to refer to $x_1$)
\begin{align*}
    \sum_{i=1}^K \langle x_i, x_{i+1} \rangle &\leq \sum_{i=1}^K ||x_i||_2 ||x_{i+1}||_2 \\
    &\leq \sum_{i=1}^K ||x_i||_2^2 \\
    &= \sum_{i=1}^K \langle x_i, x_i \rangle
\end{align*}
so the matching value can be improved by changing the cycle so that reviewers and papers at the same grid point are matched. The second inequality is because $2ab \leq a^2 + b^2$ for any $a, b \in \mathbb{R}$. Therefore, the claimed matching is indeed optimal.

Now, consider the sample of $\numpap$ reviewers in stage one produced by a random split of reviewers. The following lemma shows that with probability $1 - O(e^{-\numpap/10})$, $\Theta(\numpap)$ grid points have both reviewers present in stage one under random split. 
\begin{lemma} \label{lem:revconc}
There exists $\numpap_0$ and a constant $\zeta$ such that for all $\numpap \geq \numpap_0$, the probability that less than $\numpap/100$ grid points have both reviewers in stage one after a random reviewer split is at most $\zeta e^{- \numpap/10}$. 
\end{lemma}
\begin{proof}
There are at most ${\numpap \choose a} 3^{\numpap-a}$ ways to assign reviewers to stages such that $a$ pairs of reviewers at the same grid point are in stage one.  For all $\numpap$ and $a$ such that $\numpap+1 \geq 4a$, ${\numpap \choose a} 3^{\numpap-a} = {\numpap \choose a-1} \frac{\numpap+1-a}{a} 3^{\numpap-a} \geq {\numpap \choose a-1} 3^{\numpap-a+1}$.  Setting $a = \numpap/100$, ${\numpap \choose \numpap/100} \leq (100e)^{\numpap/100} \leq \exp(0.06\numpap)$ and $3^{\numpap-(\numpap/100)} \leq \exp(1.09 \numpap)$. Therefore, the number of ways to assign reviewers to stages such that less than $\numpap/100$ pairs are in stage one is at most $\sum_{b=0}^{(\numpap/100) - 1} {\numpap \choose b} 3^{\numpap-b} \leq (\numpap/100) {\numpap \choose \numpap/100} 3^{\numpap-(\numpap/100)} \leq \exp(1.15 \numpap + \ln(0.01 \numpap) )$. Using Sterling inequalities~\cite{robbins1955remark}, the total number of ways to assign reviewers to stages is ${2\numpap \choose \numpap} \geq \frac{2\sqrt{\pi}}{e^2 \sqrt{\numpap}} 2^{2\numpap} \geq \frac{2\sqrt{\pi}}{e^2} \exp(1.35\numpap - 0.5\ln(\numpap))$. Therefore, the probability that less than $\numpap/100$ grid points have both reviewers in stage one is at most $\frac{e^2}{2\sqrt{\pi}} \exp( -0.2 \numpap + \ln(0.01 \numpap) + 0.5\ln(\numpap) ) \leq \frac{e^2}{2\sqrt{\pi}} \exp( -0.1 \numpap) \exp(-0.1\numpap + \ln(0.01 \numpap) + 0.5\ln(\numpap) ) \leq \frac{e^2}{2\sqrt{\pi}} \exp( -0.1 \numpap)$ for sufficiently large $\numpap$.
\end{proof}
Therefore, with high probability, at least $\numpap/100$ reviewers must be assigned to a paper at a different grid point.

Consider the assignments produced  in each stage after random split, and consider the reviewers not assigned to their optimal papers by these assignments. From the set of vectors corresponding to these suboptimally-assigned reviewers, we can construct some number $K$ of disjoint cycles $C_j = \{x_1^{(j)}, \dots, x_{K_j}^{(j)}\}$, where a reviewer with vector $x_i^{(j)}$ is assigned to the paper with vector $x_{i+1}^{(j)}$ when the optimal assignment would assign them to the paper with vector $x_i^{(j)}$. By Lemma~\ref{lem:revconc}, $\sum_{j=1}^{K} |C_j| \geq \frac{\numpap}{100}$. The difference in value between the random-split assignments and the optimal assignment is
\begin{align*}
    \sum_{j=1}^{K} \sum_{i=1}^{K_j} \langle x_i^{(j)}, x_i^{(j)} - x_{i+1}^{(j)} \rangle
    &= \frac{1}{2} \sum_{j=1}^{K} \sum_{i=1}^{K_j} || x_i^{(j)} - x_{i+1}^{(j)} ||_2^2 \\
    &\geq \frac{1}{8\rank} \numpap^{-2/\rank} \sum_{j=1}^{K} |C_j| \\
    &\geq \frac{1}{8\rank} \numpap^{-2/\rank} \left( \frac{\numpap}{100} \right) 
\end{align*}
with probability at least $1 - \zeta e^{-\numpap/10}$ for sufficiently large $\numpap$. Dividing by $2\numpap$, the suboptimality is at least $\frac{1}{1600 \rank} \numpap^{-2/\rank}$.

\subsection{Proof of Theorem~\ref{thm:oneroundub}} \label{apdx:oneroundub}
In this section, we state and prove a more general version of the bound in Theorem~\ref{thm:oneroundub} that does not require that $\stagetwofrac \loadscale$ be integral. This result immediately implies the result of  Theorem~\ref{thm:oneroundub}.

In the proof, we use the following lemma. We prove this lemma following the proof of the main theorem. For some set $\mathcal{N}$ and some constant $p \in [0, 1]$, define distribution $\mathcal{I}_{p}(\mathcal{N})$ as the distribution over all subsets $A \subseteq \mathcal{N}$ induced by choosing each item $x \in \mathcal{N}$ to be in $A$ independently with probability $p$. 
%Recall that for subsets $\revset' \subseteq \revset$ and $\papset' \subseteq \papset$, desired paper load $\papload$, and maximum reviewer load $\revload$, $\mathcal{M}(\revset', \papset'; \revload, \papload)$ is the set of assignment matrices that assign a load of exactly $\papload$ to all papers in $\papset'$ (and a load of at most $\revload$ to all reviewers in $\revset'$). Define $\mathcal{M}'(\revset', \papset'; \revload, \papload)$ as the set of assignment matrices that assign a load of at most $\papload$ to all papers in $\papset'$ (and a load of at most $\revload$ to all reviewers in $\revset'$). Formally, $\adassign \in \mathcal{M}'(\revset', \papset'; \revload, \papload)$ if and only if $\sum_{\adrev \in \revset'} \adassign_{\adrev \adpap} \leq \papload$ for all $\adpap \in \papset'$, $\sum_{\adpap \in \papset'} \adassign_{\adrev \adpap} \leq \revload$ for all $\adrev \in \revset'$, and $\adassign_{\adrev \adpap} = 0$ for all $(\adrev, \adpap) \not\in \revset' \times \papset'$. 
Recall from Appendix~\ref{apdx:submod} the definition of $\valfn'$, a modified version of $\valfn$ that allows papers to be underloaded (i.e., assigned fewer reviewers than their load). 
\begin{lemma} \label{lem:indep}
Consider the modified version of $\objfn$: 
$\objfn'(\revset_2) = \mathbb{E}_{\papset_2 \sim \mathcal{I}_{\stagetwofrac}(\papset)}\left[\valfn'(\revset_2, \papset_2)\right]$.%, where
%\begin{align*}
%    \valfn'(\revset_2, \papset_2) &= \frac{1}{(1+\stagetwofrac) \numpap} \left[  \max_{A \in \mathcal{M}'(\revset \setminus \revset_2, \papset; 1, 1)} \sum_{\adrev \in \revset \setminus \revset_2, \adpap \in \papset} A_{\adrev \adpap} \simmat_{\adrev \adpap} +  \max_{B \in \mathcal{M}'(\revset_2, \papset_2; 1, 1)} \sum_{\adrev \in \revset_2, \adpap \in \papset_2} B_{\adrev \adpap} \simmat_{\adrev \adpap} \right].
%\end{align*} 
$\objfn'$ draws $\papset_2 \sim \mathcal{I}_{\stagetwofrac}(\papset)$ rather than $\papset_2 \sim \mathcal{U}_{\stagetwofrac \numpap}(\papset)$ and allows papers to be underloaded. Then,
\begin{align*}
    \mathbb{E}_{\revset_2 \sim \mathcal{I}_{\stagetwofrac / (1 + \stagetwofrac)}(\revset)}\left[ \objfn'(\revset_2) \right] \leq \mathbb{E}_{\revset_2 \sim \mathcal{U}_{(\stagetwofrac / (1 + \stagetwofrac)) \numrev}(\revset)}[\objfn(\revset_2)].
\end{align*}
\end{lemma}
This lemma shows that when attempting to lower bound the expected value of random split, we can analyze as if the second-stage reviewers and papers were drawn independently. 

We now state and prove the main theorem. We abuse notation slightly by defining $\mathcal{M}(\revset, \papset; \revload, \papload)$ to include all assignments where papers are assigned either $\lfloor \papload \rfloor$ or $\lceil \papload \rceil$ reviewers when $\papload$ is not integral; i.e., $\adassign \in \mathcal{M}(\revset', \papset'; \revload, \papload)$ if and only if $\lfloor \papload \rfloor \leq \sum_{\adrev \in \revset'} \adassign_{\adrev \adpap} \leq \lceil \papload \rceil$ for all $\adpap \in \papset'$, $\sum_{\adpap \in \papset'} \adassign_{\adrev \adpap} \leq \revload$ for all $\adrev \in \revset'$, and $\adassign_{\adrev \adpap} = 0$ for all $(\adrev, \adpap) \not\in \revset' \times \papset'$. 
\setcounter{theorem}{4}
\begin{theorem}[Generalized]
Consider any $\loadscale \in [10,000]$ and $\stagetwofrac \in \left\{ \frac{1}{100}, \dots, \frac{100}{100} \right\}$. 
Let $\epsilon = \lceil \stagetwofrac \loadscale \rceil - \lfloor \stagetwofrac \loadscale \rfloor$. 
If there exists an assignment $\adassign^{(\loadscale)} \in \mathcal{M}(\revset, \papset; \loadscale, (1+\stagetwofrac )\loadscale)$ with mean similarity $\scaledavgval$,
choosing $\revset_2$ via random split gives that
\begin{align*} 
&\mathbb{E}_{\revset_2}\left[\objfn(\revset_2)\right] \geq \\
&\scaledavgval \Bigg[  1
-  \sqrt{\frac{\stagetwofrac}{2 \pi (1+\stagetwofrac ) \lfloor (1 + \stagetwofrac) \loadscale \rfloor}}\left( 2 \sqrt{\frac{1}{1+\stagetwofrac}} 
+ \sqrt{1 -\stagetwofrac}  \right)
- \frac{(1 + 2 \stagetwofrac)}{(1+\stagetwofrac) \lceil \stagetwofrac \loadscale \rceil } \epsilon \Bigg] 
\left[ 1 - \frac{\epsilon}{\lceil(1+\stagetwofrac )\loadscale\rceil} \right].
\end{align*}
\end{theorem}
\begin{proof}
By Lemma~\ref{lem:indep}, we can consider drawing $\papset_2 \sim \mathcal{I}_{\stagetwofrac}(\papset)$ and $\revset_2 \sim \mathcal{I}_{\stagetwofrac / (1 + \stagetwofrac)}(\revset)$ and allowing papers to be underloaded. For all reviewers $\adrev \in \revset$, define the random variables $Z_{\adrev} = \begin{cases} 1 \text{ w.p. } 1 / (1+\stagetwofrac ) \\ 2 \text{ w.p. } \stagetwofrac  / (1+\stagetwofrac ) \end{cases}$ representing the stage that reviewer $\adrev$ is randomly chosen to be in. Define the random variables $Y_{\adpap} = \begin{cases} 1 \text{ w.p. } \stagetwofrac  \\ 0 \text{ w.p. } 1-\stagetwofrac  \end{cases}$ representing whether $\adpap \in \papset_2$. All of these random variables are independent. Also, denote by $\scaledval = \scaledavgval (1+\stagetwofrac) \numpap \loadscale$ the total similarity value of assignment $\adassign^{(\loadscale)}$, and denote by $\scaledval_{\adpap}$ and $\scaledval_{\adrev}$ the total similarity value of the assignments for paper $\adpap$ and reviewer $\adrev$ respectively in assignment $\adassign^{(\loadscale)}$.

The proof works as follows. We form an assignment $B^{(1)}$ in stage one with paper loads of at most $\loadscale$ and reviewer loads of at most $\loadscale$, and form an assignment $B^{(2)}$ in stage two with paper loads of at most $\lceil \stagetwofrac \loadscale \rceil$ and reviewer loads of at most $ \lceil \stagetwofrac \loadscale \rceil$. We do this by initially assigning all reviewer-paper pairs from $\adassign^{(\loadscale)}$ that are present in the same stage, and then randomly removing assignments from each paper or reviewer that is overloaded.  We then find ``final assignments'' (i.e., assignments that are feasible solutions for the two-stage assignment problem) from within $B^{(1)}$ and $B^{(2)}$. 

\paragraph{Stage One:}
First, consider stage one. Define $Binom(N, p)$ as the binomial distribution with $N$ trials and $p$ probability of success; denote by $f$ the Binomial pmf. The number of reviewers assigned by $\adassign^{(\loadscale)}$ to paper $\adpap$ and present in stage one is a $Binom\left(\lambda_{\adpap}, \frac{1}{1+\stagetwofrac}\right)$ variable, where $\lambda_{\adpap} \in \{\lfloor(1+\stagetwofrac )\loadscale\rfloor, \lceil(1+\stagetwofrac )\loadscale\rceil\}$.  Suppose that we observe the set of such reviewers, randomly remove reviewers from this set until its size is at most $\loadscale$, and then assign these reviewers to $\adpap$ in our stage one assignment $B^{(1)}$. Since each reviewer has at most $\loadscale$ assigned papers in $\adassign^{(\loadscale)}$, $B^{(1)}$ satisfies the desired load constraints on both sides.  The expected total value of the assigned reviewers after we drop reviewers from each paper at random is
\begin{align*} 
\mathbb{E}\left[\sum_{\adrev \in \revset} B^{(1)}_{\adrev \adpap} \simmat_{\adrev \adpap}\right] &= \sum_{x = 0}^{\loadscale} f\left(x; \lambda_{\adpap}, \frac{1}{1+\stagetwofrac} \right) \scaledval_p \frac{x}{\lambda_{\adpap}} + \sum_{x = \loadscale+1}^{\lambda_{\adpap}} f\left(x; \lambda_{\adpap}, \frac{1}{1+\stagetwofrac}\right) \scaledval_p \frac{\loadscale}{\lambda_{\adpap}} \\
&= \frac{\scaledval_p}{\lambda_{\adpap}} \mathbb{E}_{X \sim Binom\left(\lambda_{\adpap}, \frac{1}{1+\stagetwofrac}\right)}\left[\min(X, \loadscale) \right] \\
&\geq \frac{\scaledval_p}{\lceil(1+\stagetwofrac )\loadscale\rceil} \mathbb{E}_{X \sim Binom\left(\lfloor(1+\stagetwofrac )\loadscale\rfloor, \frac{1}{1+\stagetwofrac}\right)}\left[\min(X, \loadscale) \right].
\end{align*}
Summing over all papers, 
\begin{align*}
\mathbb{E}\left[\sum_{\adpap \in \papset} \sum_{\adrev \in \revset} B^{(1)}_{\adrev \adpap} \simmat_{\adrev \adpap}\right] &\geq \frac{\scaledval}{\lceil(1+\stagetwofrac )\loadscale\rceil} \mathbb{E}_{X \sim Binom\left(\lfloor(1+\stagetwofrac )\loadscale\rfloor, \frac{1}{1+\stagetwofrac}\right)}\left[\min(X, \loadscale) \right].
\end{align*}

Due to the loads, the matrix $\frac{1}{\loadscale} B^{(1)}$ has row sums at most $1$ and column sums at most $1$. By a generalization of the Birkhoff-von Neumann theorem~\cite{Budish2009IMPLEMENTINGRA}, this can be written as a convex combination of matrices with all entries in $\{0, 1\}$, all row sums at most $1$, and all column sums at most $1$. Each of these matrices represents an assignment obeying the reviewer and paper load constraints for the final assignment (since we allow papers to be underloaded), so they are all valid final assignments in stage one. At least one of these assignments must have a total value at least $\frac{1}{\loadscale}$ of the value of $B^{(1)}$.

\paragraph{Stage Two:}
Now, consider stage two. The number of reviewers assigned by $\adassign^{(\loadscale)}$ to paper $\adpap$ present in stage two is a $Binom\left(\lambda_{\adpap}, \frac{\stagetwofrac}{1+\stagetwofrac}\right)$ variable, where $\lambda_{\adpap} \in \{\lfloor(1+\stagetwofrac )\loadscale\rfloor, \lceil(1+\stagetwofrac )\loadscale\rceil\}$. The number of papers assigned by $\adassign^{(\loadscale)}$ to a reviewer $\adrev$ present in stage two is a $Binom(\loadscale, \stagetwofrac)$ random variable. We first calculate the total expected value of all assignments in $\adassign^{(\loadscale)}$ and present in stage two (without dropping assignments from overloaded reviewers/papers):
\begin{align*}
\mathbb{E}\left[\sum_{\adrev \in \revset_2, \adpap \in \papset_2} \adassign^{(\loadscale)}_{\adrev \adpap} \simmat_{\adrev \adpap}\right] &= \frac{\stagetwofrac^2 }{1+\stagetwofrac } \scaledval.
\end{align*}
We then construct assignment $B^{(2a)}$ from the pairs assigned in $\adassign^{(\loadscale)}$ and present in stage two by dropping reviewers from each paper at random until all papers have a load of at most $\lceil \stagetwofrac \loadscale \rceil$, with a value on paper $\adpap$ (if present in stage two) of
\begin{align*}
\mathbb{E}\left[\sum_{\adrev \in \revset} B^{(2a)}_{\adrev \adpap} \simmat_{\adrev \adpap} \Big| \adpap \in \papset_2 \right] &= \sum_{x = 0}^{\lceil \stagetwofrac \loadscale \rceil} f\left(x; \lambda_{\adpap}, \frac{\stagetwofrac}{1+\stagetwofrac}\right) \scaledval_p \frac{x}{\lambda_{\adpap}} + \sum_{x = \lceil \stagetwofrac \loadscale \rceil+1}^{\lambda_{\adpap}} f\left(x; \lambda_{\adpap}, \frac{\stagetwofrac}{1+\stagetwofrac}\right) \scaledval_p \frac{\lceil \stagetwofrac \loadscale \rceil}{\lambda_{\adpap}} \\
&= \frac{\scaledval_p}{\lambda_{\adpap}} \mathbb{E}_{X \sim Binom\left(\lambda_{\adpap}, \frac{\stagetwofrac}{1+\stagetwofrac}\right)}\left[ \min\left(X, \lceil \stagetwofrac \loadscale \rceil\right) \right] \\
&\geq \frac{\scaledval_p}{\lceil(1+\stagetwofrac )\loadscale\rceil} \mathbb{E}_{X \sim Binom\left(\lfloor(1+\stagetwofrac )\loadscale\rfloor, \frac{\stagetwofrac}{1+\stagetwofrac}\right)}\left[ \min\left(X, \lceil \stagetwofrac \loadscale \rceil\right) \right].
\end{align*}
Each paper is present in stage two with probability $\stagetwofrac$, so the overall value is
\begin{align*}
\mathbb{E}\left[\sum_{\adpap \in \papset} \sum_{\adrev \in \revset} B^{(2a)}_{\adrev \adpap} \simmat_{\adrev \adpap}\right] &\geq \frac{\stagetwofrac \scaledval}{\lceil(1+\stagetwofrac )\loadscale\rceil} \mathbb{E}_{X \sim Binom\left(\lfloor(1+\stagetwofrac )\loadscale\rfloor, \frac{\stagetwofrac}{1+\stagetwofrac}\right)}\left[ \min\left(X, \lceil \stagetwofrac \loadscale \rceil\right) \right].
\end{align*} 
We separately construct assignment $B^{(2b)}$ from the pairs assigned in $\adassign^{(\loadscale)}$ and present in stage two by dropping papers from each reviewer at random until all reviewers have a load of at most $\lceil \stagetwofrac \loadscale \rceil$, with a value on reviewer $\adrev$ (if present in stage two) of
\begin{align*}
\mathbb{E}\left[\sum_{\adpap \in \papset} B^{(2b)}_{\adrev \adpap} \simmat_{\adrev \adpap} \Big| \adrev \in \revset_2 \right] &= \sum_{x = 0}^{ \lceil \stagetwofrac \loadscale \rceil} f\left(x;  \loadscale, \stagetwofrac \right) \scaledval_r \frac{x}{\loadscale} + \sum_{x = \lceil \stagetwofrac \loadscale \rceil+1}^{\loadscale} f\left(x; \loadscale, \stagetwofrac \right) \scaledval_r \frac{\lceil \stagetwofrac \loadscale \rceil}{\loadscale} \\
&= \frac{\scaledval_r}{\loadscale} \mathbb{E}_{X \sim Binom\left(\loadscale, \stagetwofrac\right)}\left[ \min\left(X,  \lceil \stagetwofrac \loadscale \rceil\right) \right].
\end{align*}
Totalling across all reviewers, since each reviewer is present in stage two with probability $\frac{\stagetwofrac}{1 + \stagetwofrac}$,
\begin{align*}
\mathbb{E}\left[\sum_{\adrev \in \revset} \sum_{\adpap \in \papset} B^{(2b)}_{\adrev \adpap} \simmat_{\adrev \adpap}\right] &= \frac{\stagetwofrac  \scaledval}{(1+\stagetwofrac )\loadscale} \mathbb{E}_{X \sim Binom\left(\loadscale, \stagetwofrac\right)}\left[ \min\left(X, \lceil \stagetwofrac \loadscale \rceil\right) \right].
\end{align*}
Define $B^{(2)}$ as the intersection of the assigned pairs in $B^{(2a)}$ and $B^{(2b)}$; $B^{(2)}$ satisfies the desired load constraints on both sides. Its expected value is lower-bounded by the total expected value of $B^{(2a)}$ and $B^{(2b)}$ less the expected value of the pairs assigned in $\adassign^{(\loadscale)}$ and present in stage two, since the pairs assigned in $B^{(2a)}$ and $B^{(2b)}$ are subsets of the stage two pairs assigned in $\adassign^{(\loadscale)}$.
\begin{align*} 
\mathbb{E}\left[\sum_{\adrev \in \revset, \adpap \in \papset} B^{(2)}_{\adrev \adpap} \simmat_{\adrev \adpap} \right] &\geq \frac{\stagetwofrac \scaledval}{\lceil(1+\stagetwofrac )\loadscale\rceil} \mathbb{E}_{X \sim Binom\left(\lfloor(1+\stagetwofrac )\loadscale\rfloor, \frac{\stagetwofrac}{1+\stagetwofrac}\right)}\left[ \min\left(X, \lceil \stagetwofrac \loadscale \rceil\right) \right] \\
&\qquad + \frac{\stagetwofrac  \scaledval}{(1+\stagetwofrac )\loadscale} \mathbb{E}_{X \sim Binom\left(\loadscale, \stagetwofrac\right)}\left[ \min\left(X,  \lceil \stagetwofrac \loadscale \rceil\right) \right]. \\
&\qquad -  \frac{\stagetwofrac^2 }{1+\stagetwofrac } \scaledval.
\end{align*}

By the same Birkhoff-von Neumann argument as used in stage one, there exists a valid final assignment in stage two with paper loads of at most $1$, reviewer loads of at most $1$, and value at least $\frac{1}{\lceil \stagetwofrac \loadscale \rceil}$ of the value of $B^{(2)}$.

\paragraph{Total:}
Sum the total value of the $1$-load assignment in both stages and divide by $(1+\stagetwofrac )\numpap$ to get a lower bound on the expected mean similarity:
\begin{align}
&\scaledavgval \Bigg[\frac{\loadscale}{\lceil(1+\stagetwofrac )\loadscale\rceil} \mathbb{E}_{X \sim Binom\left(\lfloor(1+\stagetwofrac )\loadscale\rfloor, \frac{1}{1+\stagetwofrac}\right)}\left[\min\left(\frac{X}{\loadscale}, 1\right) \right] \nonumber \\
&\qquad + \frac{\stagetwofrac \loadscale}{\lceil(1+\stagetwofrac )\loadscale\rceil} \mathbb{E}_{X \sim Binom\left(\lfloor(1+\stagetwofrac )\loadscale\rfloor, \frac{\stagetwofrac}{1+\stagetwofrac}\right)}\left[ \min\left(\frac{X}{\lceil \stagetwofrac \loadscale \rceil}, 1 \right) \right] \nonumber \\
&\qquad + \frac{\stagetwofrac}{(1+\stagetwofrac ) } \mathbb{E}_{X \sim Binom\left( \loadscale, \stagetwofrac\right)}\left[ \min\left(\frac{X}{ \lceil \stagetwofrac \loadscale \rceil}, 1 \right) \right]
-  \frac{\stagetwofrac^2 \loadscale}{(1+\stagetwofrac) \lceil \stagetwofrac \loadscale \rceil} \Bigg] \nonumber  \\
&\geq \scaledavgval \left( \frac{\loadscale}{\lceil(1+\stagetwofrac )\loadscale\rceil} \right)  \Bigg[ \mathbb{E}_{X \sim Binom\left(\lfloor(1+\stagetwofrac )\loadscale\rfloor, \frac{1}{1+\stagetwofrac}\right)}\left[\min\left(\frac{X}{\loadscale}, 1\right) \right] \nonumber \\
&\qquad + \stagetwofrac \left( \mathbb{E}_{X \sim Binom\left(\lfloor(1+\stagetwofrac )\loadscale\rfloor, \frac{\stagetwofrac}{1+\stagetwofrac}\right)}\left[ \min\left(\frac{X}{\lceil \stagetwofrac \loadscale \rceil}, 1 \right) \right] \right. \nonumber \\
&\qquad \qquad\qquad+ \left. \mathbb{E}_{X \sim Binom\left( \loadscale, \stagetwofrac\right)}\left[ \min\left(\frac{X}{ \lceil \stagetwofrac \loadscale \rceil}, 1 \right) \right]
-  1 \right) \Bigg] . \label{eq:oneroundlbexact} 
\end{align}

Since the above bound is a function of the binomial pmf, we search for a simpler approximation. Say that $X \sim Binom(N, p)$ and $q=1-p$. The above bound is a function of $\mathbb{E}\left[\min\left(\frac{X}{\ell}, 1\right)\right]$ for three binomial random variables where $N p \leq \ell$. We approximate these binomials as if they were normals $Z \sim \mathcal{N}(Np, Npq)$, since $\frac{X - Np}{\sqrt{Npq}}$ converges in distribution to a standard normal. We use $f_{Z}$ as the pdf of $Z$, $F_Z$ as the cdf of $Z$, and $\Phi$ as the standard normal cdf.

\begin{align*}
   \mathbb{E}\left[\min\left(\frac{X}{\ell}, 1\right)\right] &\approx \mathbb{E}\left[\min\left(\frac{Z}{\ell}, 1\right)\right] \\
   &= \mathbb{E}\left[\min\left(\frac{Z}{\ell}, 1\right) | Z \leq \ell \right] P[Z \leq \ell] + \mathbb{E}\left[\min\left(\frac{Z}{\ell}, 1\right) | Z > \ell \right] P[Z > \ell] \\
   &= \frac{1}{\ell} \left(Np - Npq \frac{f_{Z}(\ell)}{F_{Z}(\ell)} \right) F_{Z}(\ell) + 1 -  F_{Z}(\ell) \\
   &=  1  - \frac{Npq}{\ell} f_{Z}(\ell)  - F_{Z}(\ell) \left(1 - \frac{Np}{\ell}  \right) \\
   &\geq 1 - \sqrt{\frac{q}{2 \pi N p}} - \Phi\left(\frac{\ell - Np}{\sqrt{N p q}}\right) \left(1 - \frac{Np}{\ell} \right) \\
   &\geq 1 - \sqrt{\frac{q}{2 \pi N p}} - \left(1 - \frac{Np}{\ell} \right) .
\end{align*}

In total, defining $\epsilon^{+} = \lceil \stagetwofrac \loadscale \rceil - \stagetwofrac \loadscale$, $\epsilon^{-} = \stagetwofrac \loadscale - \lfloor \stagetwofrac \loadscale \rfloor$, and $\epsilon = \epsilon^{+} + \epsilon^{-}$, this approximation to the lower bound gives
\begin{align}
&\scaledavgval \left(\frac{\loadscale}{\lceil(1+\stagetwofrac )\loadscale\rceil}\right)  \Bigg[  1 - \sqrt{\frac{\stagetwofrac}{2 \pi \lfloor (1 + \stagetwofrac) \loadscale \rfloor}} - \Phi\left(\frac{\loadscale - \frac{\lfloor (1 + \stagetwofrac) \loadscale \rfloor}{1 + \stagetwofrac}}{\sqrt{\frac{\lfloor (1 + \stagetwofrac) \loadscale \rfloor \stagetwofrac}{(1 + \stagetwofrac)^2}}}\right) \left(1 - \frac{\lfloor (1 + \stagetwofrac) \loadscale \rfloor}{(1 + \stagetwofrac) \loadscale} \right) \nonumber \\
&\qquad + \stagetwofrac \left( 1 
- \sqrt{\frac{1}{2 \pi \lfloor (1 + \stagetwofrac) \loadscale \rfloor \stagetwofrac}} 
- \sqrt{\frac{1 -\stagetwofrac}{2 \pi \loadscale  \stagetwofrac}} \right. \nonumber \\
&\qquad\qquad \left. 
- \Phi\left(\frac{\lceil \stagetwofrac \loadscale \rceil - \frac{\lfloor (1 + \stagetwofrac) \loadscale \rfloor \stagetwofrac}{1 + \stagetwofrac}}{\sqrt{\frac{\lfloor (1 + \stagetwofrac) \loadscale \rfloor \stagetwofrac}{(1 + \stagetwofrac)^2}}}\right) \left(1 - \frac{\stagetwofrac \lfloor (1 + \stagetwofrac) \loadscale \rfloor}{\lceil \stagetwofrac \loadscale \rceil (1 + \stagetwofrac)} \right)
- \Phi\left(\frac{ \lceil \stagetwofrac \loadscale \rceil -  \stagetwofrac \loadscale}{\sqrt{(1 - \stagetwofrac) \stagetwofrac  \loadscale}}\right) \left(1 - \frac{\stagetwofrac \loadscale}{\lceil \stagetwofrac \loadscale \rceil} \right)
\right) \Bigg] \label{eq:oneroundlbapprox} \\
&\geq \scaledavgval \Bigg[  1
- \frac{2}{(1+\stagetwofrac )}\sqrt{\frac{\stagetwofrac}{2 \pi \lfloor (1 + \stagetwofrac) \loadscale \rfloor}} 
- \frac{1}{(1+\stagetwofrac )}\sqrt{\frac{\stagetwofrac(1 -\stagetwofrac)}{2 \pi \loadscale }}  \nonumber \\
&\qquad- \frac{1 + 2 \stagetwofrac}{1+\stagetwofrac }\left(1 - \frac{\stagetwofrac \lfloor (1 + \stagetwofrac) \loadscale \rfloor}{\lceil \stagetwofrac \loadscale \rceil (1 + \stagetwofrac)} 
\right) \Bigg]\left(\frac{(1+\stagetwofrac )\loadscale}{\lceil(1+\stagetwofrac )\loadscale\rceil}\right)  \nonumber  \\
&\geq  \scaledavgval \Bigg[  1
-  \sqrt{\frac{\stagetwofrac}{2 \pi (1+\stagetwofrac ) \lfloor (1 + \stagetwofrac) \loadscale \rfloor}}\left( 2 \sqrt{\frac{1}{1+\stagetwofrac}} 
+ \sqrt{1 -\stagetwofrac}  \right) \nonumber \\
&\qquad- \frac{(1 + 2 \stagetwofrac)}{(1+\stagetwofrac) \lceil \stagetwofrac \loadscale \rceil }\left(\frac{\stagetwofrac}{1+\stagetwofrac}  \epsilon^{-} + \epsilon^{+}
\right) \Bigg] \left( 1 - \frac{\epsilon^{+}}{\lceil(1+\stagetwofrac )\loadscale\rceil} \right) \nonumber \\
&\geq \scaledavgval \Bigg[  1
-  \sqrt{\frac{\stagetwofrac}{2 \pi (1+\stagetwofrac ) \lfloor (1 + \stagetwofrac) \loadscale \rfloor}}\left( 2 \sqrt{\frac{1}{1+\stagetwofrac}} 
+ \sqrt{1 -\stagetwofrac}  \right) \nonumber \\
&\qquad- \frac{(1 + 2 \stagetwofrac)}{(1+\stagetwofrac) \lceil \stagetwofrac \loadscale \rceil } \epsilon \Bigg] \left( 1 - \frac{\epsilon}{\lceil(1+\stagetwofrac )\loadscale\rceil} \right) \nonumber .
\end{align}
Via simulation, we confirm that the approximation in~\eqref{eq:oneroundlbapprox} is in fact a lower bound on the expression in~\eqref{eq:oneroundlbexact} for all $\loadscale \in [10^4]$ and $\stagetwofrac \in \left\{ \frac{1}{100}, \dots, \frac{100}{100} \right\}$. In Figure~\ref{fig:bounds} of Section~\ref{sec:cond2}, we plot the more precise bound of~\eqref{eq:oneroundlbapprox}. 
\end{proof}

\subsubsection{Proof of Lemma~\ref{lem:indep}} \label{apdx:indeplem}
It remains to prove Lemma~\ref{lem:indep}. We first prove a supplementary lemma, from which the main lemma follows. 
\begin{lemma} \label{lem:indep2} 
Consider a set $\mathcal{N}$ of $N$ items and a submodular function $g : 2^\mathcal{N} \to \mathbb{R}$. Then, $E_{A \sim \mathcal{I}_{p}(\mathcal{N})}[g(A)] \leq E_{A \sim \mathcal{U}_{p N}(\mathcal{N})}[g(A)]$.
\end{lemma}
\begin{proof}
Consider the following randomized procedure $h$, which takes in a set $D \subseteq \mathcal{N}$ and constructs a set containing exactly $pN$ items. If $|D|=pN$, return $h(D)=D$. If $x = |D| - pN > 0$, then choose a subset $B \subseteq D$ uniformly at random such that $|B| = x$ and return $h(D) = D \setminus B$. If $x = pN - |D| > 0$, choose a subset $C \subseteq \mathcal{N} \setminus D$ uniformly at random such that $|C| = x$ and return $h(D) = D \cup C$. 

If $D \sim \mathcal{I}_{p}(\mathcal{N})$ then $h(D) \sim \mathcal{U}_{p N}(\mathcal{N})$, since all subsets of size $pN$ have an equal chance to be created. We will show that $\mathbb{E}_{D \sim \mathcal{I}_{p}(\mathcal{N})}[g(h(D)) - g(D)] \geq 0$, proving that $\mathbb{E}_{D \sim \mathcal{I}_{p}(\mathcal{N})}[g(D)] \leq \mathbb{E}_{A \sim \mathcal{U}_{p N}(\mathcal{N})}[g(A)]$. More specifically, we show that for each $x > 0$, 
\begin{align*}
    \mathbb{E}_{D \sim \mathcal{I}_{p}(\mathcal{N})}[g(h(D)) - g(D) \mid |D| = pN + x] + \mathbb{E}_{D \sim \mathcal{I}_{p}(\mathcal{N})}[g(h(D)) - g(D) \mid |D| = pN - x] \geq 0.
\end{align*} 

Since $g$ is submodular, for any subsets $A \subseteq \mathcal{N}$, $C \subseteq A$, $B \subseteq \mathcal{N} \setminus A$, we have that $g((A \setminus C) \cup B) - g(A \setminus C) \geq g(A \cup B) - g(A)$. 
\begin{align}
& \mathbb{E}_{D \sim \mathcal{I}_{p}(\mathcal{N})}[g(h(D)) - g(D) \mid |D| = pN + x] \nonumber \\
&\qquad + \mathbb{E}_{D \sim \mathcal{I}_{p}(\mathcal{N})}[g(h(D)) - g(D) \mid  |D| = pN - x] \nonumber  \\
&= \frac{1}{{N \choose pN+x} {pN+x \choose x}} \sum_{D \subseteq \mathcal{N} : |D| = pN + x} \sum_{B \subseteq D : |B| = x}  g(D \setminus B) - g(D) \nonumber  \\
&\qquad + \frac{1}{{N \choose pN-x} {N - pN +x \choose x}} \sum_{D \subseteq \mathcal{N} : |D| = pN - x}  \sum_{C \subseteq \mathcal{N} \setminus D : |C| = x}  g(D \cup C) - g(D) \label{eq:submod_expec} \\
&= \frac{1}{{N \choose pN} {N-pN \choose x}} \sum_{A \subseteq \mathcal{N} : |A| = pN} \sum_{B \subseteq \mathcal{N} \setminus A : |B| = x}  g(A) - g(A \cup B) \nonumber  \\
&\qquad +  \frac{1}{{N \choose pN} {pN \choose x}} \sum_{A \subseteq \mathcal{N} : |A| = pN}  \sum_{C \subseteq A : |C| = x}  g(A) - g(A \setminus C)\label{eq:submod_rewrite1} \\
&= \frac{1}{{N \choose pN} {pN \choose x} {N-pN \choose x}} \nonumber \\
&\qquad \sum_{A \subseteq \mathcal{N} : |A| = pN} \sum_{B \subseteq \mathcal{N} \setminus A : |B| = x} \sum_{C \subseteq A : |C| = x} g(A) -  g(A \cup B) +    g(A) - g(A \setminus C) \nonumber  \\
&= \frac{1}{{N \choose pN} {pN \choose x} {N-pN \choose x}} \nonumber \\
&\qquad \sum_{A \subseteq \mathcal{N} : |A| = pN} \sum_{B \subseteq \mathcal{N} \setminus A : |B| = x} \sum_{C \subseteq A : |C| = x} g(A) -  g(A \cup B) +    g((A \setminus C) \cup B) - g(A \setminus C)  \label{eq:submod_rewrite2} \\
&\geq 0 \nonumber.
\end{align}
\eqref{eq:submod_expec} writes out the expected value as a sum over all choices of $D$ and all sets sampled by the procedure $h$. \eqref{eq:submod_rewrite1} rewrites the sums using $A = D \setminus B$ and $A = D \cup C$; each choice of $D, B$ in the original sum corresponds to exactly one choice of $A, B$ in the new sum. \eqref{eq:submod_rewrite2} re-arranges the sum to exchange each $g(A)$ term for a $g((A \setminus C) \cup B)$ term; in both cases each set of size $pN$ is counted ${pN \choose x} {N-pN \choose x}$ times in the sum (exactly once for each choice of $B, C$). 
\end{proof}

We also show that $\objfn'$ and $\valfn'$ are submodular.
\begin{proposition} \label{prop:submod2} 
$\valfn'(\revset_2, \papset_2)$ is submodular in $\revset_2$ and $\papset_2$. Further, $\objfn'$ is submodular in $\revset_2$. 
\end{proposition}
\begin{proof}
Note that $\max_{\adassign \in \mathcal{M}'(\revset', \papset'; 1, 1)} \sum_{\adrev \in \revset', \adpap \in \papset'} \adassign_{\adrev \adpap} \simmat_{\adrev \adpap}$ is a submodular function of the reviewer set $\revset'$ when the paper set $\papset'$ is held fixed and of the paper set $\papset'$ when the reviewer set is held fixed~\cite{kulik2019generalized}. Submodularity in $\revset_2$ is equivalent to submodularity in $\revset_1 = \revset \setminus \revset_2$, so $\valfn'(\revset_2, \papset_2)$ is submodular in $\revset_2$ and $\papset_2$. As a sum over terms submodular in $\revset_2$, $\objfn'$ is submodular in $\revset_2$.
\end{proof}

We now prove the main lemma. Since $\simmat \geq 0$, there exists a maximum-similarity assignment from within $\mathcal{M}'(\revset', \papset'; 1, 1)$ that meets all paper load constraints with equality when $|\revset'| \geq |\papset'|$, and thus is contained in $\mathcal{M}(\revset', \papset'; 1, 1)$. Also, $\mathcal{M}(\revset', \papset'; 1, 1) \subseteq \mathcal{M}'(\revset', \papset'; 1, 1)$. 
Thus, when $|\revset_2| \geq \stagetwofrac\numpap$ and $\numrev - |\revset_2| \geq \numpap$, $\valfn(\revset_2, \papset_2) = \valfn'(\revset_2, \papset_2)$. Further, by Proposition~\ref{prop:submod2}, $\valfn'$ is submodular in $\papset_2$. Therefore, by Lemma~\ref{lem:indep2}, $\objfn(\revset_2) \geq \objfn'(\revset_2)$ whenever $|\revset_2| = \frac{\stagetwofrac}{1 + \stagetwofrac} \numrev$ (since $\numrev \geq (1 + \stagetwofrac) \numpap$). This shows that
\begin{align*}
    \mathbb{E}_{\revset_2 \sim \mathcal{U}_{(\stagetwofrac / (1 + \stagetwofrac)) \numrev}(\revset)}[\objfn(\revset_2)] \geq \mathbb{E}_{\revset_2 \sim \mathcal{U}_{(\stagetwofrac / (1 + \stagetwofrac)) \numrev}(\revset)}[\objfn'(\revset_2)].
\end{align*}

By Proposition~$\ref{prop:submod2}$, $\objfn'$ is submodular in $\revset_2$. Therefore, by Lemma~\ref{lem:indep2},
\begin{align*}
    \mathbb{E}_{\revset_2 \sim \mathcal{U}_{(\stagetwofrac / (1 + \stagetwofrac)) \numrev}(\revset)}[\objfn'(\revset_2)] \geq \mathbb{E}_{\revset_2 \sim \mathcal{I}_{\stagetwofrac / (1 + \stagetwofrac)}(\revset)}\left[ \objfn'(\revset_2) \right] .
\end{align*}

\subsection{Proof of Theorem~\ref{thm:tworoundub}}\label{apdx:tworoundub}
In this section, we state and prove a more general version of the bound in Theorem~\ref{thm:tworoundub} that does not require that $\frac{\loadscale}{4}$ be integral. This result immediately implies the result of  Theorem~\ref{thm:tworoundub}.
\setcounter{theorem}{5}
\begin{theorem}[Generalized]
Suppose $\stagetwofrac=1$, and consider any $\loadscale \in [10,000]$.  
Define $\epsilon = \lceil \frac{\loadscale}{4} \rceil - \frac{\loadscale}{4}$. 
Suppose there exists an assignment $\adassign^{(1)} \in \mathcal{M}(\revset, \papset; 1, 2)$ with mean similarity $\optavgval$.
Suppose there also exists an assignment $\adassign^{(\loadscale)} \in \mathcal{M}(\revset, \papset; \loadscale, 2\loadscale)$ with mean similarity $\scaledavgval$ that does not contain any of the pairs assigned in $\adassign^{(1)}$.
Then, choosing $\revset_2$ via random split gives that  
\begin{align*}
\mathbb{E}_{\revset_2}\left[\objfn(\revset_2)\right] \geq \frac{3}{4} \optavgval + \frac{\scaledavgval}{4} \left[ 1 - \frac{\sqrt{7}+\sqrt{6}}{2\sqrt{\pi \loadscale}} - \frac{3\epsilon}{\lceil \loadscale/4 \rceil } \right].
\end{align*}
\end{theorem}
\begin{proof}
We attempt to construct an assignment in each stage in two rounds. We first match all available pairs from $\adassign^{(1)}$ (tiebreaking randomly between the two reviewers if both are available), and then attempt to construct a larger assignment from $\adassign^{(\loadscale)}$. 

By Lemma~\ref{lem:indep}, we can consider drawing $\papset_2 \sim \mathcal{I}_{\stagetwofrac}(\papset)$ and $\revset_2 \sim \mathcal{I}_{\stagetwofrac / (1 + \stagetwofrac)}(\revset)$ and allowing papers to be underloaded. For all reviewers $\adrev \in \revset$, define the random variables $Z_{\adrev} = \begin{cases} 1 \text{ w.p. } 1/2 \\ 2 \text{ w.p. } 1/2 \end{cases}$ representing the stage that reviewer $\adrev$ is randomly chosen to be in. For each pair of reviewers $(i, j)$ that are matched to the same paper in $\adassign^{(1)}$, define the random variables $F_{ij} = \begin{cases} i \text{ w.p. } 1/2 \\ j \text{ w.p. } 1/2 \end{cases}$ representing the reviewer that will be assigned in round one if both are in the same stage. All of these random variables are independent. Define the total similarity value of the assignments as $\optval = 2 \numpap \optavgval$ and $\scaledval = 2 \numpap \loadscale \scaledavgval$. For $\adassign^{(\loadscale)}$, define the total similarity value assigned to paper $\adpap$ and reviewer $\adrev$ respectively as $\scaledval_{\adpap}$ and $\scaledval_{\adrev}$. 

\paragraph{Round One:} We first match all available pairs from $\adassign^{(1)}$. For any paper $\adpap \in \papset$, call $a, b$ the two reviewers assigned to $\adpap$ by $\adassign^{(1)}$. The value assigned to paper $\adpap$ across both stages is represented by a random variable $V_{\adpap} = \mathbb{I}[Z_a \neq Z_b] (S_{a\adpap} + S_{b\adpap}) + \mathbb{I}[Z_a = Z_b](S_{a\adpap}\mathbb{I}[F_{ab} = a] + S_{b\adpap}\mathbb{I}[F_{ab} = b])$. $\mathbb{E}[V_\adpap] = \frac{3}{4} (S_{a\adpap} + S_{b\adpap})$, so $\mathbb{E}[\sum_{\adpap \in \papset} V_{\adpap}] = \frac{3}{4} \optval$ is the total expected value assigned in round 1.

\paragraph{Round Two:} Fixing the round one assignments, we now attempt to find a matching for all remaining papers and reviewers by matching pairs from within $\adassign^{(\loadscale)}$. We first attempt to find an assignment with paper and reviewer loads of at most $\theta = \lceil \loadscale / 4 \rceil$  among the remaining reviewers and papers in each stage. We start with the pairs from $\adassign^{(\loadscale)}$ that both are  present in this stage and were not matched in round one, and randomly drop entries from each reviewer and paper until they are no longer overloaded. This argument mirrors the one made in the proof of Theorem~\ref{thm:oneroundub}.

We consider stage one without loss of generality. We start by constructing an assignment $C$ to include all pairs assigned in $\adassign^{(\loadscale)}$ where the reviewer and paper both were unmatched in round one and are in stage one. Each reviewer-paper pair in $\adassign^{(\loadscale)}$ can be assigned in $C$ with probability $\frac{1}{32}$, so $\mathbb{E}\left[ \sum_{\adrev \in \revset, \adpap \in \papset} C_{\adrev \adpap} \simmat_{\adrev \adpap} \right] = \frac{\scaledval}{32}$.

We then construct an assignment $B^{(1a)}$ from $C$ by removing assigned reviewers from each paper at random until each paper has load at most $\theta$.  Fix some paper $\adpap$, and define $W_\adpap$ as the event that paper $\adpap$ was not assigned in round one. The number of reviewers assigned to $\adpap$ in $\adassign^{(\loadscale)}$ that are in stage one and not assigned in round one is a $Binom(2 \loadscale, 1/8)$ random variable. The expected value assigned to $\adpap$ in this assignment is (using $f$ as the Binomial pmf),
\begin{align*}
\mathbb{E}\left[ \sum_{\adrev \in \revset} B^{(1a)}_{\adrev \adpap} \simmat_{\adrev \adpap} \Bigg| W_\adpap \right] &= \sum_{x = 0}^{\theta} f(x; 2\loadscale, 1/8) \scaledval_{\adpap} \frac{x}{2\loadscale} + \sum_{x = \theta+1}^{2\loadscale} f(x; 2\loadscale, 1/8) \scaledval_{\adpap} \frac{\theta}{2\loadscale} \\
&= \frac{\scaledval_{\adpap}}{2\loadscale} \mathbb{E}_{X \sim Binom(2 \loadscale, 1/8)}\left[ \min(X, \theta) \right] .
\end{align*}
Summing over all papers, since each paper has a $1/4$ change of being unmatched in round one,
\begin{align*}
\mathbb{E}\left[ \sum_{\adpap \in \papset} \sum_{\adrev \in \revset} B^{(1a)}_{\adrev \adpap} \simmat_{\adrev \adpap} \right] &= \frac{\scaledval}{8 \loadscale} \mathbb{E}_{X \sim Binom(2 \loadscale, 1/8)}\left[ \min(X, \theta) \right] .
\end{align*}

We separately construct an assignment $B^{(1b)}$ from $C$ by removing assigned papers from each reviewer at random until each reviewer has load at most $ \theta$. Fix some reviewer $\adrev$, and define $W_\adrev$ as the event that reviewer $\adrev$ was not assigned in round one. The number of papers assigned to $\adrev$ in $\adassign^{(\loadscale)}$ that are not assigned in round one is a $Binom( \loadscale, 1/4)$ random variable. The expected value assigned to $\adrev$ in this assignment is,
\begin{align*}
\mathbb{E}\left[ \sum_{\adpap \in \papset} B^{(1b)}_{\adrev \adpap} \simmat_{\adrev \adpap} \Bigg| W_\adrev  \right] &= \sum_{x = 0}^{ \theta} f(x;  \loadscale, 1/4) \scaledval_{\adrev} \frac{x}{ \loadscale} + \sum_{x =  \theta+1}^{\loadscale} f(x;  \loadscale, 1/4) \scaledval_{\adrev} \frac{ \theta}{ \loadscale} \\
&= \frac{\scaledval_{\adrev}}{\loadscale} \mathbb{E}_{X \sim Binom(\loadscale, 1/4)}\left[ \min(X, \theta) \right] .
\end{align*}
Summing over all reviewers, since each reviewer has a $1/8$ change of being both unmatched in round one and present in stage one, 
\begin{align*}
\mathbb{E}\left[ \sum_{\adrev \in \revset} \sum_{\adpap \in \papset} B^{(1b)}_{\adrev \adpap} \simmat_{\adrev \adpap} \right] &= \frac{\scaledval}{8  \loadscale} \mathbb{E}_{X \sim Binom( \loadscale, 1/4)}\left[ \min(X,  \theta) \right] .
\end{align*}

Finally, we construct $B^{(1)}$ to include all pairs assigned in both $B^{(1a)}$ and $B^{(1b)}$. It has value at least equal to the total value of $B^{(1a)}$ and $B^{(1b)}$ less the value of $C$, since the assigned pairs in $B^{(1a)}$ and $B^{(1b)}$ are subsets of the assigned pairs in $C$.
\begin{align*}
&\mathbb{E}\left[ \sum_{\adrev \in \revset, \adpap \in \papset} B^{(1)}_{\adrev \adpap} \simmat_{\adrev \adpap} \right] \\
&\qquad \geq \frac{\scaledval}{8 \loadscale} \left[ \mathbb{E}_{X \sim Binom(2 \loadscale, 1/8)}\left[ \min(X, \theta) \right] 
+  \mathbb{E}_{X \sim Binom( \loadscale, 1/4)}\left[ \min(X,  \theta) \right] 
- \frac{\loadscale}{4} \right].
\end{align*}
By construction this assignment has paper loads of at most $\theta$ and reviewer loads of at most $ \theta$ (among all reviewers and papers unmatched in round one and present in stage one). 

By a generalization of the Birkhoff-von Neumann theorem~\cite{Budish2009IMPLEMENTINGRA}, there exists an assignment with paper loads of at most $1$ and reviewer loads of at most $1$ among all reviewers and papers unmatched in round one and present in stage one, with value at least $\frac{1}{\theta}$ of the value of $B^{(1)}$. Totalling over both stages and dividing by $2 \numpap$, the round two assignments contribute at least 
\begin{align}
\frac{\scaledavgval}{4} \left[ \mathbb{E}_{X \sim Binom(2 \loadscale, 1/8)}\left[ \min\left(\frac{X}{\theta}, 1\right) \right] 
+  \mathbb{E}_{X \sim Binom( \loadscale, 1/4)}\left[ \min\left(\frac{X}{ \theta}, 1\right) \right] 
- \frac{\loadscale}{4 \theta} \right] \label{eq:tworoundlbexact}
\end{align}
to the mean assignment value.

If $X \sim Binom(N, p)$, the above bound is a function of $\mathbb{E}\left[ \min\left(\frac{X}{\ell}, 1\right) \right]$ for two binomial random variables where $N p \leq \ell$. Using the normal approximation presented in the proof of Theorem~\ref{thm:oneroundub}, we get the following approximation to the above bound (defining $\epsilon = \lceil \loadscale / 4 \rceil - (\loadscale/4)$):
\begin{align}
&\frac{\scaledavgval}{4} \left[ 1 - \frac{\sqrt{7}+\sqrt{6}}{2\sqrt{\pi \loadscale}} - \left(1 - \frac{\loadscale/4}{\lceil \loadscale/4 \rceil} \right) \left( \Phi\left(\frac{\epsilon}{\sqrt{\frac{7}{32} \loadscale}} \right) + \Phi\left(\frac{\epsilon}{\sqrt{\frac{3}{16} \loadscale}} \right) + 1\right) \right] \label{eq:tworoundlbapprox} \\
&\geq \frac{\scaledavgval}{4} \left[ 1 - \frac{\sqrt{7}+\sqrt{6}}{2\sqrt{\pi \loadscale}} - 3 \left(1 - \frac{\loadscale/4}{\lceil \loadscale/4 \rceil} \right) \right]  \nonumber \\
&= \frac{\scaledavgval}{4} \left[ 1 - \frac{\sqrt{7}+\sqrt{6}}{2\sqrt{\pi \loadscale}} - \frac{3\epsilon}{\lceil \loadscale/4 \rceil } \right] \nonumber.
\end{align}
Via simulation, we confirm that the approximation in~\eqref{eq:tworoundlbapprox} is in fact a lower bound on the expression in~\eqref{eq:tworoundlbexact} for all $\loadscale \in [10^4]$. In Figure~\ref{fig:bounds} of Section~\ref{sec:cond2}, we plot the more precise bound of~\eqref{eq:tworoundlbapprox}. 
\end{proof}

\end{document}